\documentclass{article} % For LaTeX2e
\usepackage{iclr2027_conference,times}

\usepackage{amsmath,amsfonts,bm}

\def\1{\bm{1}}

\DeclareMathAlphabet{\mathsfit}{\encodingdefault}{\sfdefault}{m}{sl}
\SetMathAlphabet{\mathsfit}{bold}{\encodingdefault}{\sfdefault}{bx}{n}

\DeclareMathOperator*{\argmin}{arg\,min}

\usepackage{graphicx}
\usepackage{hyperref}
\usepackage{url}

\usepackage{booktabs}

\newcommand{\xb}{\mathbf{x}}
\newcommand{\yb}{\mathbf{y}}
\newcommand{\zb}{\mathbf{z}}
\newcommand{\bb}{\mathbf{b}}

\newcommand{\Ab}{\mathbf{A}}

\newcommand{\Cb}{\mathbf{C}}
\newcommand{\Kb}{\mathbf{K}}
\newcommand{\Pb}{\mathbf{P}}
\newcommand{\Qb}{\mathbf{Q}}
\newcommand{\Rb}{\mathbf{R}}
\newcommand{\Sb}{\mathbf{S}}
\newcommand{\Vb}{\mathbf{V}}
\newcommand{\Wb}{\mathbf{W}}
\newcommand{\Xb}{\mathbf{X}}
\newcommand{\Zb}{\mathbf{Z}}

\newcommand{\cU}{\mathcal{U}}
\newcommand{\cV}{\mathcal{V}}
\newcommand{\cT}{\mathcal{T}}
\newcommand{\cH}{\mathcal{H}}
\newcommand{\cS}{\mathcal{S}}
\newcommand{\cL}{\mathcal{L}}

\newcommand{\RR}{\mathbb{R}}

\newcommand{\EE}{\mathbb{E}}

\usepackage{amsthm}
\newtheorem{theorem}{Theorem}
\newtheorem{lemma}{Lemma}
\newtheorem{corollary}{Corollary}
\newtheorem{assumption}{Assumption}
\newtheorem{definition}{Definition}
\newtheorem{remark}{Remark}
\newtheorem{setting}{Setting}

\title{Neural Scaling Laws of Transformer Operator Network}

\author{
Haoran Yan$^{1}$,
Zhongjie Shi$^{1}$,
Yuanzhe Xi$^{2}$,
Peng Chen$^{3}$,
Wenjing Liao$^{1}$
\\[0.5em]
$^{1}$School of Mathematics, Georgia Institute of Technology,
Atlanta, GA 30332, USA
\\
$^{2}$Department of Mathematics, Emory University,
Atlanta, GA 30322, USA
\\
$^{3}$School of Computational Science and Engineering,
Georgia Institute of Technology,
Atlanta, GA 30332, USA
}

\iclrfinalcopy 
\begin{document}

\maketitle
\lhead{Preprint}

\begin{abstract}
Transformers have emerged as powerful architectures for learning solution operators of physical systems. Empirically the prediction error has been observed to decrease when the data size and model size increase, suggesting neural scaling behavior. Yet a theoretical understanding of such scaling laws for transformer-based operator learning remains limited. In this work, we develop a theoretical framework for characterizing the approximation and generalization errors of transformer-based operator learning. Our analysis builds on a local-to-global approximation principle that is naturally aligned with the softmax attention mechanism and yields discretization-invariant output functions. On approximation theory, we derive a universal approximation error of transformer-based operator learning for H\"older-regular operators. On generalization theory, we establish a power scaling law between the prediction error  and the training data size. The rate of convergence represented by the scaling exponent explicitly reflects the dimensions of the input and output domains, the regularity of the underlying functions and operators, and crucially, the intrinsic dimension of the input function class. By exploiting this intrinsic low-dimensional structure, our analysis yields a power-law generalization rate for operator learning, in contrast to the logarithmic-type power-law rates appearing in existing analyses of operator learning with feedforward neural networks.  Numerical experiments validate the predicted power-law scaling and confirm that the convergence rate varies systematically with the intrinsic dimension of the input function class. 
\end{abstract}

\section{Introduction}
\vspace{-0.2cm}
 Learning solution operators for differential equations and operators arising in inverse problems is a fundamental problem in science and engineering. Many problems in applications can be formulated as an operator between function spaces where the input may represent an initial condition, boundary condition, forcing term, material coefficient, or geometric description, and the output represents the corresponding solution or quantity of interest.
Operator-learning methods seek to learn the input-to-output map and can therefore be repeatedly evaluated for new inputs once training is completed. This perspective has motivated a rapidly growing family of neural operator architectures, including Deep Operator Networks (DeepONets) \citep{lu2021learning}, Fourier Neural Operators (FNOs) \citep{li2020fourier} and other neural operators, together with numerous subsequent developments in theory \citep{lanthaler2022error,kovachki2024operator,liu2024deep,liu2026scaling,yang2026generalization} and computation \citep{pathak2022fourcastnet,li2023geometry}. 

More recently, Transformers have become increasingly prominent in operator learning. 
The attention mechanism in Transformers  \citep{vaswani2017attention} provides a natural way to represent samples of an input function as tokens, aggregate information globally across these observations, and generate predictions at arbitrary output locations. A growing family of transformer-based operator architectures has demonstrated strong performance for learning solution operators of partial differential equations and other physical systems, including the Galerkin Transformer \citep{cao2021choose}, OFormer  \citep{li2022transformer}, GNOT \citep{hao2023gnot}, Position-induced Transformer (PiT) \citep{chen2024positional}, Transolver \citep{wu2024transolver}, and many others \citep{bryutkin2024hamlet,luo2024hierarchical}.
Numerical experiments in these works \citep{li2022transformer,hao2023gnot} have demonstrated systematic improvements as training data increase, suggesting neural scaling behavior, which appears in a wide range of AI models \citep{kaplan2020scaling,havrilla2024understanding}.

These empirical advances motivate a quantitative theoretical understanding of how the approximation and generalization performance of transformer-based operator learning scales with model complexity, training-data size, and the underlying structure of the function spaces. Understanding these scaling laws can in turn guide the allocation of model capacity and training data, enabling more efficient use of computational resources.

When feedforward neural networks are utilized for operator learning, neural scaling laws about the approximation and generalization errors are theoretically characterized for DeepONets \citep{lanthaler2022error,liu2026scaling}, FNOs \citep{kovachki2021universal} and an encoder-decoder network \citep{liu2024deep}. Under regularity assumptions on the input and output functions and Lipschitz continuity of the operator, existing bounds yield only slow log-power scaling of the approximation and generalization errors with respect to model size and training-data size, respectively \citep{kovachki2024operator,liu2026scaling}. This slow log-power scaling arises from the difficulty of learning in infinite-dimensional spaces, and it is shown to be optimal for Lipschitz functionals  \citep[Proposition 2.21]{lanthaler2026parametric}. This slow log-power scaling can be improved to the power scaling law when the learning problem exhibits a low-dimensional structure \citep{liu2024deep,liu2025generalization,liu2026scaling}.

For transformer-based operator learning, existing works have investigated the approximation properties through universal approximation theory \citep{calvello2025continuum,shih2025transformers,zappala2026universal,furuya2026function}. However, these works do not provide a quantitative scaling between the approximation error and the network size, and the generalization error has not been analyzed in the literature.

To fill this gap, we develop a theoretical framework to answer the following question:
\begin{center}
\textit{How does transformer-based operator learning error scale with model size and data size? 
}
\end{center}

In this paper, we answer this question by establishing quantitative approximation and generalization theories of transformer-based operator learning.
We impose mild H\"older continuity assumptions on the input and output functions, as well as on the target operator. The complexity of the input function class is characterized by the $L^2$ covering number, which aligns with classical complexity definitions on function/metric spaces  and statistical learning theory \citep{hurewicz2015dimension,kolmogorov1959varepsilon,wainwright2019high}. 
This complexity characterization arises naturally in scientific applications where the input functions are governed by a small number of physical parameters, such as parameterized initial or boundary conditions, source locations and amplitudes, or geometric and material parameters.

To the best of our knowledge, this paper is the first to theoretically characterize neural scaling laws for Transformer-based operator learning. Our main contributions are summarized as follows.

\begin{itemize}
\vspace{-0.2cm}

\item
\textbf{Approximation and model scaling.}
On approximation theory, we explicitly construct a two-block softmax Transformer network to approximate $\gamma$-H\"older continuous operators in Theorem~\ref{thm:transformer_operator_approx}. When the transformer has model size $N_{\rm total}$, the operator approximation error scales like $N_{\rm total}^{-1/s}$ where the parameter $s$ is specified in \eqref{eq:complexity_exponent} with explicit dependence on the operator regularity, function regularity, and input function-space complexity.

\item
\textbf{Generalization and data scaling.} 
On generalization theory, we derive a generalization error bound for the empirical risk minimizer in Theorem~\ref{thm:generalization_bound}. When a transformer with a proper architecture is trained on $n$ input-output function pairs, the squared operator generalization error scales like $n^{-\frac{2}{2+s}}$, which characterizes a data scaling law between the operator prediction error and the number of training data size. 
\vspace{-0.2cm}

\end{itemize}

Our approximation and generalization results quantify the effects of operator regularity, function regularity, input-function space complexity, input-function discretization, model size, training-data size, and output observation noise. By leveraging the low-dimensional complexity of the input function space, we prove power scaling laws.
This contrasts with existing quantitative analyses for feedforward-neural-network-based operator learning, and more broadly infinite-dimensional learning, which exhibit logarithmic-type power scaling laws when the low-dimensional input function complexity is not considered \citep{mhaskar1997neural,kovachki2024operator,lanthaler2026parametric,shi2025nonlinear,liu2026scaling}. 

Our proof builds on a novel local-to-global approximation scheme based on two-level Softmax Partitions Of Unity (POU): one POU is over the input function space and the other POU is over the output domain. The proof is then followed by an explicit realization of this two-level approximation by a two-block Transformer. This construction is naturally aligned with the Softmax attention mechanism and yields discretization-invariant outputs. The point-wise feedforward layers are only used for exact linear feature assembly.

{\bf Organization.}
The remainder of the paper is organized as follows. Section~\ref{sec:operator_learning_setup} introduces the operator-learning problem, data-generation setting, and Transformer architecture. Section~\ref{sec:main_results} presents the approximation and
generalization results, together with the resulting model and data
scaling laws. Section~\ref{sec:proof_sketch} provides proof sketches for the main theorems. Section~\ref{sec:experiments} presents numerical experiments on the Burgers' and KdV solution operators. Additional definitions, architectural details, extensions, complete proofs, supporting lemmas, and experimental details are in the appendices.

{\bf Notation.}
For a positive integer $m$, we write $[m]:=\{1,\ldots,m\}$. We use lower-case letters for scalars, bold lower-case letters for vectors, and bold upper-case letters for matrices. For a probability measure $\rho$, we write $\|f\|_{L^2(\rho)}^2:=\int |f|^2\,d\rho$. For a collection of network parameters $\Theta$, $\|\Theta\|_\infty$ denotes the maximum absolute value of its entries, and $\operatorname{vec}(\Zb)$ denotes the column-wise vectorization of $\Zb$.

\vspace{-0.2cm}
\section{Operator Learning Problem and Transformer Architecture}
\label{sec:operator_learning_setup}
\vspace{-0.2cm}

\subsection{Problem setup}
\vspace{-0.2cm}

We consider supervised learning of a nonlinear operator $G:\cU\to\cV$ from finitely many evaluations of input functions and pointwise output queries. Let $\Omega_{\cU}=[0,1]^{d_1}$ be the input domain and let $\Omega_{\cV}\subseteq[0,1]^{d_2}$ be the output query domain. Here $\cU$ is a class of scalar-valued functions on $\Omega_{\cU}$, and $\cV$ is a class of scalar-valued functions on $\Omega_{\cV}$.

\begin{setting}[Data Generation]
\label{setting:data}
A training data set $\cS$ is generated under the following protocol:
\begin{enumerate}
    \item \textbf{Function sampling.}
    The input functions $u_1,\ldots,u_n$ are drawn i.i.d. from a probability distribution $\rho_u$ supported on $\cU$.

    \item \textbf{Input discretization.}
    Let $X_{n_x}:=(\xb_1,\ldots,\xb_{n_x})\subset\Omega_{\cU}$ be an input grid, shared across all sampled functions and independent of the sampled functions. For each $u_i$, define its discretized representation by
    \begin{equation}
    \label{eq:input_discretization}
    S_{n_x}(u_i) := \bigl((\xb_1,u_i(\xb_1)),\ldots (\xb_{n_x},u_i(\xb_{n_x})) \bigr).
    \end{equation}

    \item \textbf{Output query sampling.}
    The query points $\{\yb_{i,j}\}_{j=1}^{n_y}$ are drawn i.i.d. from a probability distribution $\rho_y$ supported on $\Omega_{\cV}$, independently of $\{u_i\}_{i=1}^n$. The observed scalar outputs are
    $$
    v_{i,j} = G(u_i)(\yb_{i,j})+\xi_{i,j},
    \qquad
    i\in[n],\ j\in[n_y],
    $$
    where the noise variables $\{\xi_{i,j}\}_{i\in[n],\,j\in[n_y]}$ are independent, independent of all input functions and query points, and mean-zero sub-Gaussian with variance proxy $\sigma^2$.
\end{enumerate}
The training sample can be written as $\cS = \left\{\left(S_{n_x}(u_i),\yb_{i,j},v_{i,j}\right): i\in[n],\ j\in[n_y] \right\}$. The data-generation protocol is illustrated in Figure~\ref{fig:data_generation}.
\begin{figure}
    \centering
    \includegraphics[width=0.8\linewidth]{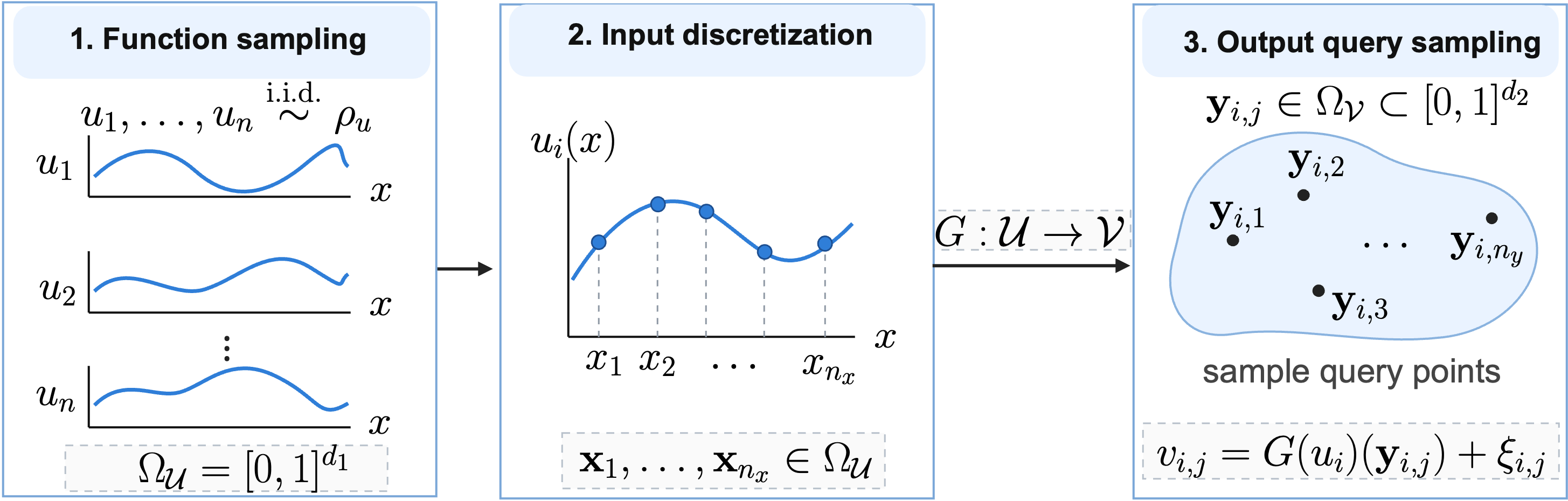}
    \caption{
    Illustration of the operator-learning data-generation setting. Each input function is observed on a shared input grid, while output values are queried at independently sampled locations.
    }
    \label{fig:data_generation}
\end{figure}
\end{setting}

For any predictor $T$ mapping $(S_{n_x}(u),\yb)$ to a scalar, define the
population risk by
\begin{equation}
\label{eq:population_risk}
\cL(T) := \EE_{u\sim\rho_u}\EE_{\yb\sim\rho_y}\left[\left(T(S_{n_x}(u),\yb)-G(u)(\yb) \right)^2 \right].
\end{equation}
For a hypothesis class $\cH$, define the empirical risk on the training sample
$\cS$ by
\begin{equation}
\label{eq:empirical_risk}
\textstyle
\widehat{\cL}_{\cS}(T) := \frac{1}{nn_y}\sum_{i=1}^{n}\sum_{j=1}^{n_y}\left(T(S_{n_x}(u_i),\yb_{i,j})-v_{i,j}\right)^2.
\end{equation}
Let 
\begin{equation}
\label{eq:minimizer}
\widehat T\in\argmin_{T\in\cH}\widehat{\cL}_{\cS}(T)
\end{equation} be an empirical
risk minimizer. Since $\widehat T$ depends on the random training
sample $\cS$, our statistical objective is to bound the expected population risk 
\begin{equation}
\label{eq:statistical_objective}
\EE_{\cS}\cL(\widehat T) = \EE_{\cS}\EE_{u\sim\rho_u}\EE_{\yb\sim\rho_y}\left[\left(\widehat T(S_{n_x}(u),\yb)-G(u)(\yb)\right)^2\right],
\end{equation}
where $\cH$ is chosen as the Transformer hypothesis class introduced in Subsection~\ref{subsec:transformer_architecture}.

\vspace{-0.2cm}
\subsection{Transformer Architecture}
\label{subsec:transformer_architecture}
\vspace{-0.2cm}

Given the discretized input function $S_{n_x}(u)$ defined in \eqref{eq:input_discretization} and an output query $\yb\in\Omega_{\cV}$, we organize the input observations and the query into a sequence of tokens.
\begin{equation}
\label{eq:transformer_input_tokens}
\Xb(u,\yb) :=
\begin{bmatrix}
\xb_1 & \cdots & \xb_{n_x} & \mathbf 0_{d_1} & \cdots & \mathbf 0_{d_1}
\\
\mathbf 0_{d_2} & \cdots & \mathbf 0_{d_2} & \yb & \cdots & \mathbf 0_{d_2}
\\
u(\xb_1) & \cdots & u(\xb_{n_x}) & 0 & \cdots & 0
\end{bmatrix}.
\end{equation}
The first $n_x$ columns encode the spatial locations and corresponding function values in $S_{n_x}(u)$, while one additional column encodes the output query $\yb$.

A preprocessing map converts \eqref{eq:transformer_input_tokens}
into the initial hidden representation
\begin{equation}
\label{eq:transformer_preprocessing}
\Zb_0
=
\mathcal P\bigl(S_{n_x}(u),\yb\bigr)
=
\Wb_E\Xb(u,\yb)
+
\bb_E\mathbf 1_P^\top
+
\Pb_{\rm pos}
\in\RR^{D\times P},
\end{equation}
where $\Wb_E$ and $\bb_E$ define a token-wise affine embedding and
$\Pb_{\rm pos}$ is a fixed structural-positional encoding. The explicit
choice of the preprocessing map used in the constructive approximation
proof is given in Appendix~\ref{app:transformer_realization}.

The representation is then propagated through $L$ encoder blocks. For $\ell\in[L]$, let $A_\ell$ denote a multi-head self-attention layer and $F_\ell$ a point-wise ReLU feedforward layer, and write
\begin{equation}
\label{eq:transformer_encoder_blocks}
\widehat{\Zb}_\ell = A_\ell(\Zb_{\ell-1}),
\qquad
\Zb_\ell = F_\ell(\widehat{\Zb}_\ell).
\end{equation}
The precise definitions of $A_\ell$ and $F_\ell$ are provided in
Appendix~\ref{app:transformer_architecture}. 
The scalar prediction is obtained through the linear readout
\begin{equation}
\label{eq:transformer_readout}
T_\Theta\bigl(S_{n_x}(u),\yb\bigr) = \mathbf c_{L+1}^{\top}\operatorname{vec}(\Zb_L).
\end{equation}

We next define the Transformer hypothesis class used in our analysis.

\begin{definition}[Transformer Network Class]
\label{def:transformer_class}
For depth $L\in\mathbb N$, embedding dimension $D\in\mathbb N$, sequence length $P\in\mathbb N$, layer configurations $\{H^\ell\}_{\ell=1}^L$, $\{d_k^\ell\}_{\ell=1}^L$, $\{d_v^\ell\}_{\ell=1}^L$, $\{d_{\rm ff}^\ell\}_{\ell=1}^L\subset\mathbb N$, and parameter magnitude bound $M_{\max}>0$, define
$$
\begin{aligned}
&\cT\Big(L, D, P, \{H^\ell\}_{\ell=1}^L, \{d_k^\ell\}_{\ell=1}^L, \{d_v^\ell\}_{\ell=1}^L, \{d_{\rm ff}^\ell\}_{\ell=1}^L, M_{\max}\Big)
\\
&\quad =
\left\{T_\Theta \ \middle|\
\begin{array}{l}
T_\Theta \text{ is an $L$-block encoder-only Transformer of the form \eqref{eq:transformer_preprocessing}--\eqref{eq:transformer_readout},} \\
\text{with embedding dimension $D$, sequence length $P$, and parameter bound } \|\Theta\|_\infty\le M_{\max}, \\
\text{and for each block $\ell\in[L]$, there are $H^\ell$ heads, query/key dimension $d_k^\ell$,} \\
\text{value dimension $d_v^\ell$, and FFN hidden width $d_{\rm ff}^\ell$}
\end{array}
\right\}.
\end{aligned}
$$
\end{definition}

\vspace{-0.2cm}
\section{Main Results}
\label{sec:main_results}
\vspace{-0.2cm}

We now present the main theoretical results. We first state the assumptions used throughout the analysis. Under these assumptions, we establish a quantitative approximation result for Transformer-based operator learning and derive the resulting model scaling law. We then establish a generalization bound for the empirical risk minimizer and derive the corresponding data scaling law. 

\vspace{-0.2cm}
\subsection{Assumptions}
\label{subsec:assumptions}
\vspace{-0.2cm}

To establish the approximation and generalization theory, we introduce the following assumptions on the operator, the input function class, and the output function class.

\begin{assumption}[Operator Continuity]
\label{assum:operator}
The operator $G:\cU\to\cV$ satisfies a $\gamma$-H\"older continuity condition in the $L^2(\Omega_{\cU})$ and $L^2(\rho_y)$ metrics. Namely, there exist constants $L_G>0$ and $\gamma\in(0,1]$ such that, for all $u_1,u_2\in\cU$,
\begin{equation}
\label{eq:operator_continuity}
\|G(u_1)-G(u_2)\|_{L^2(\rho_y)} \le L_G\|u_1-u_2\|_{L^2(\Omega_{\cU})}^{\gamma}.
\end{equation}
\end{assumption}
\begin{assumption}[Input Function Class]
\label{assum:input_space}
The input function class $\cU$ satisfies the following conditions.
\begin{enumerate}
    \item \textbf{Smoothness.}
    There exist constants $\alpha>0$ and $B_{\cU}>1$ such that
    \begin{equation}
    \label{eq:input_holder_bound}
    \sup_{u\in\cU}
    \|u\|_{C^\alpha(\Omega_{\cU})}
    \le B_{\cU},
    \end{equation}
    where the H\"older norm is defined in Definition~\ref{def:holder_norm} of Appendix~\ref{app:additional_definitions}.

    \item \textbf{Low-dimensional structure.}
    For every $r_u\in(0,1]$, there exists a  $r_u$-cover
    $\{u_k\}_{k=1}^{C_U}\subset\cU$ under the $L^2(\Omega_{\cU})$ norm, that is,
    $\sup_{u\in\cU}\min_{k\in[C_U]}\|u-u_k\|_{L^2(\Omega_{\cU})}\le r_u$. Moreover, the covering number satisfies
    \begin{equation}
    \label{eq:input_covering}
    C_U\le C_1r_u^{-d_{\cU}},
    \end{equation}
    for some constants $C_1\ge1$ and $d_{\cU}\ge0$.
\end{enumerate}
\end{assumption}

\begin{assumption}[Output Function Class]
\label{assum:output_space}
There exist constants $\beta \in (0,1]$ and $B_{\cV}>0$ such that
\begin{equation}
\label{eq:output_holder_bound}
\sup_{v\in\cV}
\|v\|_{C^\beta(\Omega_{\cV})}
\le B_{\cV}.
\end{equation}
\end{assumption}

\begin{remark}
The assumptions above separate the three sources of complexity in the analysis. Assumption~\ref{assum:operator} is a stability condition for the operator: nearby input functions in $L^2(\Omega_{\cU})$ have nearby outputs in $L^2(\rho_y)$. The $C^\alpha$ and $C^\beta$ regularity assumptions for input and output functions in Assumption~\ref{assum:input_space}-1 and Assumption~\ref{assum:output_space} are standard in operator learning theory \citep{lanthaler2022error,lanthaler2026parametric,kovachki2024operator,liu2024deep,liu2026scaling}. 

The covering condition in Assumption~\ref{assum:input_space}-2 captures the low-dimensional structure of $\cU$ under the $L^2(\Omega_{\cU})$ metric. By introducing Assumption~\ref{assum:input_space}-2, we can incorporate the function complexity into the theoretical framework. Roughly speaking, $d_{\cU}$ denotes the number of parameters in the input function class.  Such intrinsic low-dimensional structure is not merely a theoretical assumption, but arises naturally in many scientific applications where input functions are determined by a small number of physical parameters, such as parametrized initial or boundary conditions, source locations and amplitudes, material coefficients, geometric parameters, or low-dimensional latent representations of random fields.
\end{remark}

\vspace{-0.2cm}
\subsection{Approximation theory and model scaling law}
\label{subsec:transformer_approximation}
\vspace{-0.2cm}

We first establish a quantitative approximation bound for the Transformer class  in Definition~\ref{def:transformer_class}.

\begin{theorem}[Transformer Approximation for Operator Learning]
\label{thm:transformer_operator_approx}
In Setting~\ref{setting:data}, suppose Assumptions~\ref{assum:operator}, \ref{assum:input_space},
\ref{assum:output_space} hold and $d_{\cU}/\gamma+d_2/\beta \ge d_1/\alpha$. For the input discretization, let $X_{n_x}$ be a sampling grid
constructed from tensor-product Gauss--Legendre points on a uniform partition of $\Omega_{\cU}=[0,1]^{d_1}$. Let $\epsilon_0\in(0,e^{-1}]$ be a structural constant. For any $\epsilon\in(0,\epsilon_0)$, suppose the sequence length $P=C_UC_V+1$ satisfies the compatibility condition $P\ge n_x+2$. Then there exists a Transformer network
$$
T_\Theta
\in
\cT\Big(
L,D,P,\{H^\ell\}_{\ell=1}^L,
\{d_k^\ell\}_{\ell=1}^L,
\{d_v^\ell\}_{\ell=1}^L,
\{d_{\rm ff}^\ell\}_{\ell=1}^L,
M_{\max}
\Big)
$$
such that
\begin{equation}
\label{eq:transformer_approximation_error}
\sup_{u\in\cU}
\left\|
T_\Theta(S_{n_x}(u),\cdot)-G(u)
\right\|_{L^2(\rho_y)}
\le
\epsilon.
\end{equation}
The network can be chosen with the following structural parameters:
\begin{itemize}
\item To control the input-discretization error, choose $n_x=\left\lceil \left({16B_{\cV}M_uC_{\rm quad}}/{\epsilon}\right)^{d_1/\alpha} \right\rceil$.

    \item $L=2$, $D=d_1+d_2+n_x+9$, and the sequence length satisfies
    $P\le C_P\epsilon^{-d_{\cU}/\gamma-d_2/\beta}$.    

    \item In block $\ell=1$, $H^1=(n_x+3)P+2$, $d_k^1=5$, $d_v^1=2$, and
    $d_{\rm ff}^1=2D$.

    \item In block $\ell=2$, $H^2=1$, $d_k^2=d_v^2=1$, and
    $d_{\rm ff}^2=2D$.
\end{itemize}
Here $M_u$ is defined in \eqref{eq:Mu_definition} of
Appendix~\ref{app:oracle_pou}, $C_P$ is the positive constant defined in
\eqref{eq:CP_definition} of Appendix~\ref{app:proof:transformer_operator_approx}, and $C_{\rm quad}$ is the quadrature constant in Lemma~\ref{lemma:quadrature} of Appendix~\ref{app:oracle_pou}. The maximum parameter magnitude $M_{\max}:= \|\Theta\|_\infty$ and the total number of network parameters $N_{\rm total}$ satisfy
\begin{equation}
\label{eq:transformer_parameter_magnitude}
M_{\max} \le C_{\rm mag}\epsilon^{-q_M}\left(\log{\epsilon^{-1}}\right)^{1+d_1/\alpha},
\end{equation}
\begin{equation}
\label{eq:transformer_parameter_count}
N_{\rm total} \le C_N \epsilon^{-s}\left(\log{\epsilon^{-1}}\right)^{2d_1/\alpha},
\end{equation}
where
\begin{equation}
\label{eq:complexity_exponent}
s := \frac{d_{\cU}}{\gamma} + \frac{d_2}{\beta} + \frac{2d_1}{\alpha}\left(1+\frac{2}{\gamma}\right)
\end{equation}
and
$$
q_M := \max\left\{2\left(
\frac{d_{\cU}}{\gamma} + \frac{d_2}{\beta}\right), \frac{d_1}{\alpha} \left(1+\frac{2}{\gamma}\right) + \frac{2}{\gamma}, \frac{d_1}{\alpha} \left(1+\frac{2}{\gamma}\right) + \frac{2}{\beta} \right\}.
$$
The positive constants $C_{\rm mag}$ and $C_N$ are defined in
\eqref{eq:Cmag_definition} and \eqref{eq:CN_definition}
in Appendix~\ref{app:proof:transformer_operator_approx}.
\end{theorem}

The proof of Theorem~\ref{thm:transformer_operator_approx} is given in
Appendix~\ref{app:proof:transformer_operator_approx}.
The theorem has the following implications.

\begin{itemize}

\item \textbf{Model scaling law.}
Theorem~\ref{thm:transformer_operator_approx} quantifies how the
approximation accuracy improves as the Transformer model size increases. A model of size $N_{\rm total}$ can achieve an approximation error in the order of $N_{\rm total}^{-1/s}$ up to logarithmic factors.  %Hence, $1/s$ serves as the model scaling exponent. 
The scaling exponent $-1/s$ explicitly reflects the effect of the intrinsic input dimension $d_{\cU}$, and the function domain dimensions $d_1,d_2$, and the function and operator regularity. By leveraging the input intrinsic dimension $d_{\cU}$, Theorem \ref{thm:transformer_operator_approx} yields a power model scaling law. In contrast, without the exploration of the input intrinsic dimension, the model scaling law is proved in the order of $(\log N_{\rm total})^{-{\rm power}}$ for feedforward-type operator learning with   \citep{kovachki2024operator,liu2026scaling,lanthaler2026parametric}. For band-limited input functions, the rate of convergence can be further improved in Appendix~\ref{app:bandlimited}.

\item \textbf{First quantitative approximation result for transformer operator network.} Existing approximation theories for transformer-based operator learning establish universal approximation properties
\citep{calvello2025continuum,shih2025transformers,zappala2026universal,furuya2026function}, but do not explicitly quantify the relation between approximation accuracy and Transformer model size. To the best of our knowledge, Theorem~\ref{thm:transformer_operator_approx} is the first quantitative approximation result for transformer-based operator learning that explicitly characterizes the scaling between the
approximation error and the Transformer model size.

\end{itemize}

\vspace{-0.2cm}
\subsection{Generalization theory and data scaling law}
\label{subsec:generalization_bound}
\vspace{-0.2cm}

We next turn from approximation to statistical learning. Our goal is to bound the expected population risk $\EE_{\cS}\cL(\widehat T)$ in \eqref{eq:statistical_objective}. To uniformly control the prediction error relative to the bounded target operator, we use a clipped Transformer class.

For the output bound $B_{\cV}>0$ in Assumption~\ref{assum:output_space},
define the clipping map $\Pi_{B_{\cV}}(t) = \min\{\max\{t,-B_{\cV}\},B_{\cV}\}$. For a fixed approximation tolerance $\epsilon$, let $\cT_\epsilon$ denote the Transformer class specified in Theorem~\ref{thm:transformer_operator_approx}, and define the corresponding clipped hypothesis class by
$
\cH_\epsilon^{\rm clip} := \left\{\Pi_{B_{\cV}}\circ T: T\in\cT_\epsilon \right\}.
$
We then define the clipped empirical risk minimizer (ERM) by 
\begin{equation}
\widehat T_\epsilon \in \argmin_{T\in\cH_\epsilon^{\rm clip}} \widehat{\cL}_{\cS}(T).
\label{eq:minimizerclip}
\end{equation}

\begin{theorem}[Generalization Error of the Clipped Transformer ERM]
\label{thm:generalization_bound}
In Setting~\ref{setting:data}, suppose
Assumptions~\ref{assum:operator}, \ref{assum:input_space},
\ref{assum:output_space} hold and  $d_{\cU}/\gamma+d_2/\beta \ge d_1/\alpha$. Let $n,n_y\ge2$, and use the input discretization $X_{n_x}$ in Theorem~\ref{thm:transformer_operator_approx}. For any $\epsilon\in(0,\epsilon_0)$, suppose the sequence-length compatibility condition $P=C_UC_V+1\ge n_x+2$ is satisfied. Then the clipped ERM
$\widehat T_\epsilon$ in \eqref{eq:minimizerclip} satisfies
\begin{equation}
\label{eq:generalization_epsilon_bound}
\EE_{\cS}\cL(\widehat T_\epsilon)
\le
C_{\rm gen}
\left[
\epsilon^2
+
\left(
1+\frac{\sigma^2}{n_y}
\right)
\frac{
\epsilon^{-s}
\left(\log\frac{1}{\epsilon}\right)^{\frac{2d_1}{\alpha}+1}
}{n}
\right],
\end{equation}
where $s$ is defined in \eqref{eq:complexity_exponent}, and the constant $C_{\rm gen}>0$ is defined in \eqref{eq:Cgen_definition} of Appendix~\ref{app:proof:generalization_bound}, which depends only on the structural quantities in the assumptions and $\sigma$, but not on $\epsilon$, $n$, or $n_y$. 

Moreover, choose
$$
\epsilon = \left[\frac{1+\sigma^2/n_y}{n}(\log n)^{\frac{2d_1}{\alpha}+1}\right]^{\frac{1}{2+s}}.
$$
Provided that this choice lies in $(0,\epsilon_0)$ and satisfies the preceding sequence-length compatibility condition, we have
\begin{equation}
\label{eq:genscaling}
\EE_{\cS}\cL(\widehat T_{\epsilon})
\le
C_{\rm rate}\left[\frac{1+\sigma^2/n_y}{n}(\log n)^{\frac{2d_1}{\alpha}+1}\right]^{\frac{2}{2+s}},
\end{equation}
where the constant $C_{\rm rate}:=2C_{\rm gen}$.

\end{theorem}

The proof of Theorem~\ref{thm:generalization_bound} is given in
Appendix~\ref{app:proof:generalization_bound}. The theorem has the following implications.

\begin{itemize}
\item \textbf{Data scaling law.}
Theorem~\ref{thm:generalization_bound} quantifies a power scaling law between the generalization/prediction error and the training data size $n$ in \eqref{eq:genscaling}. The mean squared generalization error is in the order of  $n^{-\frac{2}{2+s}}$ up to logarithmic factors. The scaling exponent $-2/(2+s)$ characterizes the effect from the input intrinsic dimension $d_{\cU}$, and the function domain dimensions $d_1,d_2$, and the function/operator regularity. The output query size $n_y$  helps to reduce the noise contribution.
A large $n_y$ allows one to reduce the noise effect, since the upper bound in \eqref{eq:genscaling} depends on $\sigma^2/{n_y}$.

\item \textbf{First quantitative generalization result for transformer operator network.}
Empirical experiments on transformer-based operator learning have demonstrated systematic performance improvements as the training data size increases \citep{li2022transformer,hao2023gnot}, but a theoretical characterization of this data dependence remains underdeveloped. To the best of our knowledge, Theorem~\ref{thm:generalization_bound} provides the first quantitative generalization result that explicitly characterizes the data scaling law for transformer-based operator learning.

\end{itemize}

\vspace{-0.2cm}
\section{Proof Sketch}
\label{sec:proof_sketch}
\vspace{-0.2cm}

\subsection{Proof sketch of Theorem~\ref{thm:transformer_operator_approx}.}
\vspace{-0.2cm}

The proof consists of two main steps, corresponding to Appendices~\ref{app:oracle_pou} and~\ref{app:transformer_realization}.
First, building on the two-level Softmax partition-of-unity (POU) framework for in-context learning introduced by \citet{shi2026transformers}, we construct an oracle approximation of the target operator: one Softmax POU localizes the input function over the input function space, while the other localizes the output query over the output domain. Second, we show that this oracle approximation can be explicitly realized by a two-block Transformer, where the attention mechanism implements the localization and aggregation induced by the two-level POU constructions.

$\bullet$ {\bf Step 1: Oracle approximation via two-level Softmax POU
(Appendix~\ref{app:oracle_pou}).
}
We first construct an oracle approximation of the target operator independently of the Transformer architecture. The construction is illustrated in
Figure~\ref{fig:two_level_pou}. Let $\{u_k\}_{k=1}^{C_U}\subset\cU$ be an $r_u$-cover of the input function class and let $\{\zb_l\}_{l=1}^{C_V}\subset\Omega_{\cV}$ be an $r_y$-cover of the output domain. We use the localization parameters $M_u$ and $M_y$ defined in
\eqref{eq:Mu_definition} and \eqref{eq:My_definition} of Appendix~\ref{app:oracle_pou}. The input-space Softmax POU is constructed from the discretized inner
products
\begin{equation}
\label{eq:discrete_input_space_weight}
\textstyle
\widetilde\beta_{k,n_x}(u)
:=
\frac{
\exp\left(
2M_u\langle u,u_k\rangle_{n_x}
-
M_u\|u_k\|_{L^2(\Omega_{\cU})}^2
\right)
}{
\sum_{k'=1}^{C_U}
\exp\left(
2M_u\langle u,u_{k'}\rangle_{n_x}
-
M_u\|u_{k'}\|_{L^2(\Omega_{\cU})}^2
\right)
},
\end{equation}
while the output-domain Softmax POU is
\begin{equation}
\label{eq:eta_l}
\textstyle
\widetilde\eta_l(\yb)
:=
\frac{
\exp\left(
2M_y\yb^\top\zb_l-M_y\|\zb_l\|_2^2
\right)
}{
\sum_{l'=1}^{C_V}
\exp\left(
2M_y\yb^\top\zb_{l'}-M_y\|\zb_{l'}\|_2^2
\right)
}.
\end{equation}
Combining the two partitions gives the oracle approximation
\begin{equation}
\label{eq:discrete_two_level_pou}
\textstyle
\widetilde G_{r_u,r_y,n_x}(u)(\yb)
=
\sum_{k=1}^{C_U}
\sum_{l=1}^{C_V}
G(u_k)(\zb_l)
\widetilde\beta_{k,n_x}(u)
\widetilde\eta_l(\yb).
\end{equation}

\begin{figure}[t]
    \centering
    \includegraphics[width=0.9\linewidth]{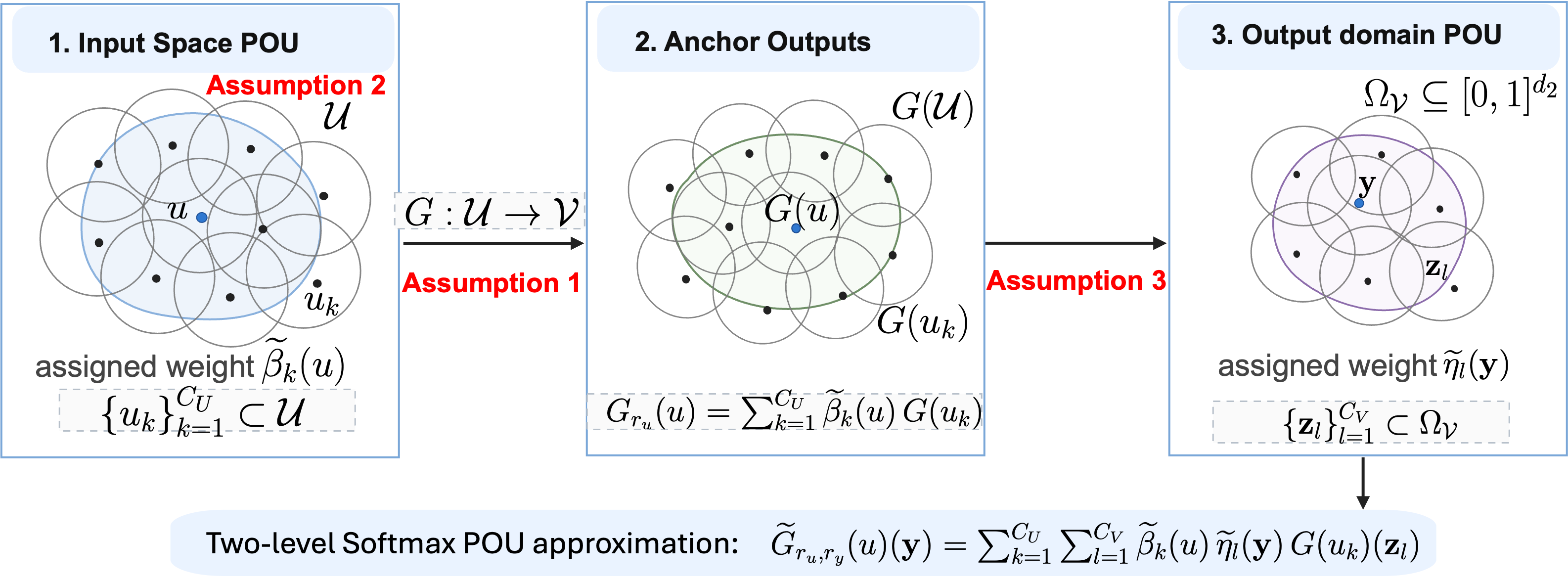}
    \caption{
    Illustration of the two-level Softmax POU approximation.
    The input-space POU localizes $u$ over function-space anchors
    $\{u_k\}_{k=1}^{C_U}$, while the output-domain POU localizes the
    query $\yb$ over anchors $\{\zb_l\}_{l=1}^{C_V}$.
    Their combination yields the oracle approximation in
    \eqref{eq:discrete_two_level_pou}.
    }
    \label{fig:two_level_pou}
\end{figure}

\textbf{Error control for the oracle construction
(Lemmas~\ref{lemma:input_space_pou}--\ref{lemma:total_approximation_error}
in Appendix~\ref{app:oracle_pou}).}
The supporting lemmas quantify the input-space localization, output-domain localization, and input-discretization errors. Together, they show that the oracle approximation converges to the target operator as $r_u$, $r_y$, and the input discretization are refined.

$\bullet$ {
\bf Step 2: Transformer realization of the oracle approximation
(Appendix~\ref{app:transformer_realization}).
}
We next show that the oracle approximation \eqref{eq:discrete_two_level_pou} can be realized by a two-block Transformer  in Section~\ref{subsec:transformer_architecture}. 

\textbf{Pair-token preprocessing (Lemma~\ref{lemma:pair_token_preprocess} in
Appendix~\ref{app:transformer_realization}).}
The preprocessing map in \eqref{eq:transformer_preprocessing} is chosen
so that the first $P-1=C_UC_V$ active tokens are indexed by the
input-output anchor pairs $(u_k,\zb_l)\in\{u_k\}_{k=1}^{C_U}\times\{\zb_l\}_{l=1}^{C_V}$.

\textbf{Linear feature extraction and joint-logit assembly
(Lemmas~\ref{lemma:pair_token_affine_mha} and
\ref{lemma:pair_token_joint_logit_ffn} in
Appendix~\ref{app:transformer_realization}).}
The first attention layer $A_1$ computes in parallel the linear features
appearing in the input- and output-localization logits in
\eqref{eq:discrete_input_space_weight} and \eqref{eq:eta_l}, together
with the corresponding anchor values $G(u_k)(\zb_l)$.
The point-wise feedforward layer $F_1$ assembles these features into the
joint localization logit $2M_u\langle u,u_k\rangle_{n_x}-M_u\|u_k\|_{L^2}^2
+2M_y\yb^\top\zb_l-M_y\|\zb_l\|_2^2$ for each anchor pair $(k,l)$.

\textbf{Softmax aggregation
(Lemma~\ref{lemma:pair_token_softmax_aggregation} in
Appendix~\ref{app:transformer_realization}).}
The second attention layer $A_2$ applies Softmax to the joint
localization logits. Since each joint logit is the sum of the
input-space and output-domain logits, the resulting normalized weight
factorizes as $\widetilde\beta_{k,n_x}(u)\widetilde\eta_l(\yb)$. Using $G(u_k)(\zb_l)$ as the corresponding values, $A_2$ therefore performs exactly the two-level aggregation in \eqref{eq:discrete_two_level_pou}. The second feedforward layer $F_2$ is chosen as the identity, and the linear readout in \eqref{eq:transformer_readout} extracts the scalar output.

The supporting lemmas in Appendix~\ref{app:transformer_realization}
show that this two-block construction realizes the oracle approximation
to the accuracy required in Theorem~\ref{thm:transformer_operator_approx}.
The full proof in Appendix~\ref{app:proof:transformer_operator_approx} combines this realization with the oracle approximation from Step~1 to
yield the approximation bound \eqref{eq:transformer_approximation_error}
and, by tracking the resulting architecture size, the model scaling law
\eqref{eq:transformer_parameter_count}.

\subsection{Proof sketch of Theorem~\ref{thm:generalization_bound}.}
\vspace{-0.2cm}

The proof consists of two main steps, corresponding to
Appendices~\ref{app:covering_number} and \ref{app:proof:generalization_bound}. First, we control the variance term, which measures the statistical estimation error of the empirical risk minimizer, through a covering-number bound for the clipped Transformer hypothesis class. Second, the approximation result in Theorem~\ref{thm:transformer_operator_approx} controls the bias term by showing that the Transformer class can approximate the target operator with squared error of order $\epsilon^2$. The proof of Theorem~\ref{thm:generalization_bound} combines these two bounds and balances them to obtain the resulting data scaling law.

\textbf{Covering-number control of the variance
(Lemma~\ref{lemma:transformer_covering} in
Appendix~\ref{app:covering_number}).}
Using the covering number defined in
Definition~\ref{def:covering_number} of
Appendix~\ref{app:additional_definitions},
Lemma~\ref{lemma:transformer_covering} gives
$$
\log\mathcal N
\left(
\eta,\cH_\epsilon^{\rm clip},
\|\cdot\|_{\infty,\infty}
\right)
\le
N_{\rm total}
\log\left(
{
503552P^7D^{17}M_{\max}^{21}B_{\cU}^5
}/{\eta}
\right).
$$

\textbf{Bias--variance tradeoff
(Appendix~\ref{app:proof:generalization_bound}).}
The approximation bound \eqref{eq:transformer_approximation_error}
controls the bias term, which is of order $\epsilon^2$. The covering-number estimate above controls the variance term through the empirical risk minimization property and the sub-Gaussian noise assumption. Combining these two gives \eqref{eq:generalization_epsilon_bound}. Balancing the two terms by the choice of $\epsilon$ in Theorem~\ref{thm:generalization_bound} yields the data scaling law \eqref{eq:genscaling}.

\vspace{-0.2cm}
\section{Experiments}
\label{sec:experiments}
\vspace{-0.2cm}

We use two nonlinear PDE solution operators to test a central implication of Theorem~\ref{thm:generalization_bound}: whether increasing the intrinsic dimension $d_{\cU}$ of the input function class is associated with slower empirical data scaling. We consider the one-dimensional viscous Burgers' equation $u_t+uu_x=\nu u_{xx}$ and the Korteweg--de Vries (KdV) equation $u_t+6uu_x+u_{xxx}=0$, representing dissipative and dispersive dynamics, respectively, and learn the final-time solution operator $G:u_0(\cdot)\mapsto u(\cdot,T)$.

For Burgers', we set $\nu=0.02$ and $T=0.05$ and consider parameterized compactly supported bump families with $d_{\cU}\in\{2,8\}$. For KdV, we use exact multi-soliton families on $x\in[0,44]$ at $T=1$ with $d_{\cU}\in\{4,8\}$.

For both equations, we use $n\in\{32,64,128,256,512,1024\}$ independently sampled training functions, while keeping the Transformer  architecture and training protocol fixed within each problem. Test MSE is evaluated in physical units over $128$ held-out functions and $1024$ uniformly spaced spatial points. Common implementation and evaluation details are provided in Appendix~\ref{app:common_experimental_setup}, with PDE-specific input
families and additional details in Appendices~\ref{app:Burgers'_experiments} and~\ref{app:kdv_experiments}.

\begin{figure}[t]
    \centering
    \begin{minipage}[t]{0.47\linewidth}
        \centering
        \includegraphics[width=\linewidth]
        {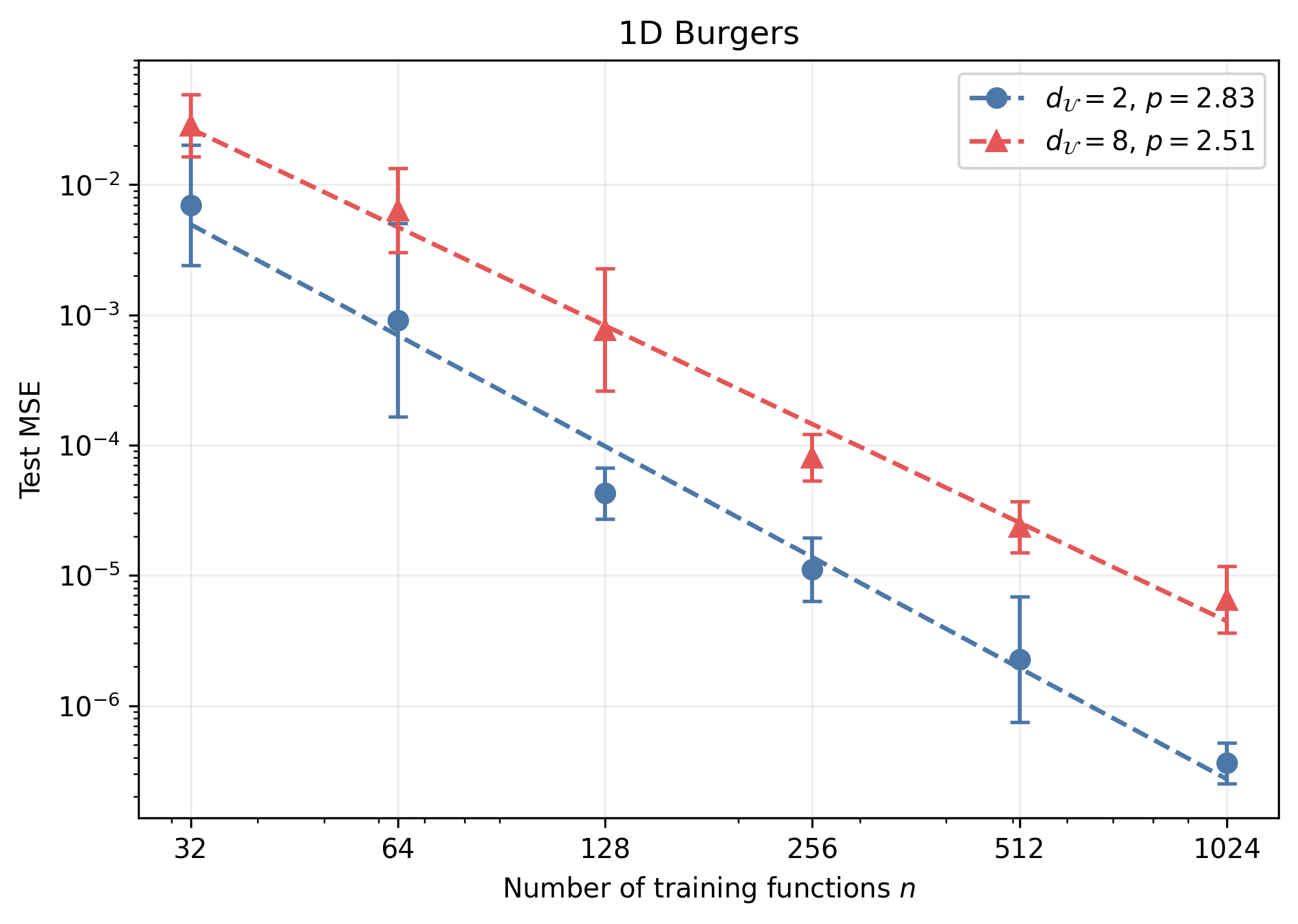}
    \end{minipage}
    \hspace{0.04\linewidth}
    \begin{minipage}[t]{0.47\linewidth}
        \centering
        \includegraphics[width=\linewidth]
        {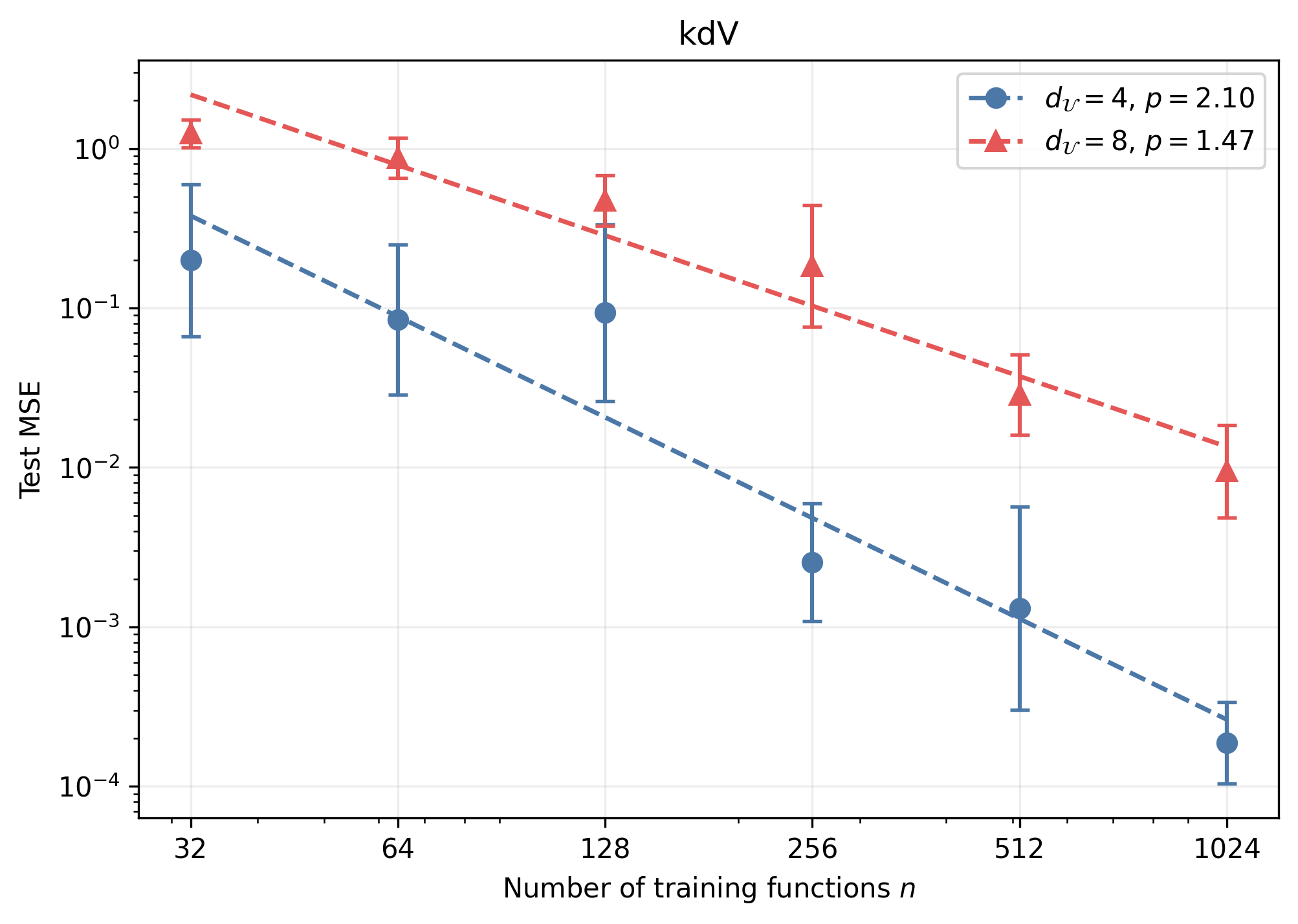}
    \end{minipage}
    \caption{
    Empirical data scaling for Burgers' (left) and KdV (right).
    Markers show the geometric mean test MSE over five repeats, with
    error bars corresponding to one sample standard deviation of the
    log MSE; dashed lines show the fitted power law
    $\mathrm{MSE}\approx Cn^{-p}$.
    }
    \label{fig:pde_scaling}
    \vspace{-0.2cm}
\end{figure}

Figure~\ref{fig:pde_scaling} shows clear power-law decay over
$n=32$--$1024$ for both equations. Burgers' yields $p=2.83$ for
$d_{\cU}=2$ and $p=2.51$ for $d_{\cU}=8$, while KdV yields
$p=2.10$ for $d_{\cU}=4$ and $p=1.47$ for $d_{\cU}=8$.
Despite the distinct dissipative and dispersive dynamics, both problems
exhibit the same ordering: the higher-dimensional input family has the
smaller empirical decay exponent. This trend is qualitatively consistent
with Theorem~\ref{thm:generalization_bound}, where increasing $d_{\cU}$
increases $s$ and moves the exponent $-2/(2+s)$ toward zero; the
empirical exponents are not interpreted as quantitative estimates of
the theoretical upper-bound rate.

\vspace{-0.2cm}
\section{Conclusion}
\vspace{-0.2cm}

Conclusion: We establish quantitative approximation and generalization theories for Transformer-based operator learning through a two-level Softmax POU construction and its explicit Transformer realization.   The resulting power model and data scaling laws explicitly capture the operator regularity and the regularity and complexity of the input and output function spaces. Experiments on Burgers' and KdV equations exhibit the predicted quantitative data scaling law and its dependence on the input intrinsic dimension.

Limitation: The resulting scaling laws in this paper are constructive upper bounds and are not intended to characterize optimal exponents or the performance of every trained architecture. Sharpening these bounds and understanding how architectural choices and optimization affect the observed scaling behavior are future-work directions.

\subsection*{Acknowledgement}

Haoran Yan, Zhongjie Shi and Wenjing Liao acknowledge support from the National Science Foundation un-
der the NSF DMS 2145167 and the U.S. Department of Energy under the DOE SC0024348. Peng Chen acknowledges support from the National Science Foundation under award NSF DMS-2245111 and the U.S. Department of Energy under contract DE-AC05-00OR22725. Yuanzhe Xi acknowledges support from the National Science Foundation under award NSF DMS-2038118 and DMS-2513118.

\subsection*{AI use statement}

In this work, we used generative AI tools to assist with polishing writing, finding related literature for references, designing the research and experiment idea, examining mathematical derivations, and implementation of numerical experiments. In particular, AI tools were used to assist with the design and analysis of the numerical experiments, the development and debugging of code for synthetic PDE datasets including Burgers' and KdV, and the verification and refinement of theoretical arguments. All AI-assisted mathematical arguments, code, numerical results, and manuscript content were independently reviewed and verified by the authors. We take responsibility for the final content of this work.

%\subsection*{Reproducibility statement}

%Complete proofs of the theoretical results are provided in the appendix. Additional details on the PDE benchmarks, input-function families, Transformer architecture, training protocol, and evaluation procedure are provided in Appendix~\ref{app:additional_experiments}.

\bibliography{iclr2027_conference}
\bibliographystyle{iclr2027_conference}

\newpage

\appendix

\section*{Appendix}

\section{Additional Definitions, Transformer Architecture, and Extensions}
\label{app:definitions_architecture_extensions}

This section collects additional definitions and Transformer architecture details used throughout the analysis, and provides the extension of our
approximation result to band-limited input functions.

\subsection{Additional Definitions}
\label{app:additional_definitions}

\begin{definition}[H\"older Norm]
\label{def:holder_norm}
Let $\alpha=m_\alpha+\nu_\alpha$, where
$m_\alpha\in\mathbb N_0$ and $\nu_\alpha\in(0,1]$.
For a function $f:\Omega\to\mathbb R$, define
\begin{equation}
\label{eq:holder_norm_definition}
\|f\|_{C^\alpha(\Omega)}
:=
\sum_{|\mathbf r|\le m_\alpha}
\|\partial^{\mathbf r}f\|_{L^\infty(\Omega)}
+
\max_{|\mathbf r|=m_\alpha}
\sup_{\mathbf x\neq\mathbf x'}
\frac{
|\partial^{\mathbf r}f(\mathbf x)
-\partial^{\mathbf r}f(\mathbf x')|
}{
\|\mathbf x-\mathbf x'\|_2^{\nu_\alpha}
},
\end{equation}
where $\mathbf r$ is a multi-index.
The same convention is used for $C^\beta(\Omega_{\cV})$.
\end{definition}

\begin{definition}[Covering Number]
\label{def:covering_number}
Let $(\mathcal F,d)$ be a metric space and let $\eta>0$.
The covering number $\mathcal N(\eta,\mathcal F,d)$ is the minimum
cardinality of an $\eta$-cover of $\mathcal F$ with respect to $d$.

For scalar-valued predictors considered in this paper, we use the uniform norm
$$
\|T\|_{\infty,\infty} :=
\sup_{u\in\cU,\;\yb\in\Omega_{\cV}}
|T(S_{n_x}(u),\yb)|.
$$
\end{definition}

\begin{definition}[Softmax Map]
\label{def:softmax}
For $\mathbf a=(a_1,\ldots,a_m)^\top\in\mathbb R^m$, define
\begin{equation}
\label{eq:softmax_definition}
[\operatorname{softmax}(\mathbf a)]_j
:=
\frac{\exp(a_j)}
{\sum_{r=1}^{m}\exp(a_r)},
\qquad j\in[m].
\end{equation}
\end{definition}

\subsection{Transformer Architecture}
\label{app:transformer_architecture}

We provide the precise definitions of the multi-head self-attention and
point-wise feedforward maps used in
Section~\ref{subsec:transformer_architecture}.
For
$
\Zb=[\zb^1,\ldots,\zb^P]\in\RR^{D\times P},
$
the $h$-th attention head in block $\ell$ uses
$$
\Qb_\ell^h,\Kb_\ell^h\in\RR^{d_k^\ell\times D},
\qquad
\Vb_\ell^h\in\RR^{d_v^\ell\times D}.
$$
The attention score matrix is defined by
\begin{equation}
\label{eq:attention_score}
\Sb_\ell^h(\Zb)
:=
(\Kb_\ell^h\Zb)^\top(\Qb_\ell^h\Zb)
\in\RR^{P\times P}.
\end{equation}
Using the Softmax map defined in
\eqref{eq:softmax_definition}, the corresponding column-wise attention
matrix is
\begin{equation}
\label{eq:attention_weights}
[\Ab_\ell^h(\Zb)]_{:,j}
:=
\operatorname{softmax}
\left(
[\Sb_\ell^h(\Zb)]_{:,j}
\right),
\qquad j\in[P].
\end{equation}
The $h$-th attention head is then
\begin{equation}
\label{eq:attention_head}
\operatorname{head}_\ell^h(\Zb)
:=
\Vb_\ell^h\Zb\,\Ab_\ell^h(\Zb)
\in\RR^{d_v^\ell\times P}.
\end{equation}
With
$
\Wb_\ell^O\in
\RR^{D\times(H^\ell d_v^\ell)},
$
the multi-head self-attention map is
\begin{equation}
\label{eq:mha_definition}
A_\ell(\Zb)
:=
\Wb_\ell^O
\begin{bmatrix}
\operatorname{head}_\ell^1(\Zb)
\\
\vdots
\\
\operatorname{head}_\ell^{H^\ell}(\Zb)
\end{bmatrix}
\in\RR^{D\times P}.
\end{equation}

The point-wise feedforward map is defined by
\begin{equation}
\label{eq:ffn_definition}
F_\ell(\Zb)
:=
\Wb_\ell^2
\sigma\!\left(
\Wb_\ell^1\Zb+\bb_\ell^1\mathbf 1_P^\top
\right)
+
\bb_\ell^2\mathbf 1_P^\top,
\end{equation}
where
$$
\Wb_\ell^1\in\RR^{d_{\rm ff}^\ell\times D},
\qquad
\Wb_\ell^2\in\RR^{D\times d_{\rm ff}^\ell},
$$
and $\sigma(t)=\max\{t,0\}$ is applied entrywise.

Therefore, the $\ell$-th encoder block used in our analysis is
\begin{equation}
\label{eq:encoder_block_definition}
\widehat{\Zb}_\ell
=
A_\ell(\Zb_{\ell-1}),
\qquad
\Zb_\ell
=
F_\ell(\widehat{\Zb}_\ell).
\end{equation}

\subsection{Extension to Band-limited Input Functions}
\label{app:bandlimited}

Recall from Theorem~\ref{thm:transformer_operator_approx} that the
complexity exponent $s$ defined in \eqref{eq:complexity_exponent} is
$$
s
=
\frac{d_{\cU}}{\gamma}
+
\frac{d_2}{\beta}
+
\frac{2d_1}{\alpha}
\left(1+\frac{2}{\gamma}\right).
$$
The last term, $\frac{2d_1}{\alpha}\left(1+\frac{2}{\gamma}\right)$, arises from the discretization of the input functions. Indeed, for a general $C^\alpha$ input class, controlling the quadrature error requires the number of input samples $n_x$ to grow as the target accuracy $\epsilon$ decreases. Since the Transformer construction in Theorem~\ref{thm:transformer_operator_approx} has parameter complexity of order $n_x^2P$, this discretization requirement contributes directly to the model scaling exponent.

We next consider a band-limited input class, for which the required functional inner products can be recovered exactly from a fixed number of samples. In this case, $n_x$ no longer grows with $\epsilon$, and the input-discretization contribution disappears from the polynomial scaling exponent.

\begin{corollary}[Scaling for Band-limited Input Functions]
\label{cor:bandlimited_scaling}
Suppose the assumptions of Theorem~\ref{thm:transformer_operator_approx} hold. In addition, assume that every $u\in\cU$ is periodic on $\Omega_{\cU}=[0,1]^{d_1}$ and has Fourier support contained in $\left\{ \mathbf k\in\mathbb Z^{d_1}: \|\mathbf k\|_\infty\le B \right\}$, where $B$ is fixed independently of $\epsilon$.

Then the input grid in Setting~\ref{setting:data} can be chosen with $n_x=n_x(B,d_1)$ independent of $\epsilon$ such that for $u,u_k\in\cU$, $\langle u,u_k\rangle_{n_x} = \langle u,u_k\rangle_{L^2(\Omega_{\cU})}$. So the quadrature term in Lemma~\ref{lemma:total_approximation_error} vanishes. The Transformer construction of Theorem~\ref{thm:transformer_operator_approx} then satisfies $N_{\rm total} \le C_{\rm band} \epsilon^{-s_{\rm band}}$, up to logarithmic factors, where $s_{\rm band} := \frac{d_{\cU}}{\gamma} + \frac{d_2}{\beta}$.
\end{corollary}

\paragraph{Why the rate improves}
For $u,u_k\in\cU$, the product $uu_k$ has Fourier bandwidth at most $2B$. Hence a sufficiently fine tensor-product uniform grid integrates $uu_k$ exactly, so the quadrature error in Lemma~\ref{lemma:quadrature} vanishes and $n_x$ can be chosen independently of $\epsilon$. The remaining approximation argument is unchanged. In particular, $P=C_UC_V+1 = O\left( \epsilon^{-\frac{d_{\cU}}{\gamma}-\frac{d_2}{\beta}}\right)$. Since $n_x$ is now fixed, the parameter bound $N_{\rm total}=O(n_x^2P)$ reduces to $N_{\rm total}=O(P)$, which gives us $s_{\rm band}$.

\section{Proof of Theorem~\ref{thm:transformer_operator_approx}:
Approximation and Model Scaling}
\label{app:approximation_model_scaling}

Following the proof sketch in Section~\ref{sec:proof_sketch}, Appendix~\ref{app:oracle_pou} develops the two-level Softmax POU oracle approximation, Appendix~\ref{app:transformer_realization} establishes its Transformer realization, and Appendix~\ref{app:proof:transformer_operator_approx} combines these ingredients to prove Theorem~\ref{thm:transformer_operator_approx} and derive the model scaling law.

\subsection{Oracle Approximation via Two-Level Softmax POU}
\label{app:oracle_pou}

This subsection provides the supporting lemmas for Step~1 of the proof
sketch in Section~\ref{sec:proof_sketch}. Recall that the oracle
approximation is constructed using an input-function-space Softmax POU
and an output-domain Softmax POU, with localization parameters $M_u$ and $M_y$ defined in \eqref{eq:Mu_definition} and \eqref{eq:My_definition}, respectively.

\begin{lemma}[Input-function-space Softmax POU Approximation]
\label{lemma:input_space_pou}
Suppose Assumptions~\ref{assum:operator}, \ref{assum:input_space}, and \ref{assum:output_space} hold. Fix $r_u\in(0,1]$. Let $\{u_k\}_{k=1}^{C_U} \subset \mathcal U$ be an $r_u$-cover of $\mathcal{U}$ in $L^2(\Omega_{\mathcal{U}})$, where $C_U \le C_1 r_u^{-d_{\mathcal{U}}}$. Assume $r_u$ is sufficiently small so that $2B_{\cV}C_U \ge L_G r_u^\gamma$, and define
\begin{equation}
\label{eq:Mu_definition}
M_u
:=
\frac{1}{3r_u^2}
\log\left(
\frac{2B_{\cV}C_U}{L_Gr_u^\gamma}
\right).
\end{equation}

For $k\in[C_U]$, define the continuous input-space Softmax POU weight by
\begin{equation}
\label{eq:beta_k}
\widetilde\beta_k(u) :=
\frac{\exp\left(2M_u\langle u,u_k\rangle_{L^2(\Omega_{\cU})} - M_u\|u_k\|_{L^2(\Omega_{\cU})}^2 \right)}{\sum_{k'=1}^{C_U}\exp\left(2M_u\langle u,u_{k'}\rangle_{L^2(\Omega_{\cU})} - M_u\|u_{k'}\|_{L^2(\Omega_{\cU})}^2 \right)},
\end{equation}
where $\langle u, u_k \rangle_{L^2} = \int_{\Omega_{\cU}} u(\mathbf{x}) u_k(\mathbf{x}) d\mathbf{x}$. Define the input-space POU approximation by
\begin{equation}
\label{eq:input_space_pou_operator}
G_{r_u}(u) := \sum_{k=1}^{C_U} \widetilde\beta_k(u)G(u_k).
\end{equation} 
Then
\begin{equation}
\label{eq:input_space_pou_error}
\sup_{u\in\mathcal U}
\|G(u)-G_{r_u}(u)\|_{L^2(\rho_y)}
\le
(2^\gamma+1)L_Gr_u^\gamma.
\end{equation}
\end{lemma}

The proof of Lemma~\ref{lemma:input_space_pou} is given in
Appendix~\ref{app:proof:lemma:input_space_pou}.

Since each expert $G(u_k)$ is $\beta$-Hölder continuous (Assumption \ref{assum:output_space}), we next approximate it using an output-domain Softmax POU over $\Omega_{\mathcal{V}}$.

\begin{lemma}[Output-domain POU Approximation]
\label{lemma:output_domain_pou}
Suppose Assumption~\ref{assum:output_space} holds. Fix $r_y\in(0,1]$. Let $\{\mathbf{z}_l\}_{l=1}^{C_V} \subset\Omega_{\mathcal V}$ be an $r_y$-cover of $\Omega_{\mathcal{V}}$ under the Euclidean metric, and define
\begin{equation}
\label{eq:My_definition}
M_y
:=
\frac{1}{3r_y^2}
\log\left(
\frac{2C_V}{r_y^\beta}
\right).
\end{equation}
Let the output-domain Softmax POU weights
$\{\widetilde\eta_l\}_{l=1}^{C_V}$ be defined as in
\eqref{eq:eta_l}.

For each $k \in [C_U]$, define the output-domain POU approximation by
\begin{equation}
\label{eq:output_domain_pou_approximation}
\widetilde v_k(\yb) := \sum_{l=1}^{C_V}G(u_k)(\zb_l)\widetilde\eta_l(\yb).
\end{equation} Then
\begin{equation}
\label{eq:output_domain_pou_error}
\sup_{k\in[C_U]}
\|G(u_k)-\widetilde v_k\|_{L^\infty(\Omega_{\cV})} \le (2^\beta+1)B_{\cV}r_y^\beta.
\end{equation}
\end{lemma}

Lemma \ref{lemma:output_domain_pou} is proved in Appendix \ref{app:proof:lemma:output_domain_pou}. 

Combining the input-space and output-domain POU approximations, we define the continuous two-level Softmax POU approximation by
\begin{equation}
\label{eq:continuous_two_level_pou}
\widetilde G_{r_u,r_y}(u)(\yb) := \sum_{k=1}^{C_U}\sum_{l=1}^{C_V} \widetilde\beta_k(u)\widetilde\eta_l(\yb)G(u_k)(\zb_l).
\end{equation}

The input-space weights in \eqref{eq:beta_k} depend on functional inner products. Since the Transformer observes $u$ only through its values on a fixed input grid, we next replace these inner products by deterministic quadrature.

\begin{lemma}[Quadrature Error for Input Anchors]
\label{lemma:quadrature}
Suppose Assumption~\ref{assum:input_space} holds. Write $\alpha=m_\alpha+\nu$ for $m_\alpha\in\mathbb N_0$, and $\nu\in(0,1]$. Let $s_{\rm quad}:=\left\lceil\frac{m_\alpha+1}{2}\right\rceil$. For every $n_x\ge(2s_{\rm quad})^{d_1}$, define $q:=\left\lfloor\frac{n_x^{1/d_1}}{s_{\rm quad}}\right\rfloor$. 

Partition $\Omega_{\cU}=[0,1]^{d_1}$ into $q^{d_1}$ congruent cubes and apply the tensor-product $s_{\rm quad}$-point Gauss--Legendre quadrature rule on each cube. This gives $(s_{\rm quad}q)^{d_1}\le n_x$ quadrature nodes. If $(s_{\rm quad}q)^{d_1}<n_x$, append distinct points with zero weights so that the resulting fixed input grid has exactly $n_x$ points, $X_{n_x}$ in Setting~\ref{setting:data}. The associated quadrature weights $w_1,\ldots,w_{n_x}$ are nonnegative and satisfy 
\begin{equation}
\label{eq:quadrature_weight_stability}
w_j\ge 0,
\qquad
\sum_{j=1}^{n_x} w_j = 1,
\end{equation} 
Moreover, for every $h\in C^\alpha(\Omega_{\cU})$,
$$
\left|\int_{\Omega_{\cU}}h(\xb)\,d\xb - \sum_{j=1}^{n_x}w_jh(\xb_j) \right|
\le
C_{\rm quad,0} \|h\|_{C^\alpha(\Omega_{\cU})}n_x^{-\alpha/d_1},
$$
where $C_{\rm quad,0}>0$ depends only on $d_1$ and $\alpha$. Consequently, for the input-space cover $\{u_k\}_{k=1}^{C_U}\subset\cU$ used in Lemma~\ref{lemma:input_space_pou},
\begin{equation}
\label{eq:quadrature_error}
\sup_{u\in\cU}\max_{k\in[C_U]}
\left|\langle u,u_k\rangle_{L^2(\Omega_{\cU})} - \langle u,u_k\rangle_{n_x}\right|
\le
C_{\rm quad}n_x^{-\alpha/d_1},
\end{equation}
where $\langle u,u_k\rangle_{n_x} = \sum_{j=1}^{n_x}w_j u(\xb_j)u_k(\xb_j)$,
and $C_{\rm quad}>0$, defined explicitly in \eqref{eq:Cquad_definition}, depends only on $d_1$, $\alpha$, $B_{\cU}$, and the quadrature construction, but not on $u$, $u_k$, $C_U$, or $n_x$.
\end{lemma}

Lemma \ref{lemma:quadrature} is proved in Appendix \ref{app:proof:lemma:quadrature}.

% We next define the fully discrete two-level Softmax POU approximation by
% \begin{equation}
% \label{eq:discrete_two_level_pou}
% \widetilde{G}_{r_u, r_y, n_x}(u)(\yb) = \sum_{k=1}^{C_U} \left( \sum_{l=1}^{C_V} G(u_k)(\zb_l) \widetilde{\eta}_l(\yb) \right) \widetilde{\beta}_{k, n_x}(u),
% \end{equation}
% where the discrete input-space Softmax POU weight is
% \begin{equation}
% \label{eq:discrete_input_space_weight}
% \widetilde\beta_{k,n_x}(u) := \frac{\exp\left(2M_u\langle u,u_k\rangle_{n_x} - M_u\|u_k\|_{L^2(\Omega_{\cU})}^2 \right)}{\sum_{k'=1}^{C_U}
% \exp\left(2M_u\langle u,u_{k'}\rangle_{n_x} - M_u\|u_{k'}\|_{L^2(\Omega_{\cU})}^2 \right)}.
% \end{equation}
% The approximator in \eqref{eq:discrete_two_level_pou} serves as an
% oracle approximation of the target operator, with two-level Softmax POU approximations in the input space and output domain respectively. We next derive an approximation error for this oracle approximation.

We use the following standard Lipschitz stability property \cite[Corollary~A.7]{edelman2022inductive}.

\begin{lemma}[Softmax Lipschitz Stability]
\label{lemma:softmax_lipschitz}
For any $\mathbf a,\mathbf b\in\mathbb R^m$,
$$
\|\operatorname{softmax}(\mathbf a)
-
\operatorname{softmax}(\mathbf b)\|_1
\le
2\|\mathbf a-\mathbf b\|_\infty.
$$
\end{lemma}

Combining the error bounds in Lemmas~\ref{lemma:input_space_pou}, \ref{lemma:output_domain_pou}, and~\ref{lemma:quadrature}, we obtain the following error bound for the discrete oracle approximation \eqref{eq:discrete_two_level_pou}.

\begin{lemma}[Total Approximation Error]
\label{lemma:total_approximation_error}
Suppose Assumptions~\ref{assum:operator}, \ref{assum:input_space}, and \ref{assum:output_space} hold, and suppose the fixed quadrature rule satisfies Lemma~\ref{lemma:quadrature}. Let $r_u, r_y \in(0,1]$ , with $r_u$ satisfying the condition in
Lemma~\ref{lemma:input_space_pou}, and let $M_u$ and $M_y$ be defined as in \eqref{eq:Mu_definition} and \eqref{eq:My_definition}. Then
\begin{equation}
\label{eq:total_approximation_error}
\sup_{u\in\cU} \left\| G(u) - \widetilde G_{r_u,r_y,n_x}(u) \right\|_{L^2(\rho_y)}
\le
(2^\gamma+1)L_G r_u^\gamma + (2^\beta+1)B_{\cV}r_y^\beta + 4B_{\cV}M_uC_{\rm quad}n_x^{-\alpha/d_1}.
\end{equation}
\end{lemma}

Lemma \ref{lemma:total_approximation_error} is proved in Appendix \ref{app:proof:lemma:total_approximation_error}. 

\begin{remark}
In particular, since $M_u = \frac{1}{3r_u^2}\log\left(\frac{2B_{\cV}C_U}{L_Gr_u^\gamma}\right)$ and $C_U\le C_1r_u^{-d_{\cU}}$, we have $\log C_U \le \log C_1+d_{\cU}\log\frac{1}{r_u}$. Therefore, the dependence of the required input grid size $n_x$ on the intrinsic dimension $d_{\cU}$ arises through the logarithmic factor in $M_u$.
\end{remark}

\subsection{Transformer Realization of the Oracle Approximation}
\label{app:transformer_realization}

We now construct a Transformer network with $L=2$ encoder blocks to
realize the discrete two-level Softmax POU approximation. Let the
embedding dimension be $D=d_1+d_2+n_x+9$, and let the total sequence
length be $P=C_UC_V+1$. We assume $P\ge n_x+2$, so that the first
$n_x$ columns store the discretized input information, the
$(n_x+1)$-st column stores the query point, and the last column is
reserved as the null column. After the first encoder block, the first
$P-1=C_UC_V$ active columns are indexed by anchor pairs
$(k,l)\in[C_U]\times[C_V]$, with $j(k,l)=(k-1)C_V+l$.

\begin{lemma}[Pair-token Preprocessing]
\label{lemma:pair_token_preprocess}
Let $S_{n_x}(u)$ be the discretized representation of $u\in\cU$, and
let $\yb\in\Omega_{\cV}$ be the query point. Set
$$
\theta_j:=\frac{2\pi j}{P},
\qquad j\in[P].
$$
There exists a fixed preprocessing operator $\mathcal P$ such that
$\Zb_0=\mathcal P(S_{n_x}(u),\yb)\in\RR^{D\times P}$ has the form
\begin{equation}
\label{eq:pair_token_preprocess_z0}
\Zb_0=
\begin{bmatrix}
\xb_1 & \cdots & \xb_{n_x} & \mathbf 0_{d_1} & \mathbf 0_{d_1\times(P-n_x-2)} & \mathbf 0_{d_1}
\\
\mathbf 0_{d_2} & \cdots & \mathbf 0_{d_2} & \yb & \mathbf 0_{d_2\times(P-n_x-2)} & \mathbf 0_{d_2}
\\
u(\xb_1) & \cdots & u(\xb_{n_x}) & 0 & \mathbf 0_{1\times(P-n_x-2)} & 0
\\
0 & \cdots & 0 & 0 & \mathbf 0_{1\times(P-n_x-2)} & 1
\\
\mathbf 0_{n_x+3} & \cdots & \mathbf 0_{n_x+3} & \mathbf 0_{n_x+3} & \mathbf 0_{(n_x+3)\times(P-n_x-2)} & \mathbf 0_{n_x+3}
\\
1 & \cdots & 1 & 1 & \mathbf 1_{1\times(P-n_x-2)} & 1
\\
1 & \cdots & 1 & 1 & \mathbf 1_{1\times(P-n_x-2)} & 0
\\
\sin(\theta_1) & \cdots & \sin(\theta_{n_x}) & \sin(\theta_{n_x+1}) & \sin(\theta_{n_x+2}),\ldots,\sin(\theta_{P-1}) & \sin(\theta_P)
\\
\cos(\theta_1) & \cdots & \cos(\theta_{n_x}) & \cos(\theta_{n_x+1}) & \cos(\theta_{n_x+2}),\ldots,\cos(\theta_{P-1}) & \cos(\theta_P)
\end{bmatrix}.
\end{equation}
The $(d_1+d_2+2)$-th row is the null-column indicator
$(0,\ldots,0,1)$. This row is used only for routing in the first MHA
block and is not preserved after that block.
\end{lemma}

Lemma~\ref{lemma:pair_token_preprocess} is proved in Appendix~\ref{app:proof:lemma:pair_token_preprocess}.

\begin{lemma}[Parallel Pair-token Feature Construction via MHA]
\label{lemma:pair_token_affine_mha}
Let $\Zb_0\in\RR^{D\times P}$ be defined as in
Lemma~\ref{lemma:pair_token_preprocess}.
For every active column $j\in[P-1]$, write
$$
k_j:=\left\lceil\frac{j}{C_V}\right\rceil,
\qquad
l_j:=j-(k_j-1)C_V.
$$
Define the target feature terms
\begin{align*}
B_{i,k}(u)
&:=2M_u w_i u_k(\xb_i)u(\xb_i),
&& i\in[n_x],\ k\in[C_U],\\
N_k
&:=M_u\|u_k\|_{L^2(\Omega_{\cU})}^2,
&& k\in[C_U], \\
Q_l(\yb)
&:=2M_y\yb^\top\zb_l-M_y\|\zb_l\|_2^2,
&& l\in[C_V],\\
A_{k,l}
&:=G(u_k)(\zb_l),
&& k\in[C_U],\ l\in[C_V].
\end{align*}

Set $c_P:=1-\cos\left(\frac{2\pi}{P}\right)$ and define
\begin{equation}
\label{eq:B_aff_definition}
B_{\rm aff} := 1 + (2B_{\cU}+3B_{\cU}^2)M_u + 3d_2M_y + B_{\cV}.
\end{equation}
For $M>1$, define
$$
\varepsilon_{\rm feat} := 2P^2B_{\rm aff}e^{-M}.
$$
If $Pe^{-M}\le \frac14$, then there exists a first-block MHA layer
$A_1:\RR^{D\times P}\to\RR^{D\times P}$ with
$$
H^1=(n_x+3)P+2,
\qquad
d_k^1=5,
\qquad
d_v^1=2,
$$
such that
$$
\widehat{\Zb}_1=A_1(\Zb_0)
$$
has the following column representations.
\begin{equation}
\label{eq:first_block_mha_output}
(\widehat{\Zb}_1)_{:,j}
=
\begin{bmatrix}
\widetilde B_{1,k_j,l_j}\\
\vdots\\
\widetilde B_{n_x,k_j,l_j}\\
\widetilde Q_{k_j,l_j}\\
\widetilde N_{k_j,l_j}\\
\widetilde A_{k_j,l_j}\\
\mathbf 0_{D-n_x-7}\\
1\\
a_M\\
\lambda_{\rm id}\sin(\theta_j)\\
\lambda_{\rm id}\cos(\theta_j)
\end{bmatrix},
\quad j\in[P-1],
\qquad
(\widehat{\Zb}_1)_{:,P}
=
\begin{bmatrix}
\delta_1^B\\
\vdots\\
\delta_{n_x}^B\\
\delta^Q\\
\delta^N\\
\delta^A\\
\mathbf 0_{D-n_x-7}\\
1\\
0\\
\lambda_{\rm id}\sin(\theta_P)\\
\lambda_{\rm id}\cos(\theta_P)
\end{bmatrix}.
\end{equation}

The output components satisfy the following bounds.
\begin{enumerate}
\item \textbf{Feature outputs.}
For every $u\in\cU$, $\yb\in\Omega_{\cV}$, and $j\in[P-1]$,
$$
\max\left\{
\max_{i\in[n_x]}
\left|
\widetilde B_{i,k_j,l_j}-B_{i,k_j}(u)
\right|,
\left|
\widetilde Q_{k_j,l_j}-Q_{l_j}(\yb)
\right|,
\left|
\widetilde N_{k_j,l_j}-N_{k_j}
\right|,
\left|
\widetilde A_{k_j,l_j}-A_{k_j,l_j}
\right|
\right\}
\le
\varepsilon_{\rm feat}.
$$
Moreover, the feature leakage in the null column satisfies
$$
\max\left\{
\max_{i\in[n_x]}|\delta_i^B|,
|\delta^Q|,
|\delta^N|,
|\delta^A|
\right\}
\le
\varepsilon_{\rm feat}.
$$

\item \textbf{Structural outputs.}
There exist scalars $a_M,\lambda_{\rm id}>0$ such that, for every
$j\in[P]$,
$$
(\widehat{\Zb}_1)_{D-3,j}=1,
\qquad
(\widehat{\Zb}_1)_{D-2,j} = a_M\,1_{\{j<P\}},
$$
and
$$
(\widehat{\Zb}_1)_{D-1,j} = \lambda_{\rm id}\sin(\theta_j),
\qquad
(\widehat{\Zb}_1)_{D,j} = \lambda_{\rm id}\cos(\theta_j).
$$
Furthermore, under $Pe^{-M}\le \frac14$, $a_M^{-1}\le 2$ and $\lambda_{\rm id}^{-1}\le 2$.

\item \textbf{Parameter bounds.}
Define
$$
M_{A_1}
:=
\max\left\{
\|\Wb_1^O\|_{\max},
\max_{h\in[H^1]}
\max\left\{
\|\Qb_1^h\|_{\max},
\|\Kb_1^h\|_{\max},
\|\Vb_1^h\|_{\max}
\right\}
\right\}.
$$
Then
$$
M_{A_1} \le \max\left\{1,\, B_{\rm aff},\, \frac{5M}{c_P}\right\}.
$$
\end{enumerate}
\end{lemma}

Lemma~\ref{lemma:pair_token_affine_mha} is proved in
Appendix~\ref{app:proof:lemma:pair_token_affine_mha}.

\begin{lemma}[Pair-token Joint-logit Assembly via Point-wise FFN]
\label{lemma:pair_token_joint_logit_ffn}
Let $\widehat{\Zb}_1\in\RR^{D\times P}$ be the output matrix from Lemma~\ref{lemma:pair_token_affine_mha}. For every active column $j\in[P-1]$, the ideal joint logit can be written as
\begin{equation}
\label{eq:ideal_joint_logit}
\begin{aligned}
\Xi_j
&:=
\sum_{i=1}^{n_x} B_{i,k_j}(u) - N_{k_j} + Q_{l_j}(\yb) \\
&=
2M_u\langle u,u_{k_j}\rangle_{n_x} - M_u\|u_{k_j}\|_{L^2(\Omega_{\cU})}^2 + 2M_y\yb^\top\zb_{l_j} - M_y\|\zb_{l_j}\|_2^2.
\end{aligned}
\end{equation}
Let
$$
M_{\rm out} = M+(n_x+2)\left(B_{\rm aff}+2\varepsilon_{\rm feat}\right).
$$
For the null column, define
$$
\widehat\Xi_P = \sum_{i=1}^{n_x}\delta_i^B -\delta^N+\delta^Q-M_{\rm out}.
$$
There exists a point-wise FFN layer
$F_1:\RR^{D\times P}\to\RR^{D\times P}$ with hidden width
$d_{\rm ff}^1=2D$ such that $\Zb_1=F_1(\widehat{\Zb}_1)$ is of the form
\begin{equation}
\label{eq:first_block_ffn_output}
\Zb_1
=
\begin{bmatrix}
\widehat\Xi_1 & \cdots & \widehat\Xi_{P-1} & \widehat\Xi_P
\\
\widetilde A_1 & \cdots & \widetilde A_{P-1} & \delta^A
\\
1 & \cdots & 1 & 1
\\
\multicolumn{4}{c}{\mathbf 0_{(D-3)\times P}}
\end{bmatrix},
\end{equation}
where, for $j=j(k_j,l_j)\in[P-1]$,
$$
\widehat\Xi_j := \sum_{i=1}^{n_x}\widetilde B_{i,k_j,l_j} - \widetilde N_{k_j,l_j} + \widetilde Q_{k_j,l_j},
\qquad
\widetilde A_j := \widetilde A_{k_j,l_j}.
$$
For every $j=j(k_j,l_j)\in[P-1]$,
\begin{equation}
\label{eq:joint_logit_ffn_error}
|\widehat\Xi_j-\Xi_j| \le (n_x+2)\varepsilon_{\rm feat},
\qquad
|\widetilde A_j-A_{k_j,l_j}| \le \varepsilon_{\rm feat}.
\end{equation}

For the null column,
\begin{equation}
\label{eq:joint_logit_null_separation}
\widehat\Xi_P - \max_{j\in[P-1]}\widehat\Xi_j \le -M,
\qquad
|\delta^A| \le \varepsilon_{\rm feat}.
\end{equation}

Moreover, defining
$$
M_{F_1} :=
\max\left\{
\|\Wb_1^1\|_{\max},
\|\mathbf b_1^1\|_\infty,
\|\Wb_1^2\|_{\max},
\|\mathbf b_1^2\|_\infty
\right\},
$$
the FFN parameters can be chosen so that
\begin{equation}
\label{eq:first_ffn_parameter_bound}
M_{F_1} \le 2M_{\rm out}.
\end{equation}

\end{lemma}

Lemma~\ref{lemma:pair_token_joint_logit_ffn} is proved in
Appendix~\ref{app:proof:lemma:pair_token_joint_logit_ffn}.

\begin{lemma}[Pair-token Softmax Aggregation via MHA]
\label{lemma:pair_token_softmax_aggregation}
Let $\Zb_1\in\RR^{D\times P}$ be the output matrix from
Lemma~\ref{lemma:pair_token_joint_logit_ffn}.
For $j\in[P-1]$, define the active-only Softmax weights
$$
\widehat\omega_j
:=
\frac{\exp(\widehat\Xi_j)}
{\sum_{q=1}^{P-1}\exp(\widehat\Xi_q)}.
$$
There exists a single-head MHA layer
$A_2:\RR^{D\times P}\to\RR^{D\times P}$ with
$H^2=1$ and $d_k^2=d_v^2=1$ such that
$\widehat{\Zb}_2=A_2(\Zb_1)$ has the form
\begin{equation}
\label{eq:pair_token_softmax_output}
\widehat{\Zb}_2 =
\begin{bmatrix}
\widehat g(u,\yb)
&
\widehat g(u,\yb)
&
\cdots
&
\widehat g(u,\yb)
\\
\multicolumn{4}{c}{\mathbf 0_{(D-1)\times P}}
\end{bmatrix},
\end{equation}
where
$$
\widehat g(u,\yb)
=
\sum_{j=1}^{P-1}\gamma_j\widetilde A_j
+
\gamma_P\delta^A,
$$
and
$$
\gamma_j
=
\frac{\exp(\widehat\Xi_j)}
{\sum_{q=1}^{P}\exp(\widehat\Xi_q)},
\qquad
j\in[P].
$$
The attention weights satisfy
$$
0\le\gamma_P\le e^{-M},
$$
and
$$
\left|
\widehat g(u,\yb)
-
\sum_{j=1}^{P-1}\widehat\omega_j\widetilde A_j
\right|
\le
\left(B_{\cV}+2\varepsilon_{\rm feat}\right)e^{-M}.
$$

Moreover, defining
$$
M_{A_2}
:=
\max\left\{
\|\Wb_2^O\|_{\max},
\|\Qb_2\|_{\max},
\|\Kb_2\|_{\max},
\|\Vb_2\|_{\max}
\right\},
$$
the MHA parameters can be chosen so that
$$
M_{A_2}\le1.
$$
\end{lemma}

Lemma~\ref{lemma:pair_token_softmax_aggregation} is proved in Appendix~\ref{app:proof:lemma:pair_token_softmax_aggregation}.

\subsection{Proof of Theorem~\ref{thm:transformer_operator_approx}}
\label{app:proof:transformer_operator_approx}

\begin{proof}[Proof of Theorem~\ref{thm:transformer_operator_approx}]
\phantomsection
\label{proof:transformer_operator_approx}
Choose $r_u := \left(\frac{\epsilon}{4(2^\gamma+1)L_G}\right)^{1/\gamma}$ and $r_y := \left(\frac{\epsilon}{4(2^\beta+1)B_{\cV}}\right)^{1/\beta}$. By choosing $\epsilon_0$ sufficiently small, we have $r_u,r_y\in(0,1]$. Let $\{u_k\}_{k=1}^{C_U}$ be the corresponding internal $r_u$-cover of $\cU$, and let $\{\zb_l\}_{l=1}^{C_V}$ be an $r_y$-cover of $\Omega_{\cV}$. Set $P:=C_UC_V+1$ and define $M_u$ and $M_y$ as in \eqref{eq:Mu_definition} and \eqref{eq:My_definition}.

Choose $n_x := \left\lceil \left(\frac{16B_{\cV}M_uC_{\rm quad}}{\epsilon}\right)^{d_1/\alpha}\right\rceil$. By taking $\epsilon_0$ sufficiently small, we have
$n_x\ge(2s_{\rm quad})^{d_1}$, so Lemma~\ref{lemma:quadrature} applies. We also use the sequence-length compatibility condition $P\ge n_x+2$. Then Lemma~\ref{lemma:total_approximation_error} gives
\begin{equation}
\label{eq:discrete_pou_approximation_error}
\sup_{u\in\cU}
\left\|
G(u)-\widetilde G_{r_u,r_y,n_x}(u)
\right\|_{L^2(\rho_y)}
\le
\frac{3\epsilon}{4}.
\end{equation}
% The target target features are uniformly bounded by $B_{\rm aff}$ as defined in \eqref{eq:B_aff_definition}.

\medskip
\noindent\textbf{Transformer construction.}
For each $u\in\cU$ and $\yb\in\Omega_{\cV}$,
Lemma~\ref{lemma:pair_token_preprocess} constructs the initial sequence
matrix
$$
\Zb_0 = \mathcal P(S_{n_x}(u),\yb) \in\RR^{D\times P},
$$
with the structure in \eqref{eq:pair_token_preprocess_z0}. The first
$n_x$ columns store the discretized input information, the $(n_x+1)$-st
column stores the query point $\yb$, and the last column is reserved as
the null column.

For Encoder Block~1, Lemma~\ref{lemma:pair_token_affine_mha}
first constructs
$$
\widehat{\Zb}_1=A_1(\Zb_0),
$$
whose active columns contain approximations of the target features
$B_{i,k}(u)$, $Q_l(\yb)$, $N_k$, and $A_{k,l}$.

Applying Lemma~\ref{lemma:pair_token_joint_logit_ffn} then gives
$$
\Zb_1=F_1(\widehat{\Zb}_1),
$$
with the form in \eqref{eq:first_block_ffn_output}. In particular, for
each active column $j=j(k,l)$, the first two coordinates are the assembled
joint logit $\widehat\Xi_j$ and the corresponding anchor value
$\widetilde A_j$, respectively. The null column satisfies the separation
property \eqref{eq:joint_logit_null_separation}.

For Encoder Block~2, applying Lemma~\ref{lemma:pair_token_softmax_aggregation} gives a single-head MHA
layer $\widehat{\Zb}_2 = A_2(\Zb_1)$ such that the output has the form \eqref{eq:pair_token_softmax_output}. In particular, the first coordinate of every column equals
$$
\widehat g(u,\yb) = \sum_{j=1}^{P-1} \gamma_j\widetilde A_j + \gamma_P\delta^A,
$$
where $\widetilde A_j=\widetilde A_{k_j,l_j}$ and $\gamma_1,\ldots,\gamma_P$ are the attention weights constructed in
Lemma~\ref{lemma:pair_token_softmax_aggregation}.

The second point-wise FFN is chosen as the identity map. Specifically,
set
$$
d_{\rm ff}^2=2D,
\qquad
\Wb_2^1 =
\begin{bmatrix}
I_D\\
-I_D
\end{bmatrix},
\qquad
\mathbf b_2^1=\mathbf 0_{2D},
$$
$$
\Wb_2^2 =
\begin{bmatrix}
I_D & -I_D
\end{bmatrix},
\qquad
\mathbf b_2^2=\mathbf 0_D.
$$
Using $x=\sigma(x)-\sigma(-x)$ componentwise,
$$
\Zb_2 = F_2(\widehat{\Zb}_2) = \widehat{\Zb}_2.
$$
Finally, the scalar output is obtained by the linear readout $\mathbf c^\top\operatorname{vec}(\Zb_2)$, where $\mathbf c\in\RR^{DP}$ is the first standard basis vector. Hence
$$
\mathbf c^\top\operatorname{vec}(\Zb_2) = (\Zb_2)_{1,1} = \widehat g(u,\yb).
$$

\medskip
\noindent\textbf{Error bound analysis.}
For an active column $j=j(k,l)\in[P-1]$, let $\Xi_j$ be the ideal joint logit defined in \eqref{eq:ideal_joint_logit}. Since it is the sum
of the input-space and output-domain logits,
$$
\frac{\exp(\Xi_j)}
{\sum_{q=1}^{P-1}\exp(\Xi_q)}
=
\widetilde\beta_{k,n_x}(u)\widetilde\eta_l(\yb).
$$
By Lemma~\ref{lemma:pair_token_joint_logit_ffn},
$$
\max_{j\in[P-1]}
|\widehat\Xi_j-\Xi_j|
\le
(n_x+2)\varepsilon_{\rm feat},
$$
and
$$
|\widetilde A_j-A_{k_j,l_j}|
\le
\varepsilon_{\rm feat}.
$$
Hence Lemma~\ref{lemma:softmax_lipschitz} gives
$$
\sum_{j=1}^{P-1} \left|\widehat\omega_j - \widetilde\beta_{k_j,n_x}(u)\widetilde\eta_{l_j}(\yb)\right|
\le
2(n_x+2)\varepsilon_{\rm feat}.
$$
Using $|\widetilde A_j-A_{k_j,l_j}| \le \varepsilon_{\rm feat}$, $|A_{k_j,l_j}| \le B_{\cV}$ and \eqref{eq:discrete_two_level_pou}, we obtain
$$
\left|\sum_{j=1}^{P-1}\widehat\omega_j\widetilde A_j - \widetilde G_{r_u,r_y,n_x}(u)(\yb)\right|
\le
\varepsilon_{\rm feat} + 2B_{\cV}(n_x+2)\varepsilon_{\rm feat}
=
\left(1+2B_{\cV}(n_x+2)\right)\varepsilon_{\rm feat}.
$$
Combining this with
Lemma~\ref{lemma:pair_token_softmax_aggregation} yields
$$
\left|
\widehat g(u,\yb)
-
\widetilde G_{r_u,r_y,n_x}(u)(\yb)
\right|
\le
\left(
1+2B_{\cV}(n_x+2)
\right)\varepsilon_{\rm feat}
+
\left(
B_{\cV}+2\varepsilon_{\rm feat}
\right)e^{-M}.
$$
Since $\varepsilon_{\rm feat} = 2P^2B_{\rm aff}e^{-M}$ and $M>1$, it follows that
$$
\left|\widehat g(u,\yb) - \widetilde G_{r_u,r_y,n_x}(u)(\yb) \right|
\le
\left[B_{\cV} + 2P^2B_{\rm aff}\left(3+2B_{\cV}(n_x+2)\right)\right]e^{-M}.
$$
Choose $M := \log\left(\frac{4\left[B_{\cV} + 2P^2B_{\rm aff}\left(3+2B_{\cV}(n_x+2)\right)\right]}{\epsilon}\right)$. Then
$$
\left|\widehat g(u,\yb) - \widetilde G_{r_u,r_y,n_x}(u)(\yb)\right|
\le
\frac{\epsilon}{4}.
$$
Since the preceding bound holds uniformly for $\yb\in\Omega_{\cV}$ and $\rho_y$ is a probability distribution,
$$
\left\|
\widehat g(u,\cdot)
-
\widetilde G_{r_u,r_y,n_x}(u)
\right\|_{L^2(\rho_y)}
\le
\frac{\epsilon}{4}.
$$
Moreover, since $B_{\rm aff}\ge1$, $P\ge1$, and $\epsilon\le1$,
this choice also gives $Pe^{-M}\le\frac14$, so the conditions of the preceding Transformer construction lemmas are satisfied.

Combining the implementation error with
\eqref{eq:discrete_pou_approximation_error} and using the triangle
inequality in $L^2(\rho_y)$ gives
$$
\begin{aligned}
\sup_{u\in\cU}
\left\|
T_\Theta(S_{n_x}(u),\cdot)-G(u)
\right\|_{L^2(\rho_y)}
&\le
\sup_{u\in\cU}
\left\|
T_\Theta(S_{n_x}(u),\cdot)
-
\widetilde G_{r_u,r_y,n_x}(u)
\right\|_{L^2(\rho_y)}
\\
&\quad+
\sup_{u\in\cU}
\left\|
\widetilde G_{r_u,r_y,n_x}(u)-G(u)
\right\|_{L^2(\rho_y)}
\\
&\le
\frac{\epsilon}{4}
+
\frac{3\epsilon}{4}
=
\epsilon.
\end{aligned}
$$

\medskip
\noindent\textbf{Parameter complexity.}
We count the total architectural (dense) parameters $N_{\rm total}$ of the constructed Transformer:

\begin{itemize}
\item \textbf{Pre-processing:} The shared affine embedding, its bias, and the fixed structural-positional matrix have $D(P+d_1+d_2+2)$ parameters.

\item \textbf{MHA$_1$:} With $H^1=(n_x+3)P+2$, $d_k^1=5$, and $d_v^1=2$, there are $14D\bigl((n_x+3)P+2\bigr)$ parameters.

\item \textbf{FFN$_1$:} Hidden dimension $2D$ gives $4D^2+3D$ parameters.

\item \textbf{MHA$_2$:} One head with $d_k^2=d_v^2=1$ gives $4D$ parameters.

\item \textbf{FFN$_2$:} Hidden dimension $2D$ gives $4D^2+3D$ parameters.

\item \textbf{Readout:} The linear projection over $\operatorname{vec}(\Zb_2)$ has $DP$ parameters.
\end{itemize}

Summing these contributions yields
$$
\begin{aligned}
N_{\rm total}
&=
D(P+d_1+d_2+2) + 14D\bigl((n_x+3)P+2\bigr) + 2(4D^2+3D) + 4D + DP
\\
&=
8D^2 + 14D(n_x+3)P + D(d_1+d_2+2P+2) + 38D.
\end{aligned}
$$
Since $n_x\ge1$, $P\ge1$, and $D=d_1+d_2+n_x+9 \le (d_1+d_2+10)n_x$, we obtain
\begin{align*}
N_{\rm total}
&\le
8(d_1+d_2+10)^2n_x^2P
+
56(d_1+d_2+10)n_x^2P
+
(d_1+d_2+10)(d_1+d_2+4)n_x^2P
+
38(d_1+d_2+10)n_x^2P
\\
&=
(d_1+d_2+10)(9d_1+9d_2+178)n_x^2P
\\
&=:C_{\rm arch}n_x^2P.
\end{align*}
where $C_{\rm arch} := (d_1+d_2+10)(9d_1+9d_2+178)$. By \eqref{eq:input_covering} and the choices of $r_u$ and $r_y$,
$$
C_U
\le
C_1
\left(
4(2^\gamma+1)L_G
\right)^{d_{\cU}/\gamma}
\epsilon^{-d_{\cU}/\gamma}.
$$
Since $\Omega_{\cV}\subseteq[0,1]^{d_2}$ is compact, it admits an
internal $r_y$-cover with
$$
C_V
\le
C_2 r_y^{-d_2},
$$
where $C_2>0$ depends only on $d_2$. Therefore, by the choice of $r_y$,
$$
C_V
\le
C_2
\left(
4(2^\beta+1)B_{\cV}
\right)^{d_2/\beta}
\epsilon^{-d_2/\beta}.
$$
for a constant $C_2>0$ depending only on $d_2$. Hence
\begin{equation}
\label{eq:P_epsilon_bound}
P=C_UC_V+1
\le
C_P
\epsilon^{-d_{\cU}/\gamma-d_2/\beta},
\end{equation}
where 
\begin{equation}
\label{eq:CP_definition}
C_P
:=
1+C_1C_2
\left(4(2^\gamma+1)L_G\right)^{d_{\cU}/\gamma}
\left(4(2^\beta+1)B_{\cV}\right)^{d_2/\beta}.
\end{equation}

By \eqref{eq:Mu_definition}, \eqref{eq:input_covering}, and the choice of
$r_u$, we have
$$
M_u
\le
\frac{1}{3}
\left(
\frac{4(2^\gamma+1)L_G}{\epsilon}
\right)^{2/\gamma}
\log\left[
\frac{2B_{\cV}C_1}{L_G}
\left(
\frac{4(2^\gamma+1)L_G}{\epsilon}
\right)^{(d_{\cU}+\gamma)/\gamma}
\right].
$$
Hence, for $\epsilon\le e^{-1}$,
\begin{equation}
\label{eq:Mu_epsilon_bound}
    M_u \le C_{M_u}\epsilon^{-2/\gamma}\log\frac1\epsilon,
\end{equation}
where
$$
C_{M_u}
=
\frac{\left(4(2^\gamma+1)L_G\right)^{2/\gamma}}{3}
\left[
\max\left\{
\log\left(
\frac{2B_{\cV}C_1}{L_G}
\left(4(2^\gamma+1)L_G\right)^{(d_{\cU}+\gamma)/\gamma}
\right),
0
\right\}
+
1+\frac{d_{\cU}}{\gamma}
\right].
$$
Therefore, by the choice of $n_x$,
\begin{equation}
\label{eq:nx_epsilon_bound}
n_x
\le
C_x\epsilon^{-\frac{d_1}{\alpha}(1+\frac{2}{\gamma})}\left(\log\frac1\epsilon\right)^{d_1/\alpha},
\end{equation}
where 
\begin{equation}
\label{eq:Cx_definition}
C_x
:=
2\left(
16B_{\cV}C_{\rm quad}C_{M_u}
\right)^{d_1/\alpha}.
\end{equation}

Combining \eqref{eq:P_epsilon_bound} and \eqref{eq:nx_epsilon_bound} with
$N_{\rm total}\le C_{\rm arch}n_x^2P$ gives
$$
N_{\rm total}
\le
C_N
\epsilon^{-s}
\left(\log\frac1\epsilon\right)^{2d_1/\alpha},
$$
where
\begin{equation}
\label{eq:CN_definition}
C_N:=C_{\rm arch}C_x^2C_P.
\end{equation}
\medskip
\noindent\textbf{Parameter magnitude.}
By Lemmas~\ref{lemma:pair_token_affine_mha},
\ref{lemma:pair_token_joint_logit_ffn}, and
\ref{lemma:pair_token_softmax_aggregation},
$$
M_{A_1}
\le
\max\left\{
1,\,
B_{\rm aff},\,
\frac{5M}{c_P}
\right\}.
$$
$$
M_{F_1}
\le
2M_{\rm out}
=
2M
+
2(n_x+2)
\left(
B_{\rm aff}+2\varepsilon_{\rm feat}
\right),
\qquad
M_{A_2}\le1.
$$
The preprocessing parameters, the second identity FFN, and the final
readout have magnitude at most $1$. By the choice of $M$ in the error-bound analysis,
$$
\varepsilon_{\rm feat}
\le
\frac{\epsilon}{12}
\le
\frac{1}{12}.
$$
Since $B_{\rm aff}\ge1$,
$$
B_{\rm aff}+2\varepsilon_{\rm feat}
\le
B_{\rm aff}+\frac{1}{6}
\le
\frac{7}{6}B_{\rm aff}.
$$
Using $n_x\ge1$, so that $n_x+2\le3n_x$, we obtain
$$
\begin{aligned}
M_{F_1}
&\le
2M
+
2(n_x+2)
\left(
B_{\rm aff}+2\varepsilon_{\rm feat}
\right)
\\
&\le
2M
+
2(3n_x)\frac{7}{6}B_{\rm aff}
\\
&=
2M+7n_xB_{\rm aff}.
\end{aligned}
$$
Moreover,
$$
c_P
=
2\sin^2\left(\frac{\pi}{P}\right)
\le2,
$$
and hence
$$
M
\le
2\frac{M}{c_P}.
$$
Therefore,
$$
M_{F_1}
\le
4\frac{M}{c_P}
+
7n_xB_{\rm aff}.
$$
Since $n_x\ge1$ and $B_{\rm aff}\ge1$,
$$
M_{A_1}
\le
5\max\left\{
\frac{M}{c_P},
n_xB_{\rm aff}
\right\}.
$$
Consequently,
$$
M_{\max}
\le
11\max\left\{
1,\frac{M}{c_P},n_xB_{\rm aff}
\right\}.
$$

The elementary bound $\sin(\pi/P)\ge2/P$ for $P\ge2$ gives
\begin{equation}
\label{eq:cP_inverse_bound}
\frac{1}{c_P}
\le
\frac{P^2}{8}.
\end{equation}

Moreover, by \eqref{eq:My_definition}, the choice of $r_y$, and
$C_V\le C_2r_y^{-d_2}$, for $\epsilon\le e^{-1}$,
\begin{equation}
\label{eq:My_epsilon_bound}
M_y
\le
C_{M_y}
\epsilon^{-2/\beta}
\log\frac1\epsilon,
\end{equation}
where
$$
C_{M_y}
=
\frac{\left(4(2^\beta+1)B_{\cV}\right)^{2/\beta}}{3}
\left[
\max\left\{
\log\left(
2C_2
\left(4(2^\beta+1)B_{\cV}\right)^{(d_2+\beta)/\beta}
\right),
0
\right\}
+
1+\frac{d_2}{\beta}
\right].
$$

By \eqref{eq:B_aff_definition},
\eqref{eq:Mu_epsilon_bound}, and
\eqref{eq:My_epsilon_bound}, define
$$
C_{\rm aff}
:=
1+B_{\cV}
+
(2B_{\cU}+3B_{\cU}^2)C_{M_u}
+
3d_2C_{M_y}.
$$
Then, since $\epsilon\le e^{-1}$,
\begin{equation}
\label{eq:Baff_epsilon_bound}
B_{\rm aff}
\le
C_{\rm aff}
\left(
\epsilon^{-2/\gamma}
+
\epsilon^{-2/\beta}
\right)
\log\frac1\epsilon.
\end{equation}

We next bound the routing parameter $M$. Set
$$
C_{\rm aux}
:=
B_{\cV}
+
4C_P^2C_{\rm aff}C_x
\left(3+6B_{\cV}\right).
$$
By the choice of $M$ in the error-bound analysis,
\eqref{eq:P_epsilon_bound},
\eqref{eq:nx_epsilon_bound}, and
\eqref{eq:Baff_epsilon_bound}, using $n_x+2\le3n_x$, we obtain
$$
B_{\cV}
+
2P^2B_{\rm aff}
\left(
3+2B_{\cV}(n_x+2)
\right)
\le
C_{\rm aux}
\epsilon^{-R_M}
\left(
\log\frac1\epsilon
\right)^{1+d_1/\alpha},
$$
where
$$
R_M
:=
2\left(
\frac{d_{\cU}}{\gamma}
+
\frac{d_2}{\beta}
\right)
+
\frac{d_1}{\alpha}
\left(
1+\frac{2}{\gamma}
\right)
+
\max\left\{
\frac{2}{\gamma},
\frac{2}{\beta}
\right\}.
$$
Therefore,
$$
M
\le
C_M\log\frac1\epsilon,
$$
where one may take
$$
C_M
:=
\max\left\{
1,\,
\max\left\{
\log(4C_{\rm aux}),
0
\right\}
+
R_M
+
\frac{d_1}{\alpha}
+
2
\right\}.
$$

Combining this bound with
\eqref{eq:cP_inverse_bound} and
\eqref{eq:P_epsilon_bound} gives
$$
\frac{M}{c_P}
\le
\frac{C_MC_P^2}{8}
\epsilon^{-2\left(
\frac{d_{\cU}}{\gamma}
+
\frac{d_2}{\beta}
\right)}
\log\frac1\epsilon.
$$
Moreover, by
\eqref{eq:nx_epsilon_bound} and
\eqref{eq:Baff_epsilon_bound},
$$
n_xB_{\rm aff}
\le
C_xC_{\rm aff}
\left[
\epsilon^{
-\frac{d_1}{\alpha}(1+\frac{2}{\gamma})
-\frac{2}{\gamma}}
+
\epsilon^{
-\frac{d_1}{\alpha}(1+\frac{2}{\gamma})
-\frac{2}{\beta}}
\right]
\left(
\log\frac1\epsilon
\right)^{1+d_1/\alpha}.
$$

Therefore,
$$
M_{\max}
\le
C_{\rm mag}
\epsilon^{-q_M}
\left(
\log\frac1\epsilon
\right)^{1+d_1/\alpha},
$$
where
$$
q_M
:=
\max\left\{
2\left(
\frac{d_{\cU}}{\gamma}
+
\frac{d_2}{\beta}
\right),
\frac{d_1}{\alpha}
\left(
1+\frac{2}{\gamma}
\right)
+
\frac{2}{\gamma},
\frac{d_1}{\alpha}
\left(
1+\frac{2}{\gamma}
\right)
+
\frac{2}{\beta}
\right\},
$$
and 
\begin{equation}
\label{eq:Cmag_definition}
C_{\rm mag}
:=
11\max\left\{
1,\,
\frac{C_MC_P^2}{8},\,
2C_xC_{\rm aff}
\right\}.
\end{equation}
\end{proof}

\section{Proof of Theorem~\ref{thm:generalization_bound}:
Generalization and Data Scaling}
\label{app:generalization_data_scaling}

Following the proof sketch in Section~\ref{sec:proof_sketch},
Appendix~\ref{app:covering_number} establishes the covering-number bound
for the Transformer class, which controls the variance term.
Appendix~\ref{app:proof:generalization_bound} then uses the approximation
bound from Theorem~\ref{thm:transformer_operator_approx} to control the
bias term, combines the resulting bias and variance bounds, and balances
them to prove Theorem~\ref{thm:generalization_bound} and derive the data
scaling law.

\subsection{Covering Number Bound for the Transformer Class}
\label{app:covering_number}

\begin{lemma}[Covering Number of the Transformer Class]
\label{lemma:transformer_covering}
Let $\cT_\epsilon$ denote the class of two-block Transformers with the
architecture and parameter bounds specified in Theorem~\ref{thm:transformer_operator_approx}, with fixed preprocessing
as in Lemma~\ref{lemma:pair_token_preprocess}, and let
$\cH_\epsilon^{\rm clip}$ be the corresponding clipped class.
Assume $P\ge n_x+2$ and $M_{\max}\ge1$. Then, for every
$\eta\in(0,1]$,
\begin{equation}
\label{eq:transformer_covering_bound}
\log\mathcal N
\left(
\eta,\cH_\epsilon^{\rm clip},
\|\cdot\|_{\infty,\infty}
\right)
\le
N_{\rm total}
\log\left(
\frac{
503552\,P^7D^{17}M_{\max}^{21}B_{\cU}^5
}{\eta}
\right).
\end{equation}
\end{lemma}

The proof of Lemma~\ref{lemma:transformer_covering} is given in
Appendix~\ref{app:proof:lemma:transformer_covering}.

\subsection{Proof of Theorem~\ref{thm:generalization_bound}}
\label{app:proof:generalization_bound}

\begin{proof}[Proof of Theorem~\ref{thm:generalization_bound}]
\phantomsection
\label{proof:generalization_bound}

By Theorem~\ref{thm:transformer_operator_approx}, there exists
$T_\epsilon^\ast\in\cT_\epsilon$ such that
$$
\sup_{u\in\cU}
\left\|
T_\epsilon^\ast(S_{n_x}(u),\cdot)-G(u)
\right\|_{L^2(\rho_y)}
\le
\epsilon.
$$
Since $|G(u)(\yb)|\le B_{\cV}$ for every
$u\in\cU$ and $\yb\in\Omega_{\cV}$, clipping does not increase the
pointwise approximation error. Hence
$\Pi_{B_{\cV}}\circ T_\epsilon^\ast\in\cH_\epsilon^{\rm clip}$ and
$$
\cL(\Pi_{B_{\cV}}\circ T_\epsilon^\ast)
\le
\epsilon^2.
$$
For simplicity, we continue to denote this clipped comparator by
$T_\epsilon^\ast$.

Define the empirical clean error by
$$
\|T-G\|_{\cS}^2
:=
\frac{1}{nn_y}
\sum_{i=1}^n
\sum_{j=1}^{n_y}
\left(
T(S_{n_x}(u_i),\yb_{i,j})
-
G(u_i)(\yb_{i,j})
\right)^2.
$$
Following the decomposition in the proof of Theorem~2 of
\cite{liu2026scaling}, write
$$
\EE_{\cS}\cL(\widehat T_\epsilon)
=
T_1+T_2,
$$
where
$$
T_1
:=
2\EE_{\cS}
\|\widehat T_\epsilon-G\|_{\cS}^2,
$$
and
$$
T_2
:=
\EE_{\cS}\cL(\widehat T_\epsilon)
-
2\EE_{\cS}
\|\widehat T_\epsilon-G\|_{\cS}^2.
$$
\medskip
\noindent\textbf{Bounding $T_1$.}
Recall that
$$
v_{i,j}
=
G(u_i)(\yb_{i,j})+\xi_{i,j}.
$$
By the ERM property of $\widehat T_\epsilon$,
$$
\frac{1}{nn_y}
\sum_{i=1}^n
\sum_{j=1}^{n_y}
\left(
\widehat T_\epsilon(S_{n_x}(u_i),\yb_{i,j})-v_{i,j}
\right)^2
\le
\frac{1}{nn_y}
\sum_{i=1}^n
\sum_{j=1}^{n_y}
\left(
T_\epsilon^\ast(S_{n_x}(u_i),\yb_{i,j})-v_{i,j}
\right)^2.
$$
Expanding $v_{i,j}=G(u_i)(\yb_{i,j})+\xi_{i,j}$ gives
$$
\|\widehat T_\epsilon-G\|_{\cS}^2
\le
\|T_\epsilon^\ast-G\|_{\cS}^2
+
\frac{2}{nn_y}
\sum_{i=1}^n
\sum_{j=1}^{n_y}
\left(
\widehat T_\epsilon(S_{n_x}(u_i),\yb_{i,j})
-
T_\epsilon^\ast(S_{n_x}(u_i),\yb_{i,j})
\right)
\xi_{i,j}.
$$
% Taking expectation over the training sample and using that
% $T_\epsilon^\ast$ is fixed independently of the observation noise,
% $$
% \EE_{\cS}
% \|T_\epsilon^\ast-G\|_{\cS}^2
% =
% \cL(T_\epsilon^\ast)
% \le
% \epsilon^2,
% $$
Writing
$$
\widehat T_\epsilon-T_\epsilon^\ast
=
(\widehat T_\epsilon-G)
-
(T_\epsilon^\ast-G),
$$
the preceding inequality becomes
$$
\begin{aligned}
\|\widehat T_\epsilon-G\|_{\cS}^2
&\le
\|T_\epsilon^\ast-G\|_{\cS}^2
\\
&\quad+
\frac{2}{nn_y}
\sum_{i=1}^n\sum_{j=1}^{n_y}
\left(
\widehat T_\epsilon(S_{n_x}(u_i),\yb_{i,j})
-
G(u_i)(\yb_{i,j})
\right)\xi_{i,j}
\\
&\quad-
\frac{2}{nn_y}
\sum_{i=1}^n\sum_{j=1}^{n_y}
\left(
T_\epsilon^\ast(S_{n_x}(u_i),\yb_{i,j})
-
G(u_i)(\yb_{i,j})
\right)\xi_{i,j}.
\end{aligned}
$$
Since $T_\epsilon^\ast$ is fixed independently of the training sample
and $\EE[\xi_{i,j}\mid u_i,\yb_{i,j}]=0$,
$$
\begin{aligned}
&\EE_{\cS}
\left[
\left(
T_\epsilon^\ast(S_{n_x}(u_i),\yb_{i,j})
-
G(u_i)(\yb_{i,j})
\right)\xi_{i,j}
\right]
\\
&\qquad=
\EE_{u_i,\yb_{i,j}}
\left[
\left(
T_\epsilon^\ast(S_{n_x}(u_i),\yb_{i,j})
-
G(u_i)(\yb_{i,j})
\right)
\EE[\xi_{i,j}\mid u_i,\yb_{i,j}]
\right]
=0.
\end{aligned}
$$
Moreover,
$$
\EE_{\cS}\|T_\epsilon^\ast-G\|_{\cS}^2
=
\cL(T_\epsilon^\ast)
\le
\epsilon^2.
$$
Therefore,
$$
\EE_{\cS}
\|\widehat T_\epsilon-G\|_{\cS}^2
\le
\epsilon^2
+
\frac{2}{nn_y}
\EE_{\cS}
\sum_{i=1}^n
\sum_{j=1}^{n_y}
\left(
\widehat T_\epsilon(S_{n_x}(u_i),\yb_{i,j})
-
G(u_i)(\yb_{i,j})
\right)
\xi_{i,j}.
$$
The argument of Lemma~3 in \cite{liu2026scaling} applies to $\cH_\epsilon^{\rm clip}$ under our sampling and noise assumptions. Thus, for every $\theta>0$,
$$
\begin{aligned}
&\frac{1}{nn_y}
\EE_{\cS}
\sum_{i=1}^n
\sum_{j=1}^{n_y}
\left(
\widehat T_\epsilon(S_{n_x}(u_i),\yb_{i,j})
-
G(u_i)(\yb_{i,j})
\right)
\xi_{i,j}
\\
&\qquad\le
2\sigma
\left(
\sqrt{
\EE_{\cS}
\|\widehat T_\epsilon-G\|_{\cS}^2
}
+
\theta
\right)
\sqrt{
\frac{
4\log\mathcal N
\left(
\theta,
\cH_\epsilon^{\rm clip},
\|\cdot\|_{\infty,\infty}
\right)
+6
}{
nn_y
}
}
+
\sigma\theta.
\end{aligned}
$$
Consequently,
$$
T_1
=
2\EE_{\cS}
\|\widehat T_\epsilon-G\|_{\cS}^2
\le
2\epsilon^2
+
8\sigma
\left(
\sqrt{
\EE_{\cS}
\|\widehat T_\epsilon-G\|_{\cS}^2
}
+
\theta
\right)
\sqrt{
\frac{
4\log\mathcal N
\left(
\theta,
\cH_\epsilon^{\rm clip},
\|\cdot\|_{\infty,\infty}
\right)
+6
}{
nn_y
}
}
+
4\sigma\theta.
$$
Denote
$$
\eta
:=
\sqrt{
\EE_{\cS}
\|\widehat T_\epsilon-G\|_{\cS}^2
},
$$
and set
$$
a
:=
\epsilon^2
+
2\sigma\theta
+
4\sigma\theta
\sqrt{
\frac{
4\log\mathcal N
\left(
\theta,\cH_\epsilon^{\rm clip},
\|\cdot\|_{\infty,\infty}
\right)+6
}{
nn_y
}},
$$
$$
b
:=
2\sigma
\sqrt{
\frac{
4\log\mathcal N
\left(
\theta,\cH_\epsilon^{\rm clip},
\|\cdot\|_{\infty,\infty}
\right)+6
}{
nn_y
}}.
$$
Then
$$
\eta^2\le a+2b\eta.
$$
Hence
$$
(\eta-b)^2\le a+b^2,
$$
which implies
$$
\eta^2\le2a+4b^2.
$$
Therefore,
$$
\begin{aligned}
T_1
&=
2\EE_{\cS}
\|\widehat T_\epsilon-G\|_{\cS}^2
\\
&\le
4\epsilon^2
+
8\sigma\theta
+
16\sigma\theta
\sqrt{
\frac{
4\log\mathcal N
\left(
\theta,\cH_\epsilon^{\rm clip},
\|\cdot\|_{\infty,\infty}
\right)+6
}{
nn_y
}}
+
32\sigma^2
\frac{
4\log\mathcal N
\left(
\theta,\cH_\epsilon^{\rm clip},
\|\cdot\|_{\infty,\infty}
\right)+6
}{
nn_y
}.
\end{aligned}
$$
\medskip
\noindent\textbf{Bounding $T_2$.}
The argument of Lemma~4 in \cite{liu2026scaling} applies to the clipped
Transformer class $\cH_\epsilon^{\rm clip}$ under our sampling and boundedness assumptions. Thus, for every $\theta>0$,
$$
T_2
\le
\frac{19B_{\cV}^2}{n}
\log\mathcal N
\left(
\frac{\theta}{4B_{\cV}},
\cH_\epsilon^{\rm clip},
\|\cdot\|_{\infty,\infty}
\right)
+
6\theta.
$$
\medskip
\noindent\textbf{Putting $T_1$ and $T_2$ together.}
Combining the bounds for $T_1$ and $T_2$, we obtain
$$
\begin{aligned}
\EE_{\cS}\cL(\widehat T_\epsilon)
&=
T_1+T_2
\\
&\le
4\epsilon^2
+
8\sigma\theta
\\
&\quad+
16\sigma\theta
\sqrt{
\frac{
4\log\mathcal N
\left(
\theta,
\cH_\epsilon^{\rm clip},
\|\cdot\|_{\infty,\infty}
\right)
+6
}{
nn_y
}
}
\\
&\quad+
32\sigma^2
\frac{
4\log\mathcal N
\left(
\theta,
\cH_\epsilon^{\rm clip},
\|\cdot\|_{\infty,\infty}
\right)
+6
}{
nn_y
}
\\
&\quad+
\frac{19B_{\cV}^2}{n}
\log\mathcal N
\left(
\frac{\theta}{4B_{\cV}},
\cH_\epsilon^{\rm clip},
\|\cdot\|_{\infty,\infty}
\right)
+
6\theta.
\end{aligned}
$$

By Lemma~\ref{lemma:transformer_covering},
$$
\log\mathcal N
\left(
\theta,
\cH_\epsilon^{\rm clip},
\|\cdot\|_{\infty,\infty}
\right)
\le
N_{\rm total}
\log\left(
\frac{
503552P^7D^{17}M_{\max}^{21}B_{\cU}^5
}{\theta}
\right),
$$
and
$$
\log\mathcal N
\left(
\frac{\theta}{4B_{\cV}},
\cH_\epsilon^{\rm clip},
\|\cdot\|_{\infty,\infty}
\right)
\le
N_{\rm total}
\log\left(
\frac{
2014208B_{\cV}P^7D^{17}M_{\max}^{21}B_{\cU}^5
}{\theta}
\right).
$$
Therefore,
$$
\begin{aligned}
\EE_{\cS}\cL(\widehat T_\epsilon)
&\le
4\epsilon^2
+
(8\sigma+6)\theta
\\
&\quad+
16\sigma\theta
\sqrt{
\frac{
4N_{\rm total}
\log\left(
\frac{
503552P^7D^{17}M_{\max}^{21}B_{\cU}^5
}{\theta}
\right)
+6
}{
nn_y
}
}
\\
&\quad+
32\sigma^2
\frac{
4N_{\rm total}
\log\left(
\frac{
503552P^7D^{17}M_{\max}^{21}B_{\cU}^5
}{\theta}
\right)
+6
}{
nn_y
}
\\
&\quad+
\frac{19B_{\cV}^2N_{\rm total}}{n}
\log\left(
\frac{
2014208B_{\cV}P^7D^{17}M_{\max}^{21}B_{\cU}^5
}{\theta}
\right).
\end{aligned}
$$

We now use the covering-number estimate from
Lemma~\ref{lemma:transformer_covering}. Choose
$$
\theta
:=
\frac{\epsilon^2}{1+\sigma}.
$$
By decreasing $\epsilon_0$ if necessary, we may assume that both
$\theta$ and $\theta/(4B_{\cV})$ belong to $(0,1]$.

By Theorem~\ref{thm:transformer_operator_approx},
$$
P
\le
C_P
\epsilon^{-\frac{d_{\cU}}{\gamma}-\frac{d_2}{\beta}},
$$
$$
D
\le
(d_1+d_2+10)C_x
\epsilon^{-\frac{d_1}{\alpha}(1+\frac{2}{\gamma})}
\left(\log\frac1\epsilon\right)^{d_1/\alpha},
$$
and
$$
M_{\max}
\le
C_{\rm mag}
\epsilon^{-q_M}
\left(\log\frac1\epsilon\right)^{1+d_1/\alpha}.
$$
Define
$$
C_{\rm cov}
:=
2014208(1+\sigma)(1+B_{\cV})B_{\cU}^5
C_P^7
\left((d_1+d_2+10)C_x\right)^{17}
C_{\rm mag}^{21},
$$
and
$$
C_{\log}
:=
1+
|\log C_{\rm cov}|
+
23
+
7\left(
\frac{d_{\cU}}{\gamma}
+
\frac{d_2}{\beta}
\right)
+
17\frac{d_1}{\alpha}
\left(
1+\frac{2}{\gamma}
\right)
+
21q_M
+
38\frac{d_1}{\alpha}.
$$
Since $\epsilon\le e^{-1}$, we have
$\log\log(1/\epsilon)\le\log(1/\epsilon)$. Hence
$$
\log\left(
\frac{
503552P^7D^{17}M_{\max}^{21}B_{\cU}^5
}{\theta}
\right)
\le
C_{\log}\log\frac1\epsilon,
$$
and
$$
\log\left(
\frac{
2014208B_{\cV}P^7D^{17}M_{\max}^{21}B_{\cU}^5
}{\theta}
\right)
\le
C_{\log}\log\frac1\epsilon.
$$
Therefore, by Lemma~\ref{lemma:transformer_covering} and the parameter-count
bound in Theorem~\ref{thm:transformer_operator_approx},
$$
\log\mathcal N
\left(
\theta,
\cH_\epsilon^{\rm clip},
\|\cdot\|_{\infty,\infty}
\right)
\le
C_NC_{\log}
\epsilon^{-s}
\left(\log\frac1\epsilon\right)^{\frac{2d_1}{\alpha}+1},
$$
and
$$
\log\mathcal N
\left(
\frac{\theta}{4B_{\cV}},
\cH_\epsilon^{\rm clip},
\|\cdot\|_{\infty,\infty}
\right)
\le
C_NC_{\log}
\epsilon^{-s}
\left(\log\frac1\epsilon\right)^{\frac{2d_1}{\alpha}+1}.
$$

Moreover,
$$
(8\sigma+6)\theta
=
\frac{8\sigma+6}{1+\sigma}\epsilon^2
\le
8\epsilon^2.
$$
Using $2ab\le a^2+b^2$,
$$
\begin{aligned}
&16\sigma\theta
\sqrt{
\frac{
4\log\mathcal N
\left(
\theta,
\cH_\epsilon^{\rm clip},
\|\cdot\|_{\infty,\infty}
\right)
+6
}{
nn_y
}
}
\\
&\qquad\le
8\theta
+
8\sigma^2\theta
\frac{
4\log\mathcal N
\left(
\theta,
\cH_\epsilon^{\rm clip},
\|\cdot\|_{\infty,\infty}
\right)
+6
}{
nn_y
}
\\
&\qquad\le
8\epsilon^2
+
8\sigma^2
\frac{
4C_NC_{\log}
\epsilon^{-s}
\left(\log\frac1\epsilon\right)^{\frac{2d_1}{\alpha}+1}
+6
}{
nn_y
}.
\end{aligned}
$$
Since $\epsilon\le e^{-1}$,
$$
\epsilon^{-s}
\left(\log\frac1\epsilon\right)^{\frac{2d_1}{\alpha}+1}
\ge1.
$$
Substituting the preceding estimates into the bound for
$T_1+T_2$ gives
$$
\EE_{\cS}\cL(\widehat T_\epsilon)
\le
C_{\rm gen}
\left[
\epsilon^2
+
\left(
1+\frac{\sigma^2}{n_y}
\right)
\frac{
\epsilon^{-s}
\left(\log\frac1\epsilon\right)^{\frac{2d_1}{\alpha}+1}
}{n}
\right],
$$
where one may take
\begin{equation}
\label{eq:Cgen_definition}
C_{\rm gen}
:=
20
+
40\left(4C_NC_{\log}+6\right)
+
19B_{\cV}^2C_NC_{\log}.
\end{equation}
This proves the generalization bound.

To obtain the stated statistical rate, choose
$$
\epsilon
:=
\left[
\frac{1+\sigma^2/n_y}{n}
(\log n)^{\frac{2d_1}{\alpha}+1}
\right]^{\frac{1}{2+s}}.
$$
For any $n\ge2$ such that this choice satisfies
$\epsilon\in(0,\epsilon_0)$ and the sequence-length compatibility condition $P\ge n_x+2$, since $s\ge 2d_1/\alpha$ and $1+\sigma^2/n_y\ge1$, we have $\epsilon\ge n^{-1}$. Thus, $\log\frac1{\epsilon}\le \log n$.

Moreover, by definition of $\epsilon$,
$$
\epsilon^{2+s}
=
\frac{1+\sigma^2/n_y}{n}
(\log n)^{\frac{2d_1}{\alpha}+1}.
$$
Therefore,
$$
\begin{aligned}
&
\left(
1+\frac{\sigma^2}{n_y}
\right)
\frac{
\epsilon^{-s}
\left(\log\frac1{\epsilon}\right)^{\frac{2d_1}{\alpha}+1}
}{n}
\le
\left(
1+\frac{\sigma^2}{n_y}
\right)
\frac{
\epsilon^{-s}
(\log n)^{\frac{2d_1}{\alpha}+1}
}{n}
=
\epsilon^2.
\end{aligned}
$$
Hence
$$
\EE_{\cS}\cL(\widehat T_{\epsilon})
\le
2C_{\rm gen}\epsilon^2
=
C_{\rm rate}
\left[
\frac{1+\sigma^2/n_y}{n}
(\log n)^{\frac{2d_1}{\alpha}+1}
\right]^{\frac{2}{2+s}},
$$
where
$$
C_{\rm rate}:=2C_{\rm gen}.
$$
This completes the proof.
\end{proof}

\section{Proofs of Supporting Lemmas}
\label{app:lemma_proofs}

\subsection{Proofs for the Two-Level Softmax POU Approximation}

\subsubsection{Proof of Lemma~\ref{lemma:input_space_pou}}
\label{app:proof:lemma:input_space_pou}
\begin{proof}[Proof of Lemma~\ref{lemma:input_space_pou}]
\phantomsection
\label{proof:input_space_pou}

Fix an arbitrary $u\in\cU$.
\ \\
\medskip
\noindent\textbf{Step 1: Covering and definition of the input-space Softmax POU.}
Since $\{u_k\}_{k=1}^{C_U}\subset\cU$ is an internal
$r_u$-cover of $\cU$ under the $L^2(\Omega_{\cU})$ metric, we can choose $k^\ast(u)\in \arg\min_{k\in[C_U]} \|u - u_k\|_{L^2(\Omega_{\cU})}$ such that $\|u- u_{k^\ast(u)}\|_{L^2(\Omega_{\cU})}\le r_u$.
The input-space weights are defined in \eqref{eq:beta_k} using affine logits. Equivalently, they can be written in the distance-score form
\begin{equation}
\label{eq:beta_k_distance_form}
\widetilde\beta_k(u)
=
\frac{
\exp\left(
M_u\left(
r_u^2-\|u-u_k\|_{L^2(\Omega_{\cU})}^2
\right)
\right)
}{
\sum_{k'=1}^{C_U}
\exp\left(
M_u\left(
r_u^2-\|u-u_{k'}\|_{L^2(\Omega_{\cU})}^2
\right)
\right)
},
\qquad k\in[C_U].
\end{equation}
Then $\widetilde\beta_k(u)\ge0$ and
$\sum_{k=1}^{C_U}\widetilde\beta_k(u)=1$.

To see the equivalence with \eqref{eq:beta_k}, notice that
$r_u^2-\|u\|_{L^2(\Omega_{\cU})}^2$ is common to all Softmax logits and
therefore cancels in the normalization. Indeed,
$$
\begin{aligned}
\widetilde\beta_k(u)
&=
\frac{
\exp\left(
M_u\left(
r_u^2-\|u\|_{L^2(\Omega_{\cU})}^2
\right)
\right)
\exp\left(
2M_u\langle u,u_k\rangle_{L^2(\Omega_{\cU})}
-
M_u\|u_k\|_{L^2(\Omega_{\cU})}^2
\right)
}{
\sum_{k'=1}^{C_U}
\exp\left(
M_u\left(
r_u^2-\|u\|_{L^2(\Omega_{\cU})}^2
\right)
\right)
\exp\left(
2M_u\langle u,u_{k'}\rangle_{L^2(\Omega_{\cU})}
-
M_u\|u_{k'}\|_{L^2(\Omega_{\cU})}^2
\right)
}
\\
&=
\frac{
\exp\left(
2M_u\langle u,u_k\rangle_{L^2(\Omega_{\cU})}
-
M_u\|u_k\|_{L^2(\Omega_{\cU})}^2
\right)
}{
\sum_{k'=1}^{C_U}
\exp\left(
2M_u\langle u,u_{k'}\rangle_{L^2(\Omega_{\cU})}
-
M_u\|u_{k'}\|_{L^2(\Omega_{\cU})}^2
\right)
}.
\end{aligned}
$$

\medskip
\noindent\textbf{Step 2: Near and far error estimates.}
Using the partition-of-unity property, for every $u\in\cU$,
$$
G(u)-G_{r_u}(u)
=
G(u)-\sum_{k=1}^{C_U}\widetilde\beta_k(u)G(u_k)
=
\sum_{k=1}^{C_U}
\widetilde\beta_k(u)\bigl(G(u)-G(u_k)\bigr).
$$
By the triangle inequality in $L^2(\rho_y)$,
$$
\begin{aligned}
\|G(u)-G_{r_u}(u)\|_{L^2(\rho_y)}
&\le
\sum_{k=1}^{C_U}
\widetilde\beta_k(u)
\|G(u)-G(u_k)\|_{L^2(\rho_y)}
\\
&=
\underbrace{
\sum_{\substack{k:\,
\|u-u_k\|_{L^2(\Omega_{\cU})}\le 2r_u}}
\widetilde\beta_k(u)
\|G(u)-G(u_k)\|_{L^2(\rho_y)}
}_{(I)}
\\
&\quad+
\underbrace{
\sum_{\substack{k:\,
\|u-u_k\|_{L^2(\Omega_{\cU})}>2r_u}}
\widetilde\beta_k(u)
\|G(u)-G(u_k)\|_{L^2(\rho_y)}
}_{(II)}.
\end{aligned}
$$

For the near-center term $(I)$, Assumption~\ref{assum:operator} yields
$$
\|G(u)-G(u_k)\|_{L^2(\rho_y)}
\le
L_G\|u-u_k\|_{L^2(\Omega_{\cU})}^{\gamma}
\le
L_G(2r_u)^\gamma
$$
for all $k$ satisfying
$\|u-u_k\|_{L^2(\Omega_{\cU})}\le2r_u$.
Therefore,
$$
(I)\le 2^\gamma L_Gr_u^\gamma.
$$

For the far-center term $(II)$, since $\rho_y$ is a probability distribution and Assumption~\ref{assum:output_space} gives $\|G(u)\|_{L^\infty(\Omega_{\cV})}\le B_{\cV}$ for every $u\in\cU$,
we have
$$
\begin{aligned}
\|G(u)-G(u_k)\|_{L^2(\rho_y)}
&\le
\|G(u)\|_{L^2(\rho_y)}
+
\|G(u_k)\|_{L^2(\rho_y)}
\\
&\le
\|G(u)\|_{L^\infty(\Omega_{\cV})}
+
\|G(u_k)\|_{L^\infty(\Omega_{\cV})}
\\
&\le
2B_{\cV}.
\end{aligned}
$$
It remains to bound the total Softmax mass assigned to far centers.
By the distance-score representation in \eqref{eq:beta_k_distance_form},
the denominator satisfies
$$
\sum_{k'=1}^{C_U}
\exp\left(
M_u\left(
r_u^2-\|u-u_{k'}\|_{L^2(\Omega_{\cU})}^2
\right)
\right)
\ge
\exp\left(
M_u\left(
r_u^2-\|u-u_{k^\ast(u)}\|_{L^2(\Omega_{\cU})}^2
\right)
\right)
\ge 1.
$$
If $\|u-u_k\|_{L^2(\Omega_{\cU})}>2r_u$, then
$$
r_u^2-\|u-u_k\|_{L^2(\Omega_{\cU})}^2
<
r_u^2-4r_u^2
=
-3r_u^2.
$$
So, $\widetilde\beta_k(u) \le \exp(-3M_ur_u^2)$. Then 
$$
\sum_{\substack{k:\,\|u-u_k\|_{L^2(\Omega_{\cU})}>2r_u}} \widetilde\beta_k(u) \le C_Ue^{-3M_ur_u^2}.
$$ 
Therefore,
$$
(II) \le 2B_{\cV}C_Ue^{-3M_ur_u^2}.
$$

Combining the estimates for $(I)$ and $(II)$, we obtain
$$
\|G(u)-G_{r_u}(u)\|_{L^2(\rho_y)}
\le
2^\gamma L_Gr_u^\gamma
+
2B_{\cV}C_Ue^{-3M_ur_u^2}.
$$

\medskip
\noindent\textbf{Step 3: Choice of $M_u$.}
By the definition of $M_u$ in \eqref{eq:Mu_definition}, $M_u = \frac{1}{3r_u^2} \log\left(\frac{2B_{\cV}C_U}{L_Gr_u^\gamma}\right)$. Therefore,
$$
2B_{\cV}C_Ue^{-3M_ur_u^2} = L_Gr_u^\gamma.
$$
Since $u\in\cU$ is arbitrary, taking the supremum over $u$ gives
$$
\sup_{u\in\cU}
\|G(u)-G_{r_u}(u)\|_{L^2(\rho_y)}
\le
(2^\gamma+1)L_Gr_u^\gamma.
$$
This proves \eqref{eq:input_space_pou_error}.
\end{proof}
\subsubsection{Proof of Lemma~\ref{lemma:output_domain_pou}}
\label{app:proof:lemma:output_domain_pou}
\begin{proof}[Proof of Lemma~\ref{lemma:output_domain_pou}]
\phantomsection
\label{proof:output_domain_pou}

Fix an arbitrary $k\in[C_U]$ and $\yb\in\Omega_{\cV}$.

\medskip
\noindent\textbf{Step 1: Covering and definition of the output-domain Softmax POU.}
Since $\{\zb_l\}_{l=1}^{C_V}$ is an $r_y$-cover of $\Omega_{\cV}$ under the Euclidean metric, we can choose $l^\ast(\yb)\in\arg\min_{l\in[C_V]}\|\yb-\zb_l\|_2$ such that
$\|\yb-\zb_{l^\ast(\yb)}\|_2\le r_y$. The output-domain weights are defined in \eqref{eq:eta_l} using affine logits. Equivalently, they can be written in the distance-score form
\begin{equation}
\label{eq:eta_l_distance_form}
\widetilde\eta_l(\yb)
=
\frac{
\exp\left(M_y\left(r_y^2-\|\yb-\zb_l\|_2^2\right)\right)
}{
\sum_{l'=1}^{C_V}
\exp\left(M_y\left(r_y^2-\|\yb-\zb_{l'}\|_2^2\right)\right)
},
\qquad l\in[C_V].
\end{equation}
Then $\widetilde\eta_l(\yb)\ge0$ and
$\sum_{l=1}^{C_V}\widetilde\eta_l(\yb)=1$.
To see the equivalence with \eqref{eq:eta_l}, notice that the common term
$r_y^2-\|\yb\|_2^2$ exactly cancels in the Softmax normalization. Indeed,
$$
\begin{aligned}
\widetilde\eta_l(\yb)
&=
\frac{
\exp\left(M_y\left(r_y^2-\|\yb\|_2^2\right)\right)
\exp\left(2M_y\yb^\top\zb_l-M_y\|\zb_l\|_2^2\right)
}{
\sum_{l'=1}^{C_V}
\exp\left(M_y\left(r_y^2-\|\yb\|_2^2\right)\right)
\exp\left(2M_y\yb^\top\zb_{l'}-M_y\|\zb_{l'}\|_2^2\right)
}
\\
&=
\frac{
\exp\left(2M_y\yb^\top\zb_l-M_y\|\zb_l\|_2^2\right)
}{
\sum_{l'=1}^{C_V}
\exp\left(2M_y\yb^\top\zb_{l'}-M_y\|\zb_{l'}\|_2^2\right)
}.
\end{aligned}
$$

\medskip
\noindent\textbf{Step 2: Near and far error estimates.}
Using the partition-of-unity property, for every $k\in[C_U]$ and
$\yb\in\Omega_{\cV}$,
$$
G(u_k)(\yb)-\widetilde v_k(\yb)
=
G(u_k)(\yb)
-
\sum_{l=1}^{C_V}G(u_k)(\zb_l)\widetilde\eta_l(\yb)
=
\sum_{l=1}^{C_V}
\widetilde\eta_l(\yb)
\bigl(G(u_k)(\yb)-G(u_k)(\zb_l)\bigr).
$$
Taking the absolute value and applying the triangle inequality, we obtain
$$
\begin{aligned}
|G(u_k)(\yb)-\widetilde v_k(\yb)|
&\le
\sum_{l=1}^{C_V}
\widetilde\eta_l(\yb)
|G(u_k)(\yb)-G(u_k)(\zb_l)|
\\
&=
\underbrace{
\sum_{\substack{l:\,
\|\yb-\zb_l\|_2\le2r_y}}
\widetilde\eta_l(\yb)
|G(u_k)(\yb)-G(u_k)(\zb_l)|
}_{(I)}
\\
&\quad+
\underbrace{
\sum_{\substack{l:\,
\|\yb-\zb_l\|_2>2r_y}}
\widetilde\eta_l(\yb)
|G(u_k)(\yb)-G(u_k)(\zb_l)|
}_{(II)}.
\end{aligned}
$$

For the near-center term $(I)$, Assumption~\ref{assum:output_space} yields
$$
|G(u_k)(\yb)-G(u_k)(\zb_l)|
\le
B_{\cV}\|\yb-\zb_l\|_2^\beta
\le
B_{\cV}(2r_y)^\beta
$$
for all $l$ satisfying $\|\yb-\zb_l\|_2\le2r_y$.
Therefore,
$$
(I)\le2^\beta B_{\cV}r_y^\beta.
$$

For the far-center term $(II)$, Assumption~\ref{assum:output_space} gives
$$
|G(u_k)(\yb)-G(u_k)(\zb_l)|
\le
|G(u_k)(\yb)|+|G(u_k)(\zb_l)|
\le
2B_{\cV}.
$$
It remains to bound the total Softmax mass assigned to far centers.
By the distance-score representation in \eqref{eq:eta_l_distance_form}, the denominator satisfies
$$
\sum_{l'=1}^{C_V}
\exp\left(M_y\left(r_y^2-\|\yb-\zb_{l'}\|_2^2\right)\right)
\ge
\exp\left(M_y\left(r_y^2-\|\yb-\zb_{l^\ast(\yb)}\|_2^2\right)\right)
\ge1.
$$
If $\|\yb-\zb_l\|_2>2r_y$, then
$$
r_y^2-\|\yb-\zb_l\|_2^2
<
r_y^2-4r_y^2
=
-3r_y^2.
$$
So, $\widetilde\eta_l(\yb)\le e^{-3M_yr_y^2}$. Then,
$$
\sum_{\substack{l:\,
\|\yb-\zb_l\|_2>2r_y}}
\widetilde\eta_l(\yb)
\le
C_Ve^{-3M_yr_y^2},
$$
Therefore,
$$
(II)\le2B_{\cV}C_Ve^{-3M_yr_y^2}.
$$

Combining the estimates for $(I)$ and $(II)$, we obtain
$$
|G(u_k)(\yb)-\widetilde v_k(\yb)|
\le
2^\beta B_{\cV}r_y^\beta
+
2B_{\cV}C_Ve^{-3M_yr_y^2}.
$$

\medskip
\noindent\textbf{Step 3: Choice of $M_y$.}
By the definition of $M_y$ in \eqref{eq:My_definition}, $M_y = \frac{1}{3r_y^2} \log\left(\frac{2C_V}{r_y^\beta}\right)$. Therefore,
$$
2B_{\cV}C_Ve^{-3M_yr_y^2}
=
B_{\cV}r_y^\beta.
$$
Since $k\in[C_U]$ and $\yb\in\Omega_{\cV}$ are arbitrary, taking the supremum over $k$ and $\yb$ gives
$$
\sup_{k\in[C_U]}
\|G(u_k)-\widetilde v_k\|_{L^\infty(\Omega_{\cV})}
\le
(2^\beta+1)B_{\cV}r_y^\beta.
$$
This proves \eqref{eq:output_domain_pou_error}.
\end{proof}

\subsubsection{Proof of Lemma~\ref{lemma:quadrature}}
\label{app:proof:lemma:quadrature}
\begin{proof}[Proof of Lemma~\ref{lemma:quadrature}]
\phantomsection
\label{proof:quadrature}
With $\alpha=m_\alpha+\nu$, $s_{\rm quad}$, and $q$ chosen as in Lemma~\ref{lemma:quadrature}, we have $2s_{\rm quad}-1\ge m_\alpha$. Partitioning $\Omega_{\cU}=[0,1]^{d_1}$ into $q^{d_1}$ congruent subcubes and applying the same quadrature rule on each subcube gives a multidimensional compound rule \cite[Section~5.8]{davis1984methods}.

On each subcube, we use the tensor-product $s_{\rm quad}$-point Gauss--Legendre rule. The one-dimensional $s_{\rm quad}$-point Gauss--Legendre rule has positive weights and is exact for polynomials of degree at most $2s_{\rm quad}-1$ \cite[Section~2.7]{davis1984methods}. By the Cartesian product construction \cite[Section~5.6]{davis1984methods}, the resulting tensor-product rule has $s_{\rm quad}^{d_1}$ nodes on each subcube and is exact for every polynomial whose degree in each coordinate is at most $2s_{\rm quad}-1$.

For each cube $K$, let $p_K$ be the Taylor polynomial of $h$ of degree $m_\alpha$ about the center of $K$. The H\"older Taylor remainder gives
$$
\|h-p_K\|_{L^\infty(K)}
\le
C_T q^{-\alpha}
\|h\|_{C^\alpha(\Omega_{\cU})},
$$
where $C_T=C_T(d_1,\alpha)>0$. Since $2s_{\rm quad}-1\ge m_\alpha$, the quadrature rule is exact on $p_K$. Moreover, its weights are positive and sum to $|K|$. Hence
$$
\begin{aligned}
\left|
\int_K h(\xb)\,d\xb-Q_K(h)
\right|
&\le
\left|
\int_K(h-p_K)(\xb)\,d\xb
\right|
+
\left|Q_K(h-p_K)\right|
\\
&\le
2|K|
\|h-p_K\|_{L^\infty(K)}
\\
&\le
2C_T|K|q^{-\alpha}
\|h\|_{C^\alpha(\Omega_{\cU})}.
\end{aligned}
$$
Summing over all $q^{d_1}$ cubes and using
$\sum_K|K|=1$ gives
$$
\left|\int_{\Omega_{\cU}}h(\xb)\,d\xb-Q_q(h)\right|
\le
2C_Tq^{-\alpha}\|h\|_{C^\alpha(\Omega_{\cU})}.
$$

The composite tensor-product rule uses $N_q=(s_{\rm quad}q)^{d_1}$ quadrature points. By the choice of $q$ in Lemma~\ref{lemma:quadrature}, $N_q\le n_x$ and $q\ge \frac{n_x^{1/d_1}}{2s_{\rm quad}}$. If $N_q<n_x$, append distinct grid points with zero weights. Thus,
$$
\left|\int_{\Omega_{\cU}}h(\xb)\,d\xb - \sum_{j=1}^{n_x}w_jh(\xb_j) \right|
\le
2C_T(2s_{\rm quad})^\alpha \|h\|_{C^\alpha(\Omega_{\cU})} n_x^{-\alpha/d_1}.
$$
Therefore, the desired estimate holds with $C_{\rm quad,0}:=2C_T(2s_{\rm quad})^\alpha$.

The constructed weights are nonnegative. Since the quadrature rule is exact for the constant polynomial $1$,
$$
\sum_{j=1}^{n_x}w_j = \int_{\Omega_{\cU}}1\,d\xb = 1,
$$
which proves \eqref{eq:quadrature_weight_stability}.

Since the cover is internal, $u_k\in\cU$. By
Assumption~\ref{assum:input_space} and the standard product estimate in H\"older spaces,
$$
\|uu_k\|_{C^\alpha(\Omega_{\cU})}
\le
C_{\alpha,d_1}
\|u\|_{C^\alpha(\Omega_{\cU})}
\|u_k\|_{C^\alpha(\Omega_{\cU})}
\le
C_{\alpha,d_1}B_{\cU}^2.
$$
Applying the preceding quadrature estimate to $h=uu_k$ gives
$$
\left|
\langle u,u_k\rangle_{L^2(\Omega_{\cU})}
-
\langle u,u_k\rangle_{n_x}
\right|
\le
C_{\rm quad,0}
C_{\alpha,d_1}B_{\cU}^2
n_x^{-\alpha/d_1}.
$$
Define
\begin{equation}
\label{eq:Cquad_definition}
C_{\rm quad}
:=
C_{\rm quad,0}C_{\alpha,d_1}B_{\cU}^2.
\end{equation}
Taking the supremum over $u\in\cU$ and the maximum over
$k\in[C_U]$ then proves \eqref{eq:quadrature_error}.
\end{proof}

\subsubsection{Proof of Lemma~\ref{lemma:total_approximation_error}}
\label{app:proof:lemma:total_approximation_error}

\begin{proof}[Proof of Lemma~\ref{lemma:total_approximation_error}]
\phantomsection
\label{proof:total_approximation_error}
Fix arbitrary $u\in\cU$.

\medskip
\noindent\textbf{Step 1: Error decomposition.}
By the definitions of $G_{r_u}$ in \eqref{eq:input_space_pou_operator}, $\widetilde G_{r_u,r_y}$ in \eqref{eq:continuous_two_level_pou}, and $\widetilde G_{r_u,r_y,n_x}$ in \eqref{eq:discrete_two_level_pou}, the triangle inequality in $L^2(\rho_y)$ gives
$$
\begin{aligned}
\left\|G(u)-\widetilde G_{r_u,r_y,n_x}(u)\right\|_{L^2(\rho_y)}
&\le
\left\|G(u)-G_{r_u}(u)\right\|_{L^2(\rho_y)}
\\
&\quad+
\left\|G_{r_u}(u)-\widetilde G_{r_u,r_y}(u)\right\|_{L^2(\rho_y)}
\\
&\quad+
\left\|\widetilde G_{r_u,r_y}(u) - \widetilde G_{r_u,r_y,n_x}(u) \right\|_{L^2(\rho_y)}.
\end{aligned}
$$

\medskip
\noindent\textbf{Step 2: Bounds for the three terms.}
By Lemma~\ref{lemma:input_space_pou},
$$
\|G(u)-G_{r_u}(u)\|_{L^2(\rho_y)}
\le
(2^\gamma+1)L_Gr_u^\gamma.
$$

For the second term, using
\eqref{eq:input_space_pou_operator},
\eqref{eq:output_domain_pou_approximation}, and
\eqref{eq:continuous_two_level_pou},
$$
G_{r_u}(u)-\widetilde G_{r_u,r_y}(u) = \sum_{k=1}^{C_U} \widetilde\beta_k(u) \bigl(G(u_k)-\widetilde v_k\bigr).
$$
Therefore, by the triangle inequality in $L^2(\rho_y)$,
$$
\left\|G_{r_u}(u)-\widetilde G_{r_u,r_y}(u)\right\|_{L^2(\rho_y)}
\le
\sum_{k=1}^{C_U}\widetilde\beta_k(u)\left\|G(u_k)-\widetilde v_k \right\|_{L^2(\rho_y)}
\le
\sum_{k=1}^{C_U} \widetilde\beta_k(u) \left\|G(u_k)-\widetilde v_k \right\|_{L^\infty(\Omega_{\cV})}.
$$
By Lemma~\ref{lemma:output_domain_pou} and $\sum_{k=1}^{C_U}\widetilde\beta_k(u)=1$,
$$
\left\|G_{r_u}(u)-\widetilde G_{r_u,r_y}(u)\right\|_{L^2(\rho_y)}
\le
(2^\beta+1)B_{\cV}r_y^\beta.
$$
It remains to bound the third term. Define the continuous and discrete
logits by
$$
\ell_k(u) := 2M_u\langle u,u_k\rangle_{L^2(\Omega_{\cU})} - M_u\|u_k\|_{L^2(\Omega_{\cU})}^2,
\qquad
\ell_{k,n_x}(u) := 2M_u\langle u,u_k\rangle_{n_x} - M_u\|u_k\|_{L^2(\Omega_{\cU})}^2.
$$
By Lemma~\ref{lemma:quadrature},
$$
\left|\langle u,u_k\rangle_{L^2(\Omega_{\cU})} - \langle u,u_k\rangle_{n_x}\right|
\le
C_{\rm quad}n_x^{-\alpha/d_1}
$$
uniformly over $u\in\cU$ and $k\in[C_U]$.  Therefore, for every $k\in[C_U]$,
$$
|\ell_k(u)-\ell_{k,n_x}(u)|
=
2M_u\left|\langle u,u_k\rangle_{L^2(\Omega_{\cU})} - \langle u,u_k\rangle_{n_x}\right|
\le
2M_uC_{\rm quad}n_x^{-\alpha/d_1}.
$$
Define $\mathbf a(u) := (\ell_1(u),\ldots,\ell_{C_U}(u))$ and $\mathbf b(u) := (\ell_{1,n_x}(u),\ldots,\ell_{C_U,n_x}(u))$. Then
$$
\|\mathbf a(u)-\mathbf b(u)\|_\infty \le 2M_uC_{\rm quad}n_x^{-\alpha/d_1}.
$$
By Lemma~\ref{lemma:softmax_lipschitz},
$$
\sum_{k=1}^{C_U}
|\widetilde\beta_k(u)-\widetilde\beta_{k,n_x}(u)|
=
\|\operatorname{softmax}(\mathbf a(u)) - \operatorname{softmax}(\mathbf b(u))\|_1 
\le
2\|\mathbf a(u)-\mathbf b(u)\|_\infty \le
4M_uC_{\rm quad}n_x^{-\alpha/d_1}.
$$
Finally, using \eqref{eq:continuous_two_level_pou} and
\eqref{eq:discrete_two_level_pou},
$$
|\widetilde G_{r_u,r_y}(u)(\yb) -\widetilde G_{r_u,r_y,n_x}(u)(\yb)| 
=
\left|\sum_{k=1}^{C_U}\bigl(\widetilde\beta_k(u)-\widetilde\beta_{k,n_x}(u)\bigr)\widetilde v_k(\yb)\right|
\le
\sum_{k=1}^{C_U}|\widetilde\beta_k(u)-\widetilde\beta_{k,n_x}(u)| \,|\widetilde v_k(\yb)|.
$$
Since $\widetilde v_k(\yb)$ is a convex combination of
$\{G(u_k)(\zb_l)\}_{l=1}^{C_V}$, Assumption~\ref{assum:output_space} gives
$|\widetilde v_k(\yb)|\le B_{\cV}$. Therefore,
$$
|\widetilde G_{r_u,r_y}(u)(\yb) -\widetilde G_{r_u,r_y,n_x}(u)(\yb)|
\le
4B_{\cV}M_uC_{\rm quad}n_x^{-\alpha/d_1}.
$$
Since $\rho_y$ is a probability distribution, this uniform bound implies
$$
\left\|\widetilde G_{r_u,r_y}(u) -\widetilde G_{r_u,r_y,n_x}(u)\right\|_{L^2(\rho_y)}
\le
4B_{\cV}M_uC_{\rm quad}n_x^{-\alpha/d_1}.
$$
Combining the three bounds and taking the supremum over $u\in\cU$
proves \eqref{eq:total_approximation_error}.
\end{proof}

\subsection{Proofs for the Transformer Realization}

\subsubsection{Proof of Lemma~\ref{lemma:pair_token_preprocess}}
\label{app:proof:lemma:pair_token_preprocess}

\begin{proof}[Proof of Lemma~\ref{lemma:pair_token_preprocess}]
\phantomsection
\label{proof:pair_token_preprocess}

Let
$$
\Xb(u,\yb):=
\begin{bmatrix}
\xb_1 & \cdots & \xb_{n_x} & \mathbf 0_{d_1} & \mathbf 0_{d_1\times(P-n_x-1)}
\\
\mathbf 0_{d_2} & \cdots & \mathbf 0_{d_2} & \yb & \mathbf 0_{d_2\times(P-n_x-1)}
\\
u(\xb_1) & \cdots & u(\xb_{n_x}) & 0 & \mathbf 0_{1\times(P-n_x-1)}
\end{bmatrix}
\in\RR^{(d_1+d_2+1)\times P}.
$$
Thus, the first $n_x$ columns contain the discretized input information, the $(n_x+1)$-st column contains the query point $\yb$, and all remaining input coordinates are zero.

Choose the shared token-wise affine embedding
$$
\Wb_E=
\begin{bmatrix}
I_{d_1+d_2+1}
\\
0_{(D-d_1-d_2-1)\times(d_1+d_2+1)}
\end{bmatrix}
\in\RR^{D\times(d_1+d_2+1)},
\qquad
\bb_E=\mathbf 0_D.
$$
Then
$$
\Sb := \Wb_E\Xb(u,\yb)+\bb_E\mathbf 1_P^\top\in\RR^{D\times P}
$$
places the discretized input and query information in the first
$d_1+d_2+1$ coordinates and sets all remaining coordinates to zero.

Define $\mathbf s_\theta := \bigl(\sin(\theta_1),\ldots,\sin(\theta_P)\bigr)$ and $\mathbf c_\theta := \bigl(\cos(\theta_1),\ldots,\cos(\theta_P)\bigr)$. Let the fixed structural-positional encoding matrix be
$$
\mathbf P_{\rm pos} :=
\begin{bmatrix}
0_{(d_1+d_2+1)\times P}
\\
\mathbf e_P^\top
\\
0_{(n_x+3)\times P}
\\
\mathbf 1_P^\top
\\
\mathbf 1_P^\top-\mathbf e_P^\top
\\
\mathbf s_\theta
\\
\mathbf c_\theta
\end{bmatrix}
\in\RR^{D\times P},
$$
where $\mathbf e_P\in\RR^P$ is the $P$-th standard basis vector.
Here, $\mathbf e_P^\top=(0,\ldots,0,1)$ is the null-column indicator, while $\mathbf 1_P^\top-\mathbf e_P^\top=(1,\ldots,1,0)$ is the active-column indicator.

Finally, define the preprocessing operator by
$$
\Zb_0 := \Sb+ \Pb_{\rm pos} = \Wb_E \Xb(u,\yb) + \bb_E\mathbf 1_P^\top + \Pb_{\rm pos}.
$$
This gives exactly the representation in \eqref{eq:pair_token_preprocess_z0}.

\end{proof}

\subsubsection{Proof of Lemma~\ref{lemma:pair_token_affine_mha}}
\label{app:proof:lemma:pair_token_affine_mha}

\begin{proof}[Proof of Lemma~\ref{lemma:pair_token_affine_mha}]
\phantomsection
\label{proof:pair_token_affine_mha}

Set $M_s:=\frac{M}{c_P}$, $M_\theta:=\frac{4M}{c_P}$.
Let
$$
\chi_t:=1_{\{t<P\}},
\qquad
\nu_t:=1_{\{t=P\}},
\qquad
t\in[P].
$$
Recall that row $D-3$ of $\Zb_0$ is the all-one row, row $D-2$
equals $(1,\ldots,1,0)$ and hence stores $\chi_t$, while row
$d_1+d_2+2$ is the null-column indicator and hence stores $\nu_t$.

\medskip
\noindent\textbf{Step 1: Feature extraction heads.}
In this step, each head fixes one output feature row $r$ and one output column $j$. The indices $t$ and $j'$ are used for the  source and target columns of the attention matrix. For
$$
h=(r-1)P+j, \qquad r\in[n_x+3], \qquad j\in[P],
$$
define the source column
$$
i_r:=
\begin{cases}
r, & r\in[n_x],\\
n_x+1, & r=n_x+1,\\
1, & r\in\{n_x+2,n_x+3\}.
\end{cases}
$$
Set $M_0:=M_\theta\left(1-\frac{c_P}{2}\right)+M_s$. This step uses exactly $(n_x+3)P$ heads.

For $j<P$, write
$$
j=j(k_j,l_j)=(k_j-1)C_V+l_j.
$$
Choose $\mathbf w_{r,j}^\top\in\RR^{1\times D}$ by
$$
\mathbf w_{r,j}^\top
:=
\begin{cases}
2M_u w_r u_{k_j}(\xb_r)\mathbf e_{d_1+d_2+1}^\top,
& r\in[n_x],\ j<P,
\\[1ex]
2M_y\displaystyle\sum_{a=1}^{d_2}(\zb_{l_j})_a\mathbf e_{d_1+a}^\top - M_y\|\zb_{l_j}\|_2^2\mathbf e_{D-2}^\top,
& r=n_x+1,\ j<P,
\\[2ex]
N_{k_j}\mathbf e_{D-2}^\top,
& r=n_x+2,\ j<P,
\\[1ex]
A_{k_j,l_j}\mathbf e_{D-2}^\top,
& r=n_x+3,\ j<P,
\\[1ex]
\mathbf 0_{1\times D},
& j=P.
\end{cases}
$$

Denote the ideal target features by
$$
F_{r,j} := \mathbf w_{r,j}^\top(\Zb_0)_{:,i_r}.
$$
Then
$$
F_{r,j}
=
\begin{cases}
B_{r,k_j}(u),
& r\in[n_x],\ j<P,
\\
Q_{l_j}(\yb),
& r=n_x+1,\ j<P,
\\
N_{k_j},
& r=n_x+2,\ j<P,
\\
A_{k_j,l_j},
& r=n_x+3,\ j<P,
\\
0,
& j=P.
\end{cases}
$$
Thus $F_{r,j}$ is the ideal target feature that the first MHA layer
targets at row $r$ and column $j$.

For $j<P$,
$$
\|\mathbf w_{r,j}\|_\infty
\le
\begin{cases}
2M_uB_{\cU},
& r\in[n_x],
\\
(2+d_2)M_y,
& r=n_x+1,
\\
M_uB_{\cU}^2,
& r=n_x+2,
\\
B_{\cV},
& r=n_x+3.
\end{cases}
$$
Moreover, for every $t\in[P]$,
$$
\left|
\mathbf w_{r,j}^\top(\Zb_0)_{:,t}
\right|
\le
\begin{cases}
2M_u B_{\cU}^2,
& r\in[n_x],
\\
3d_2M_y,
& r=n_x+1,
\\
M_u B_{\cU}^2,
& r=n_x+2,
\\
B_{\cV},
& r=n_x+3.
\end{cases}
$$
Hence, by \eqref{eq:B_aff_definition},
$$
\|\mathbf w_{r,j}\|_\infty\le B_{\rm aff},
\qquad
\left|
\mathbf w_{r,j}^\top(\Zb_0)_{:,t}
\right|
\le B_{\rm aff},
$$
for every $r\in[n_x+3]$ and $j,t\in[P]$.

We construct the query and key matrices as
$$
\Qb_1^h =
\begin{bmatrix}
\mathbf e_{D-3}^\top
\\
\mathbf e_{D-1}^\top
\\
\mathbf e_D^\top
\\
M_s\sin(\theta_{i_r})\mathbf e_{D-3}^\top
\\
M_s\cos(\theta_{i_r})\mathbf e_{D-3}^\top
\end{bmatrix}
\in\RR^{5\times D},
\qquad
\Kb_1^h =
\begin{bmatrix}
M_0\mathbf e_{d_1+d_2+2}^\top
\\
M_\theta\sin(\theta_j)\mathbf e_{D-2}^\top
\\
M_\theta\cos(\theta_j)\mathbf e_{D-2}^\top
\\
\mathbf e_{D-1}^\top -\sin(\theta_P)\mathbf e_{d_1+d_2+2}^\top
\\
\mathbf e_D^\top -\cos(\theta_P)\mathbf e_{d_1+d_2+2}^\top
\end{bmatrix}
\in\RR^{5\times D}.
$$
It follows that
$$
\mathbf q_{j'}^h := (\Qb_1^h\Zb_0)_{:,j'} =
\begin{bmatrix}
1
\\
\sin(\theta_{j'})
\\
\cos(\theta_{j'})
\\
M_s\sin(\theta_{i_r})
\\
M_s\cos(\theta_{i_r})
\end{bmatrix}
\in\RR^5.
$$
And
$$
\mathbf k_t^h := (\Kb_1^h\Zb_0)_{:,t} =
\begin{bmatrix}
0
\\
M_\theta\sin(\theta_j)
\\
M_\theta\cos(\theta_j)
\\
\sin(\theta_t)
\\
\cos(\theta_t)
\end{bmatrix}
\in\RR^5,
\quad
t<P,
\qquad
\mathbf k_P^h := (\Kb_1^h\Zb_0)_{:,P} =
\begin{bmatrix}
M_0
\\
0
\\
0
\\
0
\\
0
\end{bmatrix}.
$$
Thus
$$
s_{t,j'}^h =
\begin{cases}
M_\theta\cos(\theta_{j'}-\theta_j) + M_s\cos(\theta_t-\theta_{i_r}),
& t<P,
\\
M_0,
& t=P.
\end{cases}
$$
The source-column locating term satisfies
$$
M_s\cos(\theta_t-\theta_{i_r}) = M_s,
\qquad
t=i_r,
$$
and
$$
M_s\cos(\theta_t-\theta_{i_r})
\le
M_s(1-c_P),
\qquad
t<P,\quad t\ne i_r.
$$
The output-column locating term satisfies
$$
M_\theta\cos(\theta_{j'}-\theta_j) = M_\theta,
\qquad
j'=j,
$$
and
$$
M_\theta\cos(\theta_{j'}-\theta_j)
\le
M_\theta(1-c_P),
\qquad
j'\ne j.
$$
Choose the value matrix as
$$
\Vb_1^h =
\begin{bmatrix}
\mathbf w_{r,j}^\top
\\
\mathbf 0_{1\times D}
\end{bmatrix}
\in\RR^{2\times D}.
$$
Therefore,
$$
(\Vb_1^h\Zb_0)_{:,t} =
\begin{bmatrix}
\mathbf w_{r,j}^\top(\Zb_0)_{:,t}
\\
0
\end{bmatrix},
$$
and
$$
\left|(\Vb_1^h\Zb_0)_{1,t}\right|
\le
B_{\rm aff}.
$$
We separate the target and non-target attention columns. For the target column $j'=j$, for every $t<P$ with $t\ne i_r$,
$$
s_{i_r,j}^h-s_{t,j}^h = M_s\left(1-\cos(\theta_t-\theta_{i_r})
\right) \ge c_PM_s = M,
$$
and
$$
s_{i_r,j}^h-s_{P,j}^h = M_\theta+M_s-M_0 = \frac{c_P}{2}M_\theta =
2M.
$$
Therefore,
$$
\sum_{t\ne i_r}(\Ab_1^h)_{t,j} \le \sum_{t\ne i_r} e^{s_{t,j}^h-s_{i_r,j}^h} \le Pe^{-M}.
$$
Hence,
$$
\left|\bigl(\Vb_1^h\Zb_0\Ab_1^h\bigr)_{1,j} - F_{r,j}\right| 
=
\left|\sum_{t\ne i_r}(\Ab_1^h)_{t,j}\left(\mathbf w_{r,j}^\top(\Zb_0)_{:,t} - \mathbf w_{r,j}^\top(\Zb_0)_{:,i_r}\right)\right|
\le
2PB_{\rm aff}e^{-M}.
$$
For a non-target column $j'\ne j$, for every $t<P$,
$$
s_{P,j'}^h-s_{t,j'}^h \ge M_0-M_\theta(1-c_P)-M_s = \frac{c_P}{2}M_\theta = 2M.
$$
Therefore,
$$
\sum_{t<P}(\Ab_1^h)_{t,j'}
\le
\sum_{t<P}e^{s_{t,j'}^h-s_{P,j'}^h}
\le
Pe^{-2M}
\le
Pe^{-M}.
$$
Since
$$
\mathbf w_{r,j}^\top(\Zb_0)_{:,P}=0,
$$
we have
$$
\left|\bigl(\Vb_1^h\Zb_0\Ab_1^h\bigr)_{1,j'}\right|
\le
PB_{\rm aff}e^{-M}.
$$
\medskip
\noindent\textbf{Step 2: Indicator and positional encoding heads.}
Let $h_0:=(n_x+3)P$. The head $h_0+1$ preserves the all-one row and the active-indicator row. We construct the query and key matrices as
$$
\Qb_1^{h_0+1} =
\begin{bmatrix}
\mathbf e_{D-2}^\top
\\
\mathbf e_{d_1+d_2+2}^\top
\\
\mathbf 0_{1\times D}
\\
\mathbf 0_{1\times D}
\\
\mathbf 0_{1\times D}
\end{bmatrix}
\in\RR^{5\times D},
\qquad
\Kb_1^{h_0+1} =
\begin{bmatrix}
M\mathbf e_{D-2}^\top
\\
M\mathbf e_{d_1+d_2+2}^\top
\\
\mathbf 0_{1\times D}
\\
\mathbf 0_{1\times D}
\\
\mathbf 0_{1\times D}
\end{bmatrix}
\in\RR^{5\times D}.
$$
Then
$$
s_{t,j'}^{h_0+1} = M\chi_t\chi_{j'} + M\nu_t\nu_{j'}.
$$
Define
$$
v_t^{\rm ind} := \chi_t-(P-1)e^{-M}\nu_t.
$$
Choose the value matrix
$$
\Vb_1^{h_0+1} =
\begin{bmatrix}
\mathbf e_{D-3}^\top
\\
\mathbf e_{D-2}^\top - (P-1)e^{-M}\mathbf e_{d_1+d_2+2}^\top
\end{bmatrix}
\in\RR^{2\times D}.
$$
Therefore,
$$
(\Vb_1^{h_0+1}\Zb_0)_{:,t} =
\begin{bmatrix}
1
\\
v_t^{\rm ind}
\end{bmatrix},
\qquad
t\in[P].
$$
For $j'<P$,
$$
\sum_{t=1}^P(\Ab_1^{h_0+1})_{t,j'}v_t^{\rm ind}
=
\frac{(P-1)e^M-(P-1)e^{-M}}{(P-1)e^M+1}
=
\frac{(P-1)(e^M-e^{-M})}{(P-1)e^M+1}
=:
a_M,
$$
whereas for $j'=P$,
$$
\sum_{t=1}^P(\Ab_1^{h_0+1})_{t,P}v_t^{\rm ind}
=
\frac{(P-1)-e^M(P-1)e^{-M}}{(P-1)+e^M}
=
0.
$$
Hence,
$$
\bigl(\Vb_1^{h_0+1}\Zb_0\Ab_1^{h_0+1}\bigr)_{:,j'}
=
\begin{cases}
\begin{bmatrix}
1\\
a_M
\end{bmatrix},
& j'<P,
\\[3ex]
\begin{bmatrix}
1\\
0
\end{bmatrix},
& j'=P.
\end{cases}
$$
Moreover,
$$
1-a_M =
\frac{1+(P-1)e^{-M}}{(P-1)e^M+1}
\le
2e^{-M}.
$$
Since $P\ge2$ and $Pe^{-M}\le1/4$, we have $e^{-M}\le1/8$ and hence
$$
a_M \ge 1-2e^{-M} \ge \frac34,
\qquad
a_M^{-1}\le2.
$$
Let $M_{\rm id}:=\frac{M}{c_P}$. The head $h_0+2$ preserves the sinusoidal positional rows. We use
$$
\Qb_1^{h_0+2} =
\begin{bmatrix}
\mathbf 0_{1\times D}
\\
\mathbf e_{D-1}^\top
\\
\mathbf e_D^\top
\\
\mathbf 0_{1\times D}
\\
\mathbf 0_{1\times D}
\end{bmatrix},
\qquad
\Kb_1^{h_0+2} =
\begin{bmatrix}
\mathbf 0_{1\times D}
\\
M_{\rm id}\mathbf e_{D-1}^\top
\\
M_{\rm id}\mathbf e_D^\top
\\
\mathbf 0_{1\times D}
\\
\mathbf 0_{1\times D}
\end{bmatrix},
$$
where $\Qb_1^{h_0+2}, \Kb_1^{h_0+2} \in\RR^{5\times D}$. Then
$$
s_{t,j'}^{h_0+2} = M_{\rm id}\cos(\theta_t-\theta_{j'}),
$$
and, for $t\ne j'$,
$$
s_{j',j'}^{h_0+2} - s_{t,j'}^{h_0+2} \ge c_PM_{\rm id} = M.
$$
Consequently,
$$
\sum_{t\ne j'}(\Ab_1^{h_0+2})_{t,j'} \le Pe^{-M}.
$$
Choose
$$
\Vb_1^{h_0+2} =
\begin{bmatrix}
\mathbf e_{D-1}^\top
\\
\mathbf e_D^\top
\end{bmatrix}
\in\RR^{2\times D}.
$$
By the shift-invariant Softmax calculation on the uniform periodic grid, \cite[Proof of Lemma~5]{shi2026learning},
$$
\bigl(\Vb_1^{h_0+2}\Zb_0\Ab_1^{h_0+2}\bigr)_{:,j'}
=
\lambda_{\rm id}
\begin{bmatrix}
\sin(\theta_{j'})
\\
\cos(\theta_{j'})
\end{bmatrix},
\qquad
j'\in[P],
$$
where
$$
\lambda_{\rm id} :=
\frac{\sum_{t=1}^P e^{M_{\rm id}\cos(\theta_t)}\cos(\theta_t)}{\sum_{t=1}^P e^{M_{\rm id}\cos(\theta_t)}}.
$$
Furthermore,
$$
0 \le 1-\lambda_{\rm id} \le 2(P-1)e^{-M}.
$$
Since $Pe^{-M}\le\frac14$, $2(P-1)e^{-M}\le \frac12$. Hence
$$
\lambda_{\rm id}\ge\frac12,
\qquad
\lambda_{\rm id}^{-1}\le2.
$$
\medskip
\noindent\textbf{Step 3: Output projection and feature error bounds.}
For each head $h\in[H^1]$, write
$$
\Cb^h := \Vb_1^h\Zb_0\Ab_1^h \in\RR^{2\times P}.
$$
For $r\in[n_x+3]$ and $j,q\in[P]$, define
$$
G_{r,j,q} := \left(\Cb^{(r-1)P+j}\right)_{1,q} \in\RR.
$$
The $P$ heads assigned to feature row $r$ form the block
$$
\Cb^{(r)} :=
\begin{bmatrix}
G_{r,1,1} & G_{r,1,2} & \cdots & G_{r,1,P}
\\
0 & 0 & \cdots & 0
\\
G_{r,2,1} & G_{r,2,2} & \cdots & G_{r,2,P}
\\
0 & 0 & \cdots & 0
\\
\vdots & \vdots & \ddots & \vdots
\\
G_{r,P,1} & G_{r,P,2} & \cdots & G_{r,P,P}
\\
0 & 0 & \cdots & 0
\end{bmatrix}
\in\RR^{2P\times P}.
$$
The remaining two heads give
$$
\Cb_{\rm ind} :=
\begin{bmatrix}
1 & 1 & \cdots & 1 & 1
\\
a_M & a_M & \cdots & a_M & 0
\end{bmatrix}
\in\RR^{2\times P},
$$
and
$$
\Cb_\theta :=
\lambda_{\rm id}
\begin{bmatrix}
\sin(\theta_1)
&
\sin(\theta_2)
&
\cdots
&
\sin(\theta_{P-1})
&
\sin(\theta_P)
\\
\cos(\theta_1)
&
\cos(\theta_2)
&
\cdots
&
\cos(\theta_{P-1})
&
\cos(\theta_P)
\end{bmatrix}
\in\RR^{2\times P}.
$$
Concatenating all head outputs gives
$$
\Cb_1 :=
\begin{bmatrix}
\Cb^{(1)}
\\
\vdots
\\
\Cb^{(n_x+3)}
\\
\Cb_{\rm ind}
\\
\Cb_\theta
\end{bmatrix}
\in\RR^{2H^1\times P}.
$$
Let
$$
\mathbf I_{\rm sum} :=
\begin{bmatrix}
1 & 0 & 1 & 0 & \cdots & 1 & 0
\end{bmatrix}
\in\RR^{1\times 2P},
$$
and define
$$
\Rb_{\rm feat} :=
\begin{bmatrix}
\mathbf I_{\rm sum}
&
\mathbf 0_{1\times 2P}
&
\cdots
&
\mathbf 0_{1\times 2P}
\\
\mathbf 0_{1\times 2P}
&
\mathbf I_{\rm sum}
&
\cdots
&
\mathbf 0_{1\times 2P}
\\
\vdots
&
\vdots
&
\ddots
&
\vdots
\\
\mathbf 0_{1\times 2P}
&
\mathbf 0_{1\times 2P}
&
\cdots
&
\mathbf I_{\rm sum}
\end{bmatrix}
\in
\RR^{(n_x+3)\times 2(n_x+3)P}.
$$
Choose
$$
\Wb_1^O :=
\begin{bmatrix}
\Rb_{\rm feat}
&
\mathbf 0_{(n_x+3)\times 4}
\\
\mathbf 0_{(D-n_x-7)\times 2(n_x+3)P}
&
\mathbf 0_{(D-n_x-7)\times 4}
\\
\mathbf 0_{4\times 2(n_x+3)P}
&
I_4
\end{bmatrix}
\in
\RR^{D\times 2H^1}.
$$

For $r\in[n_x+3]$ and $j\in[P]$, define
$$
\widehat F_{r,j} := \sum_{j'=1}^P G_{r,j',j}.
$$
Then the output of the first MHA layer,
$$
\widehat{\Zb}_1 = \Wb_1^O\Cb_1,
$$
satisfies, for $j<P$,
$$
(\widehat{\Zb}_1)_{:,j} =
\begin{bmatrix}
\widehat F_{1,j}
\\
\vdots
\\
\widehat F_{n_x+3,j}
\\
\mathbf 0_{D-n_x-7}
\\
1
\\
a_M
\\
\lambda_{\rm id}\sin(\theta_j)
\\
\lambda_{\rm id}\cos(\theta_j)
\end{bmatrix},
\qquad
(\widehat{\Zb}_1)_{:,P} =
\begin{bmatrix}
\widehat F_{1,P}
\\
\vdots
\\
\widehat F_{n_x+3,P}
\\
\mathbf 0_{D-n_x-7}
\\
1
\\
0
\\
\lambda_{\rm id}\sin(\theta_P)
\\
\lambda_{\rm id}\cos(\theta_P)
\end{bmatrix}.
$$
For each feature row $r\in[n_x+3]$ and each column $j\in[P]$,
one feature head contributes the target-column error and the remaining $P-1$ heads contribute non-target-column errors. Hence
$$
\begin{aligned}
|\widehat F_{r,j}-F_{r,j}|
&=
\left|
G_{r,j,j}-F_{r,j}
+
\sum_{j'\ne j}G_{r,j',j}
\right|
\\
&\le
|G_{r,j,j}-F_{r,j}|
+
\sum_{j'\ne j}|G_{r,j',j}|
\\
&\le
2PB_{\rm aff}e^{-M}
+
(P-1)PB_{\rm aff}e^{-M}
\\
&\le
2P^2B_{\rm aff}e^{-M}
=
\varepsilon_{\rm feat}.
\end{aligned}
$$
For an active column $j=j(k_j,l_j)$, this gives
$$
\max\left\{
\max_{i\in[n_x]}
\left|
\widehat F_{i,j}-B_{i,k_j}(u)
\right|,
\left|
\widehat F_{n_x+1,j}-Q_{l_j}(\yb)
\right|,
\left|
\widehat F_{n_x+2,j}-N_{k_j}
\right|,
\left|
\widehat F_{n_x+3,j}-A_{k_j,l_j}
\right|
\right\}
\le
\varepsilon_{\rm feat}.
$$
Since $F_{r,P}=0$ for every $r\in[n_x+3]$, define $\delta_i^B:=\widehat F_{i,P}$, $\delta^Q:=\widehat F_{n_x+1,P}$, $\delta^N:=\widehat F_{n_x+2,P}$ and $\delta^A:=\widehat F_{n_x+3,P}$. Then
$$
\max\left\{\max_{i\in[n_x]}|\delta_i^B|, |\delta^Q|, |\delta^N|, |\delta^A| \right\}
\le
\varepsilon_{\rm feat}.
$$

\medskip
\noindent\textbf{Step 4: Parameter bound.}

We next bound the parameters of the MHA layer. For the feature extraction heads, \eqref{eq:B_aff_definition} gives $\|\Vb_1^h\|_{\max}\le B_{\rm aff}$. Moreover, since $M_s=\frac{M}{c_P}$ and $M_\theta=\frac{4M}{c_P}$,
$$
|M_0| \le M_\theta+M_s = \frac{5M}{c_P}.
$$
Thus, the feature-extraction heads satisfy
$$
\|\Qb_1^h\|_{\max},
\|\Kb_1^h\|_{\max}
\le
\max\left\{1,\frac{5M}{c_P}\right\},
\qquad
\|\Vb_1^h\|_{\max}\le B_{\rm aff}.
$$
The two auxiliary heads satisfy the same parameter bound, and
$\|\Wb_1^O\|_{\max}\le1$. Therefore,
$$
M_{A_1}
\le
\max\left\{
1,\,
B_{\rm aff},\,
\frac{5M}{c_P}
\right\}.
$$
\end{proof}

\subsubsection{Proof of Lemma~\ref{lemma:pair_token_joint_logit_ffn}}
\label{app:proof:lemma:pair_token_joint_logit_ffn}

\begin{proof}[Proof of Lemma~\ref{lemma:pair_token_joint_logit_ffn}]
\phantomsection
\label{proof:pair_token_joint_logit_ffn}

Let $\mathbf e_r\in\RR^D$ denote the $r$-th standard basis vector.
Define
$$
\mathbf r_\Xi^\top
:=
\sum_{i=1}^{n_x}\mathbf e_i^\top
+
\mathbf e_{n_x+1}^\top
-
\mathbf e_{n_x+2}^\top
-
M_{\rm out}\mathbf e_{D-3}^\top
+
a_M^{-1}M_{\rm out}\mathbf e_{D-2}^\top.
$$
Also define
$$
\mathbf r_A^\top:=\mathbf e_{n_x+3}^\top,
\qquad
\mathbf r_1^\top:=\mathbf e_{D-3}^\top.
$$
Let
$$
\Rb_1
:=
\begin{bmatrix}
\mathbf r_\Xi^\top
\\
\mathbf r_A^\top
\\
\mathbf r_1^\top
\\
\mathbf 0_{(D-3)\times D}
\end{bmatrix}
\in\RR^{D\times D}.
$$

For every active column $j=j(k_j,l_j)\in[P-1]$, using
\eqref{eq:first_block_mha_output},
$$
\mathbf r_\Xi^\top(\widehat{\Zb}_1)_{:,j}
=
\sum_{i=1}^{n_x}\widetilde B_{i,k_j,l_j}
+
\widetilde Q_{k_j,l_j}
-
\widetilde N_{k_j,l_j}
-
M_{\rm out}
+
a_M^{-1}M_{\rm out}a_M
=
\widehat\Xi_j.
$$
Moreover,
$$
\mathbf r_A^\top(\widehat{\Zb}_1)_{:,j}
=
\widetilde A_{k_j,l_j}
=
\widetilde A_j,
\qquad
\mathbf r_1^\top(\widehat{\Zb}_1)_{:,j}
=
1.
$$

For the null column,
$$
\mathbf r_\Xi^\top(\widehat{\Zb}_1)_{:,P}
=
\sum_{i=1}^{n_x}\delta_i^B
-\delta^N
+\delta^Q
-M_{\rm out}
=
\widehat\Xi_P,
$$
since $(\widehat{\Zb}_1)_{D-2,P}=0$. Also,
$$
\mathbf r_A^\top(\widehat{\Zb}_1)_{:,P}
=
\delta^A,
\qquad
\mathbf r_1^\top(\widehat{\Zb}_1)_{:,P}
=
1.
$$
Hence
$$
\Rb_1\widehat{\Zb}_1
=
\begin{bmatrix}
\widehat\Xi_1 & \cdots & \widehat\Xi_{P-1} & \widehat\Xi_P
\\
\widetilde A_1 & \cdots & \widetilde A_{P-1} & \delta^A
\\
1 & \cdots & 1 & 1
\\
\multicolumn{4}{c}{\mathbf 0_{(D-3)\times P}}
\end{bmatrix}.
$$

We realize this linear map by a point-wise ReLU FFN of hidden width
$d_{\rm ff}^1=2D$. Choose
$$
\Wb_1^1
=
\begin{bmatrix}
I_D
\\
-I_D
\end{bmatrix},
\qquad
\mathbf b_1^1=\mathbf 0_{2D},
$$
and
$$
\Wb_1^2
=
\begin{bmatrix}
\Rb_1 & -\Rb_1
\end{bmatrix},
\qquad
\mathbf b_1^2=\mathbf 0_D.
$$
Using $x=\sigma(x)-\sigma(-x)$ componentwise, for every
$\mathbf z\in\RR^D$,
$$
\Wb_1^2
\sigma\left(
\Wb_1^1\mathbf z+\mathbf b_1^1
\right)
+
\mathbf b_1^2
=
\Rb_1\mathbf z.
$$
Therefore,
$$
\Zb_1
=
F_1(\widehat{\Zb}_1)
=
\Rb_1\widehat{\Zb}_1,
$$
which proves \eqref{eq:first_block_ffn_output}.

For every active column $j=j(k_j,l_j)$, the feature bounds in
Lemma~\ref{lemma:pair_token_affine_mha} give
$$
\begin{aligned}
|\widehat\Xi_j-\Xi_j|
&\le
\sum_{i=1}^{n_x}
\left|
\widetilde B_{i,k_j,l_j}-B_{i,k_j}(u)
\right|
\\
&\quad+
\left|
\widetilde N_{k_j,l_j}-N_{k_j}
\right|
+
\left|
\widetilde Q_{k_j,l_j}-Q_{l_j}(\yb)
\right|
\\
&\le
(n_x+2)\varepsilon_{\rm feat}.
\end{aligned}
$$
Similarly,
$$
|\widetilde A_j-A_{k_j,l_j}|
\le
\varepsilon_{\rm feat}.
$$
This proves \eqref{eq:joint_logit_ffn_error}.

For the null column,
$$
\widehat\Xi_P
\le
-M_{\rm out}
+
(n_x+2)\varepsilon_{\rm feat}.
$$
Moreover, for every active $j\in[P-1]$,
$$
|\widehat\Xi_j|
\le
(n_x+2)
\left(
B_{\rm aff}+\varepsilon_{\rm feat}
\right).
$$
Therefore,
$$
\widehat\Xi_P-\max_{j\in[P-1]}\widehat\Xi_j
\le
-M_{\rm out}
+
(n_x+2)\varepsilon_{\rm feat}
+
(n_x+2)
\left(
B_{\rm aff}+\varepsilon_{\rm feat}
\right)
=
-M,
$$
by the definition of $M_{\rm out}$.
Together with $|\delta^A|\le\varepsilon_{\rm feat}$ from
Lemma~\ref{lemma:pair_token_affine_mha}, this proves
\eqref{eq:joint_logit_null_separation}.

Finally, under $Pe^{-M}\le1/4$,
Lemma~\ref{lemma:pair_token_affine_mha} gives $a_M^{-1}\le2$.
Since $M_{\rm out}>M>1$,
$$
\|\Rb_1\|_{\max}\le 2M_{\rm out}.
$$
Hence the displayed FFN matrices satisfy
$$
M_{F_1}
\le
2M_{\rm out},
$$
which proves \eqref{eq:first_ffn_parameter_bound}.

\end{proof}

\subsubsection{Proof of Lemma~\ref{lemma:pair_token_softmax_aggregation}}
\label{app:proof:lemma:pair_token_softmax_aggregation}

\begin{proof}[Proof of Lemma~\ref{lemma:pair_token_softmax_aggregation}]
\phantomsection
\label{proof:pair_token_softmax_aggregation}
Choose
$$
\Qb_2=\mathbf e_3^\top,
\qquad
\Kb_2=\mathbf e_1^\top,
\qquad
\Vb_2=\mathbf e_2^\top.
$$
By the form of $\Zb_1$ in
\eqref{eq:first_block_ffn_output}, for every target column $j'\in[P]$,
$$
q_{j'}:=(\Qb_2\Zb_1)_{j'}=1,
$$
while for every source column $j\in[P]$,
$$
k_j:=(\Kb_2\Zb_1)_j=\widehat\Xi_j.
$$
Hence the attention scores are
$$
s_{j,j'}=k_jq_{j'}=\widehat\Xi_j,
$$
independent of $j'$, and therefore every target column receives the same
attention weights
$$
\gamma_j
=
\frac{\exp(\widehat\Xi_j)}
{\sum_{q=1}^{P}\exp(\widehat\Xi_q)}.
$$

By \eqref{eq:joint_logit_null_separation},
$$
\widehat\Xi_P-\max_{j<P}\widehat\Xi_j\le -M,
$$
and thus
$$
\gamma_P
\le
\exp\left(
\widehat\Xi_P-\max_{j<P}\widehat\Xi_j
\right)
\le e^{-M}.
$$

The value projection satisfies
$$
(\Vb_2\Zb_1)_j
=
\widetilde A_j,
\qquad j<P,
$$
and
$$
(\Vb_2\Zb_1)_P=\delta^A.
$$
Therefore,
$$
\widehat g(u,\yb)
=
\sum_{j=1}^{P-1}\gamma_j\widetilde A_j
+
\gamma_P\delta^A.
$$
Choosing $\Wb_2^O=\mathbf e_1$ yields $\widehat{\Zb}_2 = A_2(\Zb_1)$, with the form in \eqref{eq:pair_token_softmax_output}.

For $j<P$,
$$
\gamma_j=(1-\gamma_P)\widehat\omega_j,
$$
so
$$
\sum_{j=1}^{P-1}
|\gamma_j-\widehat\omega_j|
=
\gamma_P
\le e^{-M}.
$$
Using
$|\widetilde A_j|\le B_{\cV}+\varepsilon_{\rm feat}$
and
$|\delta^A|\le\varepsilon_{\rm feat}$,
we obtain
$$
\left|
\widehat g(u,\yb)
-
\sum_{j=1}^{P-1}
\widehat\omega_j\widetilde A_j
\right|
\le
\left(
B_{\cV}+2\varepsilon_{\rm feat}
\right)e^{-M}.
$$

Finally,
$$
\|\Qb_2\|_{\max}
=
\|\Kb_2\|_{\max}
=
\|\Vb_2\|_{\max}
=
\|\Wb_2^O\|_{\max}
=
1,
$$
and hence $M_{A_2}\le1$.
\end{proof}

\subsection{Proof of the Covering Number Bound}
\label{app:proof:lemma:transformer_covering}

\begin{proof}[Proof of Lemma~\ref{lemma:transformer_covering}]
\phantomsection
\label{proof:transformer_covering}

Let $T_\Theta$ and $T_{\widetilde\Theta}$ be two Transformer networks in
$\cT_\epsilon$ parameterized by $\Theta$ and $\widetilde\Theta$,
respectively, where $\|\Theta\|_\infty,\, \|\widetilde\Theta\|_\infty \le M_{\max}$. Assume the parameter difference is bounded by $\|\Theta-\widetilde\Theta\|_\infty \le \delta$. We aim to bound $\|T_\Theta-T_{\widetilde\Theta}\|_{\infty,\infty}$ with respect to $\delta$.

\medskip
\noindent\textbf{Step 1: Pre-processing and First Encoder Block.}
By Lemma~\ref{lemma:pair_token_preprocess}, the fixed preprocessing map
gives the same initial representation for the two networks, $\Zb_0=\widetilde{\Zb}_0$. Since $\xb_i\in[0,1]^{d_1}$, $\yb\in[0,1]^{d_2}$, $\|u\|_{L^\infty(\Omega_{\cU})}\le B_{\cU}$, and all
structural-positional entries are bounded by $1$, we have
$$
r_0
:=
\|\Zb_0\|_{\max}
=
\|\widetilde{\Zb}_0\|_{\max}
\le
B_{\cU}.
$$

Let $A_1$ and $F_1$ denote the MHA and FFN mappings in the first encoder
block. We first consider the concatenated output of the $H^1$ attention
heads before the output projection, denoted by
$$
\mathbf H_1(\Zb_0)
\in
\mathbb R^{(H^1d_v^1)\times P}.
$$
For a single attention head $h\in[H^1]$, define the pre-Softmax score
matrix
$$
\Sb_1^h
:=
\Zb_0^\top
(\Kb_1^h)^\top
\Qb_1^h
\Zb_0
\in
\mathbb R^{P\times P},
$$
the corresponding attention matrix
$$
\Ab_1^h
:=
\operatorname{Softmax}(\Sb_1^h)
\in
\mathbb R^{P\times P},
$$
and the head output
$$
\mathbf H_1^h(\Zb_0)
:=
\Vb_1^h\Zb_0\Ab_1^h
\in
\mathbb R^{d_v^1\times P}.
$$
Define $\widetilde{\Sb}_1^h$, $\widetilde{\Ab}_1^h$, and
$\widetilde{\mathbf H}_1^h$ analogously. Recall that
$$
d_k^1=5,
\qquad
d_v^1=2,
\qquad
H^1=(n_x+3)P+2.
$$
We decompose the difference into a parameter error and an input error:
$$
\left\|
\mathbf H_1(\Zb_0)
-
\widetilde{\mathbf H}_1(\widetilde{\Zb}_0)
\right\|_{\max}
\le
\underbrace{
\left\|
\mathbf H_1(\Zb_0)
-
\widetilde{\mathbf H}_1(\Zb_0)
\right\|_{\max}
}_{\mathrm{(I)\ Parameter\ Error}}
+
\underbrace{
\left\|
\widetilde{\mathbf H}_1(\Zb_0)
-
\widetilde{\mathbf H}_1(\widetilde{\Zb}_0)
\right\|_{\max}
}_{\mathrm{(II)\ Input\ Error}}.
$$
Since the preprocessing map is fixed, $\Zb_0=\widetilde{\Zb}_0$, and hence $\mathrm{(II)}=0$. It remains to bound the parameter error $\mathrm{(I)}$.

For the parameter error, fix a head $h\in[H^1]$. Since the input
$\Zb_0$ is the same for the two parameter choices,
$$
\begin{aligned}
\left\|
(\Kb_1^h)^\top\Qb_1^h
-
(\widetilde{\Kb}_1^h)^\top\widetilde{\Qb}_1^h
\right\|_{\max}
&\le
d_k^1
\|\Kb_1^h\|_{\max}
\|\Qb_1^h-\widetilde{\Qb}_1^h\|_{\max}
\\
&\quad+
d_k^1
\|\Kb_1^h-\widetilde{\Kb}_1^h\|_{\max}
\|\widetilde{\Qb}_1^h\|_{\max}
\\
&\le
2d_k^1M_{\max}\delta
\\
&\le
10M_{\max}\delta.
\end{aligned}
$$
Therefore,
$$
\begin{aligned}
\|\Sb_1^h-\widetilde{\Sb}_1^h\|_{\max}
&=
\left\|
\Zb_0^\top
\left[
(\Kb_1^h)^\top\Qb_1^h
-
(\widetilde{\Kb}_1^h)^\top\widetilde{\Qb}_1^h
\right]
\Zb_0
\right\|_{\max}
\\
&\le
10D^2r_0^2M_{\max}\delta.
\end{aligned}
$$
By Lemma~\ref{lemma:softmax_lipschitz}, for every column $j\in[P]$,
$$
\left\|
(\Ab_1^h)_{:,j}
-
(\widetilde{\Ab}_1^h)_{:,j}
\right\|_1
\le
2\|\Sb_1^h-\widetilde{\Sb}_1^h\|_{\max}
\le
20D^2r_0^2M_{\max}\delta.
$$

Since each attention matrix is column-stochastic,
$$
\|(\Ab_1^h)_{:,j}\|_1
=
\|(\widetilde{\Ab}_1^h)_{:,j}\|_1
=
1.
$$
Hence
$$
\begin{aligned}
\left\|
\mathbf H_1^h(\Zb_0)
-
\widetilde{\mathbf H}_1^h(\Zb_0)
\right\|_{\max}
&=
\left\|
\Vb_1^h\Zb_0\Ab_1^h
-
\widetilde{\Vb}_1^h\Zb_0\widetilde{\Ab}_1^h
\right\|_{\max}
\\
&\le
\left\|
(\Vb_1^h-\widetilde{\Vb}_1^h)
\Zb_0\Ab_1^h
\right\|_{\max}
\\
&\quad+
\left\|
\widetilde{\Vb}_1^h
\Zb_0
(\Ab_1^h-\widetilde{\Ab}_1^h)
\right\|_{\max}
\\
&\le
D\delta r_0
+
DM_{\max}r_0
\max_{j\in[P]}
\left\|
(\Ab_1^h)_{:,j}
-
(\widetilde{\Ab}_1^h)_{:,j}
\right\|_1
\\
&\le
\left(
Dr_0
+
20D^3M_{\max}^2r_0^3
\right)\delta
\\
&\le
21D^3M_{\max}^2r_0^3\delta.
\end{aligned}
$$

Since $\mathbf H_1$ is obtained by concatenating the $H^1$ head outputs,
the max norm does not increase under concatenation. Therefore,
$$
\left\|
\mathbf H_1(\Zb_0)
-
\widetilde{\mathbf H}_1(\widetilde{\Zb}_0)
\right\|_{\max}
\le
21D^3M_{\max}^2B_{\cU}^3\delta.
$$
Moreover,
$$
\|\mathbf H_1(\Zb_0)\|_{\max},
\,
\|\widetilde{\mathbf H}_1(\widetilde{\Zb}_0)\|_{\max}
\le
DM_{\max}B_{\cU}.
$$
Since $P\ge n_x+2$,
$$
H^1d_v^1
=
2\bigl((n_x+3)P+2\bigr)
\le
4P^2.
$$
The MHA outputs are
$$
A_1(\Zb_0)
=
\Wb_1^O\mathbf H_1(\Zb_0),
\qquad
\widetilde A_1(\widetilde{\Zb}_0)
=
\widetilde{\Wb}_1^O
\widetilde{\mathbf H}_1(\widetilde{\Zb}_0).
$$
Therefore,
$$
\begin{aligned}
\left\|
A_1(\Zb_0)
-
\widetilde A_1(\widetilde{\Zb}_0)
\right\|_{\max}
&\le
4P^2
\left(
21D^3M_{\max}^3B_{\cU}^3
+
DM_{\max}B_{\cU}
\right)\delta
\\
&\le
88P^2D^3M_{\max}^3B_{\cU}^3\delta.
\end{aligned}
$$
Moreover,
$$
\|A_1(\Zb_0)\|_{\max},
\,
\|\widetilde A_1(\widetilde{\Zb}_0)\|_{\max}
\le
4P^2DM_{\max}^2B_{\cU}.
$$

For the first point-wise FFN, define
$$
\mathbf Y_1
:=
\sigma\left(
\Wb_1^1A_1(\Zb_0)+\mathbf b_1^1\mathbf 1_P^\top
\right),
$$
and
$$
\widetilde{\mathbf Y}_1
:=
\sigma\left(
\widetilde{\Wb}_1^1
\widetilde A_1(\widetilde{\Zb}_0)
+
\widetilde{\mathbf b}_1^1\mathbf 1_P^\top
\right).
$$
Since ReLU is $1$-Lipschitz,
$$
\begin{aligned}
\|\mathbf Y_1-\widetilde{\mathbf Y}_1\|_{\max}
&\le
\left(
4P^2D^2M_{\max}^2B_{\cU}
+
88P^2D^4M_{\max}^4B_{\cU}^3
+
1
\right)\delta
\\
&\le
93P^2D^4M_{\max}^4B_{\cU}^3\delta.
\end{aligned}
$$
The magnitude is bounded by
$$
\|\mathbf Y_1\|_{\max},
\,
\|\widetilde{\mathbf Y}_1\|_{\max}
\le
5P^2D^2M_{\max}^3B_{\cU}.
$$

Let
$$
\Zb_1
=
F_1(A_1(\Zb_0)),
\qquad
\widetilde{\Zb}_1
=
\widetilde F_1(\widetilde A_1(\widetilde{\Zb}_0)).
$$
Since $d_{\rm ff}^1=2D$,
$$
\begin{aligned}
\|\Zb_1-\widetilde{\Zb}_1\|_{\max}
&\le
\left(
10P^2D^3M_{\max}^3B_{\cU}
+
186P^2D^5M_{\max}^5B_{\cU}^3
+
1
\right)\delta
\\
&\le
197P^2D^5M_{\max}^5B_{\cU}^3\delta.
\end{aligned}
$$
Finally, set $r_1 := 11P^2D^3M_{\max}^4B_{\cU}$. Then
$$
\|\Zb_1\|_{\max},
\,
\|\widetilde{\Zb}_1\|_{\max}
\le
r_1.
$$
Moreover, since $P,D,M_{\max},B_{\cU}\ge1$, we have $r_1\ge1$.

\medskip
\noindent\textbf{Step 2: Second Encoder Block.}
We apply the same decomposition to the second encoder block, now accounting for both the parameter perturbation and the input perturbation $\Zb_1-\widetilde{\Zb}_1$. Since
$$
H^2=1,
\qquad
d_k^2=d_v^2=1,
\qquad
d_{\rm ff}^2=2D,
$$
the MHA bound simplifies. Adapting the derivation from Step~1 to the
second block gives
$$
\begin{aligned}
\left\|
\mathbf H_2(\Zb_1)
-
\widetilde{\mathbf H}_2(\widetilde{\Zb}_1)
\right\|_{\max}
&\le
5M_{\max}^2D^3r_1^3\delta
+
5M_{\max}^3D^3r_1^2
\|\Zb_1-\widetilde{\Zb}_1\|_{\max}
\\
&\le
5M_{\max}^2D^3
\left(
11P^2D^3M_{\max}^4B_{\cU}
\right)^3
\delta
\\
&\quad+
5M_{\max}^3D^3
\left(
11P^2D^3M_{\max}^4B_{\cU}
\right)^2
\\
&\qquad\qquad\cdot
\left(
197P^2D^5M_{\max}^5B_{\cU}^3\delta
\right)
\\
&\le
6655P^6D^{12}M_{\max}^{14}B_{\cU}^3\delta
\\
&\quad+
119185P^6D^{14}M_{\max}^{16}B_{\cU}^5\delta
\\
&\le
125840P^6D^{14}M_{\max}^{16}B_{\cU}^5\delta.
\end{aligned}
$$
Its magnitude is bounded by
$$
\|\mathbf H_2(\Zb_1)\|_{\max},
\,
\|\widetilde{\mathbf H}_2(\widetilde{\Zb}_1)\|_{\max}
\le
DM_{\max}r_1
\le
11P^2D^4M_{\max}^5B_{\cU}.
$$

Applying the output projection
$\Wb_2^O\in\RR^{D\times1}$, since $H^2d_v^2=1$, we have
$$
\begin{aligned}
\left\|
A_2(\Zb_1)
-
\widetilde A_2(\widetilde{\Zb}_1)
\right\|_{\max}
&\le
M_{\max}
\left\|
\mathbf H_2(\Zb_1)
-
\widetilde{\mathbf H}_2(\widetilde{\Zb}_1)
\right\|_{\max}
\\
&\quad+
\delta
\|\widetilde{\mathbf H}_2(\widetilde{\Zb}_1)\|_{\max}
\\
&\le
M_{\max}
\left(
125840P^6D^{14}M_{\max}^{16}B_{\cU}^5\delta
\right)
\\
&\quad+
\delta
\left(
11P^2D^4M_{\max}^5B_{\cU}
\right)
\\
&\le
125851P^6D^{14}M_{\max}^{17}B_{\cU}^5\delta.
\end{aligned}
$$
Moreover,
$$
\begin{aligned}
\|A_2(\Zb_1)\|_{\max},
\,
\|\widetilde A_2(\widetilde{\Zb}_1)\|_{\max}
&\le
M_{\max}
\max\left\{
\|\mathbf H_2(\Zb_1)\|_{\max},
\|\widetilde{\mathbf H}_2(\widetilde{\Zb}_1)\|_{\max}
\right\}
\\
&\le
11P^2D^4M_{\max}^6B_{\cU}.
\end{aligned}
$$

For the second point-wise FFN, define
$$
\mathbf Y_2
:=
\sigma\left(
\Wb_2^1A_2(\Zb_1)
+
\mathbf b_2^1\mathbf 1_P^\top
\right),
\qquad
\widetilde{\mathbf Y}_2
:=
\sigma\left(
\widetilde{\Wb}_2^1
\widetilde A_2(\widetilde{\Zb}_1)
+
\widetilde{\mathbf b}_2^1\mathbf 1_P^\top
\right).
$$
Their magnitudes satisfy
$$
\begin{aligned}
\|\mathbf Y_2\|_{\max},
\,
\|\widetilde{\mathbf Y}_2\|_{\max}
&\le
DM_{\max}
\max\left\{
\|A_2(\Zb_1)\|_{\max},
\|\widetilde A_2(\widetilde{\Zb}_1)\|_{\max}
\right\}
+
M_{\max}
\\
&\le
12P^2D^5M_{\max}^7B_{\cU}.
\end{aligned}
$$
Since ReLU is $1$-Lipschitz,
$$
\begin{aligned}
\|\mathbf Y_2-\widetilde{\mathbf Y}_2\|_{\max}
&\le
D\delta
\|\widetilde A_2(\widetilde{\Zb}_1)\|_{\max}
\\
&\quad+
DM_{\max}
\left\|
A_2(\Zb_1)
-
\widetilde A_2(\widetilde{\Zb}_1)
\right\|_{\max}
+
\delta
\\
&\le
D\delta
\left(
11P^2D^4M_{\max}^6B_{\cU}
\right)
\\
&\quad+
DM_{\max}
\left(
125851P^6D^{14}M_{\max}^{17}B_{\cU}^5\delta
\right)
+
\delta
\\
&\le
125863P^6D^{15}M_{\max}^{18}B_{\cU}^5\delta.
\end{aligned}
$$

Let
$$
\Zb_2
=
F_2(A_2(\Zb_1)),
\qquad
\widetilde{\Zb}_2
=
\widetilde F_2(\widetilde A_2(\widetilde{\Zb}_1)).
$$
Since $d_{\rm ff}^2=2D$, it follows that
$$
\begin{aligned}
\|\Zb_2-\widetilde{\Zb}_2\|_{\max}
&\le
2D\delta
\|\widetilde{\mathbf Y}_2\|_{\max}
+
2DM_{\max}
\|\mathbf Y_2-\widetilde{\mathbf Y}_2\|_{\max}
+
\delta
\\
&\le
2D\delta
\left(
12P^2D^5M_{\max}^7B_{\cU}
\right)
\\
&\quad+
2DM_{\max}
\left(
125863P^6D^{15}M_{\max}^{18}B_{\cU}^5\delta
\right)
+
\delta
\\
&\le
251751P^6D^{16}M_{\max}^{19}B_{\cU}^5\delta.
\end{aligned}
$$

Finally, define
$$
r_2
:=
\max\left\{
\|\Zb_2\|_{\max},
\|\widetilde{\Zb}_2\|_{\max}
\right\}.
$$
Then
$$
\begin{aligned}
r_2
&\le
2DM_{\max}
\max\left\{
\|\mathbf Y_2\|_{\max},
\|\widetilde{\mathbf Y}_2\|_{\max}
\right\}
+
M_{\max}
\\
&\le
25P^2D^6M_{\max}^8B_{\cU}.
\end{aligned}
$$

\medskip
\noindent\textbf{Step 3: Readout Layer and Covering Number.}
The final output is computed by a linear readout vector
$\mathbf c\in\RR^{PD}$ with
$\|\mathbf c\|_\infty\le M_{\max}$. Hence,
$$
\begin{aligned}
\left|
T_\Theta(S_{n_x}(u),\yb)
-
T_{\widetilde\Theta}(S_{n_x}(u),\yb)
\right|
&\le
\|\mathbf c\|_1
\|\Zb_2-\widetilde{\Zb}_2\|_{\max}
+
\|\mathbf c-\widetilde{\mathbf c}\|_1
\|\widetilde{\Zb}_2\|_{\max}
\\
&\le
(PDM_{\max})
\left(
251751P^6D^{16}M_{\max}^{19}B_{\cU}^5\delta
\right)
\\
&\quad+
(PD\delta)
\left(
25P^2D^6M_{\max}^8B_{\cU}
\right)
\\
&\le
251776P^7D^{17}M_{\max}^{20}B_{\cU}^5\delta.
\end{aligned}
$$
Since the preceding bound is uniform over
$u\in\cU$ and $\yb\in\Omega_{\cV}$, we obtain
$$
\|T_\Theta-T_{\widetilde\Theta}\|_{\infty,\infty}
\le
L_{\rm par}\delta,
$$
where
$$
L_{\rm par}
:=
251776P^7D^{17}M_{\max}^{20}B_{\cU}^5.
$$

To obtain an $\eta$-cover of $\cT_\epsilon$ in the
$\|\cdot\|_{\infty,\infty}$ norm, it suffices to construct a
$\delta_\eta$-cover of the parameter cube
$[-M_{\max},M_{\max}]^{N_{\rm total}}$, where
$$
\delta_\eta
:=
\frac{\eta}{L_{\rm par}}.
$$
Since $\eta\le1$, $L_{\rm par}\ge1$, and $M_{\max}\ge1$, we have $\delta_\eta\le M_{\max}$. Hence the parameter cube admits an $\ell^\infty$ $\delta_\eta$-cover of
cardinality at most $\left(\frac{2M_{\max}}{\delta_\eta}\right)^{N_{\rm total}}$.

Therefore,
$$
\log\mathcal N
\left(
\eta,\cT_\epsilon,\|\cdot\|_{\infty,\infty}
\right)
\le
N_{\rm total}
\log\left(
\frac{2M_{\max}L_{\rm par}}{\eta}
\right)
=
N_{\rm total}
\log\left(
\frac{
503552P^7D^{17}M_{\max}^{21}B_{\cU}^5
}{\eta}
\right).
$$
Finally, since the clipping map $\Pi_{B_{\cV}}$ is $1$-Lipschitz,
clipping does not increase the covering number. Hence
$$
\log\mathcal N
\left(
\eta,\cH_\epsilon^{\rm clip},\|\cdot\|_{\infty,\infty}
\right)
\le
N_{\rm total}
\log\left(
\frac{
503552P^7D^{17}M_{\max}^{21}B_{\cU}^5
}{\eta}
\right).
$$
This proves \eqref{eq:transformer_covering_bound}.
\end{proof}

\section{Additional Experimental Details}
\label{app:additional_experiments}

We provide additional implementation, evaluation, and prediction details
for the Burgers' and KdV experiments in Section~\ref{sec:experiments}.

\subsection{Common Experimental Setup}
\label{app:common_experimental_setup}

Both experiments use a two-block encoder-only Transformer with model
dimension $192$, four attention heads per block, and feedforward
dimension $384$. The architecture is fixed across sample sizes and
intrinsic dimensions within each problem. Table~\ref{tab:experiment_setup}
summarizes the training and evaluation settings.

\begin{table}[h]
\centering
\caption{Training and evaluation settings for the Burgers' and KdV experiments.}
\label{tab:experiment_setup}
\begin{tabular}{lcc}
\toprule
 & Burgers' & KdV \\
\midrule
Initial learning rate & $10^{-3}$ & $3\times10^{-4}$ \\
Batch size & 256 & 64 \\
Random output queries per training function & 8 & 16 \\
Maximum epochs & 12,000 & 6,000 \\
Test functions & 128 & 128 \\
Test spatial grid & 1024 uniform points & 1024 uniform points \\
Repeats per $(d_{\cU},n)$ & 5 & 5 \\
\bottomrule
\end{tabular}
\end{table}

All models are trained with Adam. All experiments use noiseless training targets. Learning rates are reduced on validation plateaus, and early stopping
and checkpoint selection are based only on validation error. The sample
size $n$ denotes the number of independently sampled training functions
rather than the number of query--function pairs.

For both equations, the reported test MSE averages squared prediction
errors in physical units over $128$ held-out functions and $1024$
uniformly spaced spatial points. This uniform-grid evaluation provides
a discrete estimate of the population risk with $\rho_y$ taken to be
uniform over the spatial domain. At each $(d_{\cU},n)$, the plotted
marker is the geometric mean of the five test MSEs, the error bar is
one sample standard deviation in log MSE, and the exponent $p$ is
obtained from the equal-weight log-space fit
\begin{equation}
\label{eq:empirical_scaling_fit}
\log E(n)=a-p\log n.
\end{equation}

\subsection{Viscous Burgers' Equation}
\label{app:Burgers'_experiments}

For Burgers', we consider
\begin{equation}
\label{eq:Burgers'_pde}
u_t+uu_x=\nu u_{xx},
\end{equation}
with $\nu=0.02$ and $T=0.05$. The input family is
\begin{equation}
\label{eq:Burgers'_input_family}
u_0(x;\theta)
=
\sum_{j=0}^{d_{\cU}-1}\theta_j\phi_j(x),
\qquad
\theta_j\overset{\rm i.i.d.}{\sim}
\operatorname{Unif}[-0.25,0.25].
\end{equation}
To define the basis functions, let
\begin{equation}
\label{eq:Burgers'_bump}
b_h(r)
=
\begin{cases}
\exp\!\left[-\dfrac{1}{1-(r/h)^2}\right],
& |r|<h,\\
0, & |r|\ge h,
\end{cases}
\qquad h=0.05.
\end{equation}
We set
\begin{equation}
\label{eq:Burgers'_bump_basis}
\phi_j(x)
=
\frac{b_h(r_j(x))}{Z_h},
\qquad
r_j(x)
=
\left((x-j/d_{\cU}+1/2)\bmod 1\right)-1/2,
\end{equation}
for $j=0,\ldots,d_{\cU}-1$, where $Z_h$ is the periodic $L^2$
normalization constant, computed on a uniform $65{,}536$-point grid.
We consider $d_{\cU}\in\{2,8\}$ and
$n\in\{32,64,128,256,512,1024\}$.
The smaller training sets are nested subsets of the corresponding
largest training pool.

The training and validation targets are generated using a Fourier
pseudo-spectral method with two-thirds dealiasing and fourth-order
Runge--Kutta time stepping. For the uniform-grid test evaluation,
we recompute the held-out reference solutions using the periodic
Cole--Hopf transformation with a $4096$-point Fourier discretization,
validated by spatial refinement and comparison with an independent
Fourier--Runge--Kutta solver. This reevaluation uses the originally
trained checkpoints without retraining. We use a fixed subset of
$128$ functions from the original held-out test set for every sample
size.

Figure~\ref{fig:Burgers'_predictions} uses repeat $0$ and fixed test
sample $0$ within each intrinsic dimension. For visualization, the Burgers' periodic domain is shifted and displayed
using the normalized coordinate $\widetilde x\in[0,1]$.

\begin{figure}[t]
    \centering
    \includegraphics[width=0.9\linewidth]
    {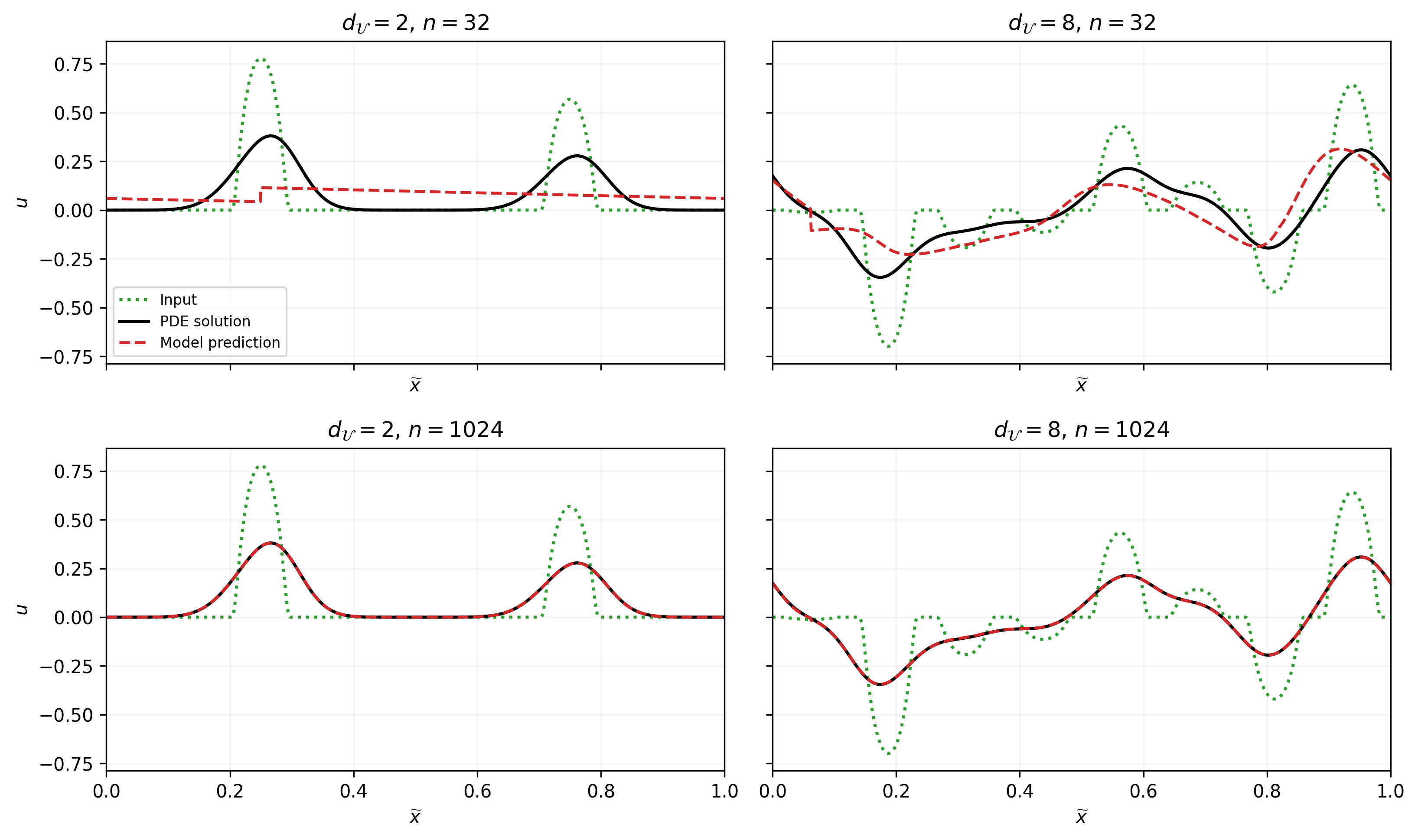}
    \caption{
    Representative Burgers' predictions. The top and bottom rows
    correspond to $n=32$ and $n=1024$, respectively; the left and right
    columns correspond to $d_{\cU}=2$ and $8$.
    Dotted curves show inputs, solid curves reference solutions,
    and dashed curves model predictions. Each column uses the same
    held-out test input across both sample sizes.
    }
    \label{fig:Burgers'_predictions}
\end{figure}

\subsection{KdV Equation}
\label{app:kdv_experiments}

For KdV, we consider
\begin{equation}
\label{eq:kdv_pde}
u_t+6uu_x+u_{xxx}=0
\end{equation}
on $x\in[0,44]$ with $T=1$. We use exact $N$-soliton solution
families parameterized by the speed and phase of each soliton. Writing
\begin{equation}
\label{eq:kdv_parameterization}
\theta
=
(c_1,p_1,\ldots,c_N,p_N),
\end{equation}
let $u^{(N)}(x,t;\theta)$ denote the corresponding exact $N$-soliton
solution and set
$$
u_0(x;\theta)=u^{(N)}(x,0;\theta),
$$
so that $d_{\cU}=2N$. We take $N=2$ and $4$, giving
$d_{\cU}=4$ and $8$, respectively.
The parameters are sampled independently according to
\begin{equation}
\label{eq:kdv_parameter_sampling}
c_j\sim
\operatorname{Unif}[0.97\bar c_j,1.03\bar c_j],
\qquad
p_j\sim
\operatorname{Unif}[\bar p_j-0.8,\bar p_j+0.8].
\end{equation}
For $d_{\cU}=4$, the nominal parameters are
$(\bar c_j)=(16,8)$ and $(\bar p_j)=(10,20)$; for
$d_{\cU}=8$, they are
$(\bar c_j)=(20,15,11,8)$ and
$(\bar p_j)=(4,10,16,22)$.
We use $n\in\{32,64,128,256,512,1024\}$.

The training, validation, and test targets are generated from the exact
multi-soliton solution. During training, both input and output amplitudes
are divided by $10$; predictions are multiplied by $10$ before computing
the test MSE in physical units. For visualization, we use $\widetilde x=x/44$ and display the amplitude
as $u/10$.

Figure~\ref{fig:kdv_predictions} uses repeat $0$ and fixed test
sample $0$ within each intrinsic dimension, with validation-selected
checkpoints and no visual sample selection.

\begin{figure}[t]
    \centering
    \includegraphics[width=0.9\linewidth]
    {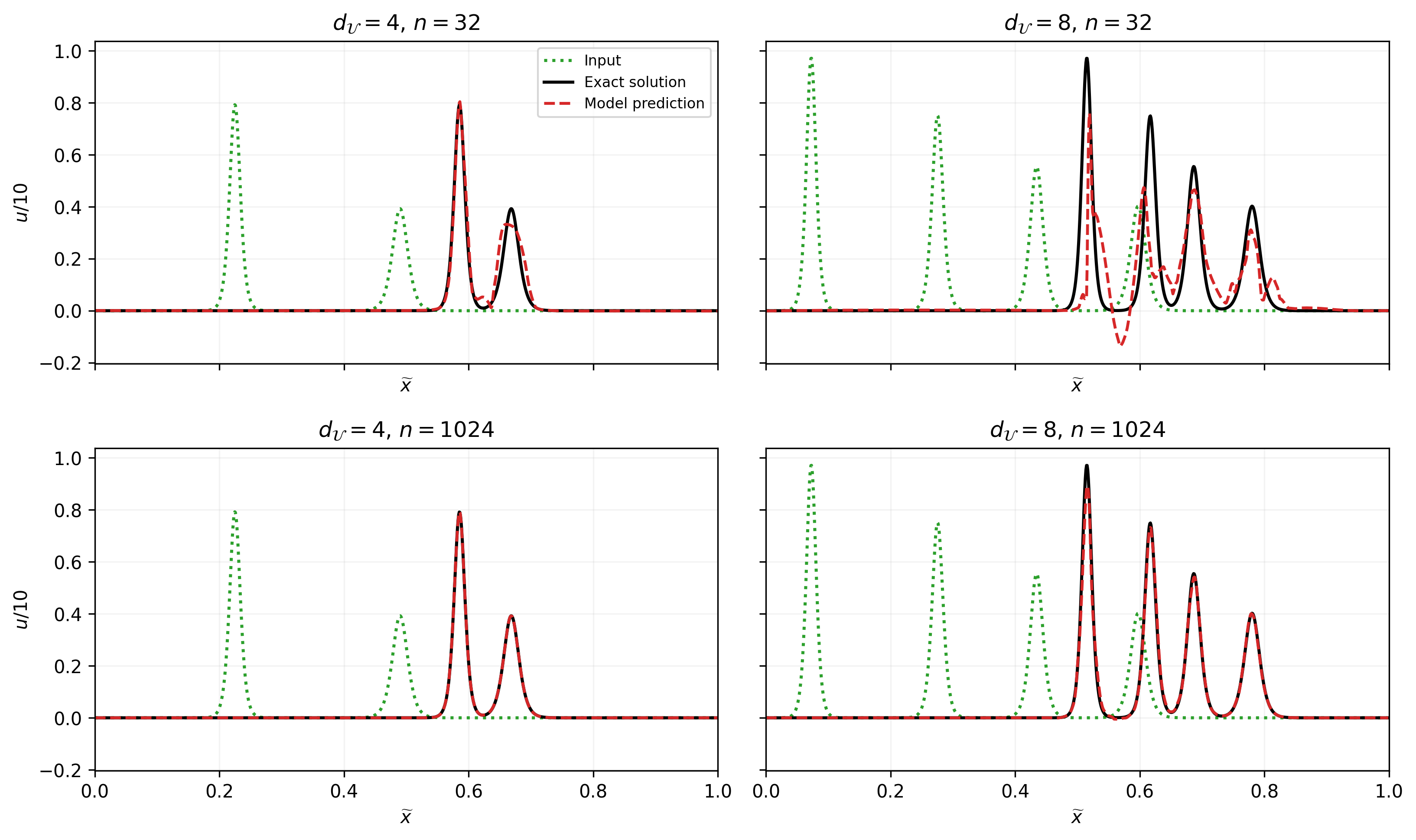}
    \caption{
    Representative KdV predictions. The top and bottom rows correspond
    to $n=32$ and $n=1024$, respectively; the left and right columns
    correspond to $d_{\cU}=4$ and $8$. Dotted curves show inputs,
    solid curves exact solutions, and dashed curves model predictions.
    Each column uses the same held-out test input across both sample sizes.
    }
    \label{fig:kdv_predictions}
\end{figure}

\end{document}